\pdfoutput=1

\documentclass{article}
\PassOptionsToPackage{numbers,sort&compress}{natbib}
\usepackage[final,main]{neurips_2026}
\usepackage[utf8]{inputenc}
\usepackage[T1]{fontenc}
\usepackage{hyperref,url,booktabs,nicefrac,microtype,xcolor,graphicx,adjustbox,array,float,placeins}
\usepackage{amsmath,amssymb,amsfonts,amsthm,mathtools}
\definecolor{paperlink}{RGB}{128,0,0}
\hypersetup{colorlinks=true,linkcolor=black,citecolor=paperlink,urlcolor=paperlink,hypertexnames=false}
\setcitestyle{numbers,square,comma}
\title{Predicting Human-Gain Curves under Local Proxy Optimization from Repeated Ratings}
\author{%
  Moonwon Choi$^{1}$\thanks{Equal contribution.}\quad
  Kisung Nam$^{2,3}$\footnotemark[1]\quad
  Seunggeun Lee$^{1}$\thanks{Corresponding author: \texttt{lee7801@snu.ac.kr}.}\\
  $^{1}$Graduate School of Data Science, Seoul National University\\
  $^{2}$Massachusetts General Hospital, Harvard Medical School\\
  $^{3}$Broad Institute of MIT and Harvard\\
  \texttt{yellowbill@snu.ac.kr}\quad
  \texttt{knam@broadinstitute.org}\quad
  \texttt{lee7801@snu.ac.kr}\\[2pt]
  {\small\url{https://github.com/pacifiic/local-proxy-diagnostics}}%
}
\hypersetup{%
  pdftitle={Predicting Human-Gain Curves under Local Proxy Optimization from Repeated Ratings},
  pdfauthor={Moonwon Choi, Kisung Nam, Seunggeun Lee}%
}

\newtheorem{theorem}{Theorem}
\newtheorem{lemma}{Lemma}
\newtheorem{proposition}{Proposition}
\newtheorem{corollary}{Corollary}

\newcommand{\E}{\mathbb{E}}
\newcommand{\Var}{\operatorname{Var}}
\newcommand{\Cov}{\operatorname{Cov}}
\newcommand{\KL}{\operatorname{KL}}
\newcommand{\piref}{\pi_{\mathrm{ref}}}

\newcommand{\sigr}{\sigma_r}
\newcommand{\Anu}{A_\nu}
\newcommand{\calD}{\mathcal{D}}
\def\INCLUDEAPPENDIX{}
\DeclareRobustCommand{\Appref}[2]{Appendix~\ifdefined\INCLUDEAPPENDIX\ref{#1}\else#2\fi}
\newcommand{\runhead}[1]{\par\noindent\textbf{#1.}\quad}

\DeclareRobustCommand{\reviewedit}[2]{\ifdefined\cleanreview #1\else{\color{red}#1} {\color{blue}(#2)}\fi}
\begin{document}
\maketitle

\begin{abstract}
Optimizing a proxy objective can improve human-perceived quality, but the resulting gain can vary with optimization strength. Average proxy--human correlation alone does not reveal how human-perceived quality changes as optimization strength increases. We propose \textnormal{\textbf{PACE}} (Proxy Alignment and Curvature Expansion), a pre-optimization diagnostic that uses repeated human ratings to predict the human-gain curve induced by a fixed proxy. The method models proxy optimization as a KL-regularized exponential tilt. At KL radius $d$, the local gain is approximated by $A\sqrt{d}+Bd$, where $A$ measures first-order proxy--human alignment and $B$ captures the leading second-order correction along the proxy-optimization path. We estimate $A$ and $B$ from repeated human ratings on calibration items and test the resulting gain-curve prediction on held-out evaluation items. Across SummEval, analysis-eligible WMT MQM clusters, an OpenMEVA-MAGS ROC/WP repeated-rating slice, and a GPT-4o judge proxy experiment, PACE consistently reduces held-out gain-prediction error relative to $A$-only, zero-gain, and reliability-only baselines. With the same pre-pool information, it is competitive with flexible proxy-conditional calibration, and a LoRA-DPO experiment provides a modest bridge to actual optimization. Transfer experiments across datasets and evaluation slices show that source coefficients are useful priors, but target calibration further improves prediction. The resulting diagnostic is local and one-sided: it characterizes gentle optimization along the positive tilt branch of a fixed item pool.
\end{abstract}

\section{Introduction}

Post-training and inference-time selection often improve model outputs by optimizing a measurable proxy for human preference. Reinforcement Learning from Human Feedback (RLHF) learns reward models from human comparisons and trains policies against the learned rewards \citep{christiano2017deep,stiennon2020learning,ouyang2022training,bai2022training}. Direct alignment methods optimize preference-derived objectives without explicitly fitting a reward model \citep{rafailov2023dpo}. Best-of-$N$, reranking, and judge-based selection apply analogous proxy pressure at inference time by selecting outputs with high reward or judge scores \citep{stiennon2020learning,beirami2025theoretical}. Although these methods can be effective, the human benefit of proxy optimization is not fixed by average proxy--human correlation: as optimization strength increases, human-perceived quality may improve, saturate, or eventually decline \citep{stiennon2020learning,gao2023scaling,pan2022effects,lambert2023alignment,manheim2019goodhart,skalse2022rewardhacking}. Thus, the practical question is not only whether a proxy is correlated with human judgments on average, but how much human gain its local optimization is expected to induce as optimization strength varies.

Repeated human ratings make this question testable. Proxy optimization upweights high-proxy items, so the relevant target is the human quality of the locally reweighted high-proxy region, not annotator agreement or global proxy--human association alone. Repeated ratings give multiple noisy views of the same item-level quality. In direct human-proxy settings, one view can serve as the proxy and the remaining ratings can define a leave-one-rater-out human gold target. For external proxies, such as learned rewards or large-language-model (LLM) judges, the same repeated-rating matrix supplies the human calibration and evaluation target. Disjoint items can then be reserved for prospective evaluation. Reliability tells us whether the human target is observable, but the magnitude of proxy-induced human gain depends on proxy--gold alignment and the quality of the high-proxy region, not on measurement stability alone. Our method, PACE (Proxy Alignment and Curvature Expansion), is predictive: a small repeated-rating calibration matrix can estimate two local coefficients, the PACE coefficients, that forecast the held-out human-gain curve induced by a fixed proxy. For example, suppose that 300 items with repeated ratings are available now, and the decision is whether to generate a pool of 1,000 candidates and optimize it against a reward model. The coefficients are estimated from the 300 rated items before the larger pool exists. The forecast then states how much human-rated quality a given optimization strength should buy within a checked range, or signals that more ratings, gentler selection, or abstention is warranted (Figure~\ref{fig:workflow}).

\begin{figure}[t]
\centering
\includegraphics[width=\textwidth]{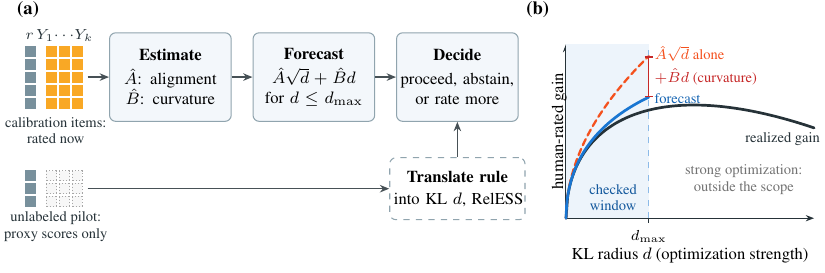}
\caption{Using PACE.  (a) Proxy scores $r$ and $k\ge2$ human ratings $Y_1,\dots,Y_k$ on calibration items give plug-in estimates of the alignment coefficient $A$ and the curvature coefficient $B$, which forecast the local human-gain curve before optimization.  A planned rule that is not a raw-score tilt, such as Best-of-$N$, is first translated into its induced KL radius and relative effective sample size (RelESS) on an unlabeled, proxy-scored pilot (dashed).  No human labels of the target pool are used.  (b) Schematic with $\hat B<0$: $\hat A\sqrt d$ alone overstates the gain, and $\hat Bd$ corrects it inside the checked window; beyond it the realized gain may saturate or decline, and no forecast is made.}
\label{fig:workflow}
\end{figure}

We study these coefficients on a smooth finite-item proxy-optimization path.  Given a reference item distribution $\piref$ and proxy score $r$, we model small proxy pressure by the exponential reweighting $\pi_t(i)\propto \piref(i)\exp(t r_i)$.  Increasing $t$ gradually concentrates mass on high-proxy items.  The corresponding Kullback--Leibler (KL) radius $d=\KL(\pi_t\Vert\piref)$ measures the strength of this intervention.

Let $J$ denote human gold quality, and let $Q(d)$ be expected human quality along this path.  The local gain curve then satisfies
\begin{equation}
\label{eq:local-intro}
Q(d)=Q(0)+A\sqrt d+Bd+O(d^{3/2}).
\end{equation}
The alignment coefficient $A$ measures standardized proxy--gold alignment in the direction in which optimization moves the distribution.  The curvature coefficient $B$ is the quadratic term in $x=\sqrt d$, so $A>0$ and $B<0$ mean that the gain begins to saturate.  Their signs and magnitudes are properties of the feedback geometry, that is, the joint structure of $J$ and $r$ estimated from calibration data.  Eq.~\eqref{eq:local-intro} is local and one-sided: it describes small KL radii on the positive tilt branch of a fixed item pool, not the terminal regime of strong optimization, and a negative $B$ is leading curvature rather than a predicted collapse.  Discrete selection rules such as Best-of-$N$ and actual policy training therefore enter as bridges, after translation into an induced KL radius, not as targets of the theorem.

We evaluate PACE prospectively: coefficients estimated on calibration items predict gain curves on disjoint evaluation items (Section~\ref{sec:protocol}).  Across SummEval, eligible WMT Multidimensional Quality Metrics (MQM) clusters, OpenMEVA-MAGS, and a GPT-4o external-judge proxy, the PACE forecast has lower held-out error than an $A$-only predictor, a zero-gain predictor, and a reliability-only surrogate.  Flexible proxy-conditional calibration by binning, splines, or isotonic regression is the natural alternative.  When every method sees only the calibration items, the PACE forecast is competitive with these calibrators, and when the proxy scores of the target pool are also available, the two-term projection of each fitted calibrator closely tracks its exact integration.  A three-seed LoRA-DPO experiment scored by a frozen silver reward model gives a modest bridge to actual optimization, in which the realized gain is about half of the idealized tilt gain at the same radius.  Cross-slice transfer shows that source coefficients are useful priors that target calibration improves.

\runhead{Contributions} We make three contributions.  First, we define the human-gain curve induced by KL-regularized proxy optimization on a finite item set.  Second, we derive a scale-free two-coefficient local law, $A\sqrt d+Bd$, in which $A$ measures standardized proxy--human alignment and $B$ the leading curvature.  Third, we turn repeated human-rating matrices into a prospective calibration protocol for held-out human-gain curves, compare it with flexible calibration under matched information, and check discrete selection rules and actual DPO training against the checked KL window.

\section{Related work}

\runhead{Proxy optimization, KL control, and gain prediction}
Preference-based alignment methods learn from human feedback through explicit reward modeling \citep{christiano2017deep,stiennon2020learning,ouyang2022training,bai2022training} or direct preference objectives \citep{rafailov2023dpo}. KL-regularized reward optimization is central to RLHF and controlled generation \citep{ahmadian2024back}: early RLHF systems penalize deviation from a reference policy \citep{ziegler2019finetuning}, and KL-regularized RL can be interpreted through variational inference \citep{korbak2022rlbayesian}.  In the nonparametric fixed-proxy limit, this objective yields an exponential tilt of the reference distribution, which is the analytic frontier used below.  A related line of work studies reward-model overoptimization, objective mismatch, direct-alignment overoptimization, Goodhart effects, and reward hacking \citep{gao2023scaling,pan2022effects,lambert2023alignment,rafailov2024directoveropt, goodhart1975problems,manheim2019goodhart,skalse2022rewardhacking}.  These works show that optimizing a proxy score does not by itself determine improvement in the intended human objective.  We ask a complementary pre-optimization question: what human-gain curve will a fixed proxy induce under small-KL optimization on the item set to be evaluated?  A direct alternative to an analytic law is to estimate the proxy-conditional mean $\E[J\mid r]$ by binning, isotonic regression, or penalized splines \citep{barlow1972statistical,eilers1996flexible} and integrate it along the tilt.  The relevant contrast is not between opaque and transparent algorithms but between what each method estimates and what information it can see, so we compare the two under matched information (Section~\ref{sec:exp-prepool}).

\runhead{Inference-time selection, reward-model evaluation, and LLM judges}
Best-of-$N$, reranking, and judge-based selection optimize scalar proxy scores at inference time, and recent work analyzes best-of-$N$-style policies theoretically and algorithmically \citep{jinnai2025regularized,beirami2025theoretical}.  Reward-model benchmarks evaluate whether learned rewards rank preference pairs correctly across challenging model outputs and domains \citep{lambert2025rewardbench}. LLM-as-a-judge methods provide scalable evaluators \citep{zheng2023judging}, but their reliability depends on prompt, model, comparison format, and evaluation design \citep{liu2023geval,wang2024fairevaluators,huang2025judge,yamauchi2025judge}. PACE asks a different question: how much held-out human gain is induced by locally optimizing such a proxy. Best-of-$N$ is used as a bridge to discrete inference-time selection, while the GPT-4o judge experiment tests PACE with an external LLM-based proxy rather than making benchmark claims about GPT-4o as a general evaluator.

\runhead{Repeated human evaluation and reliability}
Human-evaluation datasets show substantial annotator and protocol variation, motivating reliability analysis in summarization, machine translation, and broader NLP evaluation \citep{fabbri2021summeval,freitag2021large,belz2023reproducibility,hardy2025glitters}. Classical intraclass-correlation methods quantify the item-level component of repeated ratings \citep{shrout1979intraclass,mcgraw1996forming}. This literature asks how stable or reproducible human annotations are. We use repeated ratings for a different purpose: predicting the held-out human-gain curve induced by optimizing a scalar proxy. Reliability anchors the observation model, while the coefficients below model the intervention outcome.

\section{Local diagnostic}

\runhead{Observation process}
Let $J_i$ be the gold quality of item $i$ under a reference distribution $\piref$, and let $r_i$ be a proxy score.  In experiments, $J$ is a held-out human aggregate, not an unobserved latent truth.  In a simple repeated-rating model $Y_{ij}=J_i+\eta_{ij}$, with raters $j$, one rating view can serve as $r$ and the remaining views define a leave-one-rater-out gold target; for external proxies such as learned rewards or LLM judges, the repeated ratings supply the human target.  Intraclass-correlation (ICC) reliability, the share of rating variance due to the item term \citep{shrout1979intraclass,mcgraw1996forming}, anchors this measurement model, but the gain predictor below depends on proxy--gold alignment and curvature, not on reliability alone.

\runhead{Tilt path and KL-regularized frontier}
For a fixed proxy score $r$ on the reference distribution $\piref$, we model local proxy optimization by the exponential-tilt path
\begin{equation}
\label{eq:tilt}
\frac{d\pi_t}{d\piref}(i)
=
\frac{\exp(t r_i)}
{\E_{\piref}\exp(t r)},
\qquad t\ge 0 .
\end{equation}
Increasing $t$ reweights the reference distribution toward high-proxy items. We measure the strength of this intervention by the induced KL radius $d(t)=\KL(\pi_t\Vert\piref)$, and write $Q(t)=\E_{\pi_t}[J]$ for the corresponding expected human gold quality.  Whenever $d(t)$ is locally invertible on the positive branch, we write $Q(d)=Q(t(d))$ and call $\Delta(d)=Q(d)-Q(0)$ the human-gain curve induced by the fixed proxy.

The tilt path in Eq.~\eqref{eq:tilt} is the nonparametric solution of KL-regularized proxy maximization, a classical $I$-projection / variational identity \citep{csiszar1975idivergence,donsker1976asymptotic}. The proof is in \Appref{app:variational-proof}{A.1}.

\begin{proposition}[Exponential tilt as the KL-regularized frontier]
\label{prop:variational}
Fix $t>0$ and assume $\E_{\piref}[(1+|r|)e^{tr}]<\infty$.  Then, over $\pi$ with finite KL,
\begin{equation}
\label{eq:variational}
\sup_{\pi\ll\piref}
\left\{
t\,\E_\pi[r]-\KL(\pi\Vert\piref)
\right\}
=
\log \E_{\piref} e^{tr},
\end{equation}
with unique maximizer $\pi_t$ of Eq.~\eqref{eq:tilt}, which therefore maximizes $\E_\pi[r]-t^{-1}\KL(\pi\Vert\piref)$.  On any regular interval where $d(t)$ is strictly increasing, the same $\pi_t$ solves the corresponding fixed-KL problem.  The terminal regime, where $\pi_t$ concentrates on proxy maximizers, is outside our scope.
\end{proposition}

\runhead{Scale-free coefficient law}
The following coefficient law is scale-free: the proxy score need not be on the same numerical scale as human gold.  All moments and cumulants below are taken under $\piref$ unless otherwise stated.  Let $\sigr^2=\Var_{\piref}(r)$, $\kappa_r=\kappa_3(r,r,r)$, and $\kappa_{Jrr}=\kappa_3(J,r,r)$.

\begin{theorem}[Local coefficient law]
\label{thm:local}
Assume $\sigr>0$ and either $J,r$ are bounded or the joint cumulant-generating function of $(J,r)$ is $C^4$ in a neighborhood of the origin.  Along the positive exponential-tilt branch in Eq.~\eqref{eq:tilt}, as $d\downarrow0$,
\begin{equation}
\label{eq:local-law}
Q(d)=Q(0)+A\sqrt d+Bd+O(d^{3/2}),
\end{equation}
where
\begin{equation}
\label{eq:coefficients}
A=\sqrt{2}\frac{\Cov(J,r)}{\sigma_r},
\qquad
B=\frac{\kappa_{Jrr}}{\sigma_r^2}
-\frac{2\Cov(J,r)\kappa_r}{3\sigma_r^4}.
\end{equation}
\end{theorem}

\runhead{Standardized interpretation}
Let $Z=(r-\E r)/\sigma_r$ and $\bar J=J-\E J$.  Then the same coefficients can be written as
\begin{equation}
\label{eq:standardized-coefficients}
A=\sqrt2\,\E[\bar J Z],
\qquad
B=
\underbrace{\E[\bar J Z^2]}_{\text{second-order proxy--gold moment}}
-
\underbrace{\tfrac23\E[\bar JZ]\E[Z^3]}_{\text{proxy-skewness correction}}.
\end{equation}
The coefficient $A$ measures first-order standardized alignment between proxy and gold.  The coefficient $B$ is the leading curvature term on the positive tilt branch.  Its first term is a two-sided moment: it is positive when items far from the proxy mean, on either side, have above-average gold quality.  The skewness term is a coordinate correction that appears because the path is indexed by the KL radius rather than by the standardized tilt $s=t\sigma_r$.  In $s$, $Q-Q(0)=\E[\bar JZ]s+\tfrac12\E[\bar JZ^2]s^2+O(s^3)$ and $d=\tfrac12s^2+\tfrac13\E[Z^3]s^3+O(s^4)$, so $s(d)=\sqrt{2d}-\tfrac23\E[Z^3]d+O(d^{3/2})$.  With positive alignment, a right-skewed proxy spends more KL per unit of $s$, so a fixed radius buys less tilt and the correction is negative; left skew makes it positive.  The correction vanishes when either alignment or skewness is zero, and it is comparable to the first term when $\tfrac23|\E[\bar JZ]\E[Z^3]|\approx|\E[\bar JZ^2]|$ (\Appref{app:skewness}{B}).

\runhead{Proxy-error interpretation}
When proxy and gold share a rating scale, write
\begin{equation}
\label{eq:proxy-decomp}
r-\E r=(J-\E J)+\varepsilon,
\qquad
\rho=\Cov(J,\varepsilon).
\end{equation}
Then $\Cov(J,r)=\sigma_J^2+\rho$, so reliability alone can be miscalibrated when proxy error is correlated with gold quality.  Judges, learned rewards, and rerankers need not share this scale.

\runhead{Why correlation is not enough}
A global proxy--gold correlation compresses the joint distribution into one first-order association.  Proxy optimization instead reweights the item set toward high proxy scores, and $B$ captures the leading correction as the tilt begins to concentrate there.  Two proxies can therefore have similar reliability or correlation but different local gain curves.

\runhead{Scope of the local readout}
The coefficient readout is local and one-sided.  Both coefficients describe a specified finite item set, a fixed proxy, and the positive tilt direction near the reference distribution.  A positive $A$ supports optimization only inside a checked neighborhood, and, with $A>0$, a negative $B$ signals leading saturation, not an eventual downturn or a universal collapse law.  The law says nothing about the terminal regime, the negative branch, or rules that are not raw-score tilts.  The empirical protocol therefore pairs coefficient prediction with direct finite-range diagnostics: if tilted variance collapses, the effective sample size (ESS) becomes small, or the target radius is unreachable, the checked KL window should be narrowed.  \Appref{app:regime-readout}{A} summarizes coefficient sign patterns and their local implications.

\runhead{Finite-range structure}
Theorem~\ref{thm:local} is local, so finite-range use requires diagnostics rather than extrapolation.  Two exact identities motivate them.  First, $d'(t)=t\Var_{\pi_t}(r)$, so if $0<v_-\le\Var_{\pi_t}(r)\le v_+$ on $[0,T]$, then
\begin{equation}
\label{eq:quadratic}
\tfrac{v_-}{2}t^2\le d(t)\le\tfrac{v_+}{2}t^2,
\end{equation}
and $\sqrt d$ is locally proportional to the tilt strength $t$, which is why Eq.~\eqref{eq:local-law} is written in powers of $\sqrt d$.  Second, if $h(u)=\E_{\piref}[J\mid r=u]$, then
\begin{equation}
\label{eq:global-slope}
\frac{dQ}{dd}(d(t))
=
\frac{\Cov_{\pi_t}(J,r)}
{t\Var_{\pi_t}(r)}
=
\frac{\Cov_{\pi_t}(h(r),r)}
{t\Var_{\pi_t}(r)}.
\end{equation}
This expresses the gain slope through the proxy-conditional gold structure, the route to the finite-range residual bound in \Appref{app:finite-range-natural}{A.4}.  The proof of Theorem~\ref{thm:local} is in \Appref{app:local-law-proof}{A.2}.

\section{Estimation protocol}
\label{sec:protocol}

\runhead{Finite-item target}
For a finite item set $\calD$, the reference distribution is uniform and $\pi_t(i)\propto\exp(t r_i)$.  For each proxy--gold path, we trace the exact finite-item target curve
\[
Q(t)=\sum_i\pi_t(i)J_i,
\qquad
d(t)=\sum_i\pi_t(i)\log(|\calD|\pi_t(i)),
\]
and score predictions of the gain $\Delta(d)$ over the checked window $d\le0.10$.  Calibration moments give the plug-in PACE forecast
\begin{equation}
\label{eq:fit}
\widehat\Delta(d)=\hat A\sqrt d+\hat B d,
\end{equation}
labeled $A{+}B$ below, where hatted coefficients are calibration estimates (\Appref{app:prediction-concentration}{A.5}).  Human labels of evaluation items are used only to compute the target curve and scores.  In human-proxy paths, one rater supplies the proxy $r_i$ and the remaining raters define leave-one-rater-out (LOO) gold $J_i$.  In external-proxy paths, a model supplies $r_i$ and repeated human ratings define $J_i$.  The main experiments use a pooled tilt over the item set, which may change the prompt marginal; \Appref{app:bon-rank}{A} discusses prompt-conditional selection.  Gains and errors are in the native units of each human target (calibration-standardized for OOD transfer) and are compared within, not across, datasets.

\runhead{Prospective calibration protocol}
Each repeated-rating path is split by items into a 30\% calibration set and a disjoint 70\% evaluation set, and the split is resampled 50 times per path.  Calibration and evaluation items may share prompts and systems, so this tests held-out-item prediction within a slice rather than unseen-prompt or unseen-system generalization.  Repeats reuse the same path, so intervals are clustered by path.  The slices are SummEval (32 held-out-rater paths), eligible WMT MQM clusters (141 paths from 47 complete rater-set clusters), OpenMEVA-MAGS ROC/WP (10 paths), and a GPT-4o judge proxy \citep{openai2024gpt4o} with SummEval human targets (4 paths).  They keep only systems, raters, and dimensions supporting finite-item paths and disjoint gold (Table~\ref{tab:protocol-main}).

\runhead{Baselines}
Besides the $A$-only predictor $\hat A\sqrt d$ and zero gain, the reliability-only surrogate $\Anu\sqrt d$ with $\Anu=\sqrt2\,\hat\sigma_J^2/\hat\sigma_r$ uses only calibration variances of the human aggregate and the proxy.  It equals $A$ when $\rho=0$ in Eq.~\eqref{eq:proxy-decomp} but otherwise ignores proxy--gold alignment (\Appref{app:icc-random-effects}{A} and \Appref{app:scaling-a-vs-nu-filled}{B}).  Two data-driven baselines use the exact calibration-pool tilt curve: the \emph{calibration two-term fit} fits $a\sqrt d+bd$ to it by least squares, and the \emph{stored empirical curve} transfers it.

\runhead{Information regimes and flexible calibration}
A natural alternative estimates the proxy-conditional mean $h(r)\approx\E[J\mid r]$ on calibration items and integrates it along the tilt.  We use equal-frequency \emph{binned} means (piecewise constant, not monotone), a penalized cubic \emph{spline} (smooth), and \emph{isotonic} regression (nondecreasing), with the bin count and spline penalty chosen by calibration-internal cross-validation (\Appref{app:flexible}{B}).  The comparison depends on what is visible when the forecast is made.  In the \emph{pre-pool} regime, which matches the decision whether to build a larger pool, every method sees only the rated calibration items, and flexible calibrators integrate $\hat h$ over the calibration-pool tilt.  In the \emph{target-pool-aware} regime, the proxy scores of the target pool are also available, but not its human labels.  The same fitted $\hat h$ is then used in two ways.  \emph{Direct integration} evaluates $\sum_i\{\pi_{t(d)}(i)-\pi_0(i)\}\hat h(r_i)$ on the exact target-pool tilt.  \emph{Projection} inserts the pseudo-target $\hat h(r)$ into the target-pool moment formulas and predicts $A_h\sqrt d+B_hd$.  Their difference isolates the loss from truncating a calibrated curve to two terms.  Target labels only score forecasts.

\runhead{Evaluation metrics}
Endpoint MAE at $d=0.10$, pooled over valid splits, summarizes prediction at the largest checked radius.  CurveMAE averages absolute error over $d\in\{0.01,0.025,0.05,0.075,0.10\}$, first within each path over the (split, radius) cells that every compared method can predict and then equally over paths; paired differences use 2,000-draw path-clustered bootstrap intervals.  Curvature-sign agreement, a secondary and margin-sensitive diagnostic, compares $\operatorname{sign}(\hat B_{\rm cal})$ with the sign of a coefficient fitted to the exact evaluation curve over $0<d\le0.10$, not with the moment coefficient $B_{\rm eval}$ below.

\runhead{Prediction-error decomposition} Fix a local KL radius $D>0$ and suppose the evaluation population satisfies $\Delta_{\rm eval}(d)=A_{\rm eval}\sqrt d+B_{\rm eval}d+R_{\rm eval}(d)$ with $|R_{\rm eval}(d)|\le C_{\rm eval}d^{3/2}$ for all $d\in[0,D]$.
\begin{lemma}[Calibration prediction-error decomposition]
\label{lem:prediction-error}
For all $d\in[0,D]$,
\begin{equation}
\label{eq:prediction-error}
|\widehat\Delta(d)-\Delta_{\rm eval}(d)|
\le
\sqrt d\,|\hat A-A_{\rm eval}|
+d\,|\hat B-B_{\rm eval}|
+C_{\rm eval}d^{3/2}.
\end{equation}
Moreover, $|\hat A-A_{\rm eval}|\le |\hat A-A_{\rm cal}|+|A_{\rm cal}-A_{\rm eval}|$, and analogously for $B$.
\end{lemma}
The bound separates calibration estimation error, calibration--evaluation shift, and the local remainder.  Under the bounded i.i.d.\ model of \Appref{app:prediction-concentration}{A}, $|\hat A-A_{\rm cal}|,|\hat B-B_{\rm cal}|=O_p(n_{\rm cal}^{-1/2})$, so shift and the finite-range remainder are not removed by more calibration data.

\runhead{Selection rules and the checked window}
A planned rule that is not a raw-score tilt, such as Best-of-$N$ (BoN), is first translated into the distribution $p$ it induces on the reference pool $p_0$.  For BoN, $p_i$ is the exact probability that item $i$ wins among $N$ reference draws under uniform tie-breaking, which is computable from unlabeled proxy scores (\Appref{app:bon-rank}{A}).  We report
\begin{equation}
\label{eq:kl-ess}
d=\KL(p\Vert p_0),
\qquad
\mathrm{RelESS}(p,p_0)=\Big(\sum\nolimits_i p_i^2/p_{0,i}\Big)^{-1},
\end{equation}
the induced radius and the effective support that remains relative to $p_0$ \citep{kong1992note}.  For atomless scores, BoN gives $d_N=\log N-(N-1)/N$ and $\mathrm{RelESS}_N=(2N-1)/N^2$, so already $d_2\approx0.193$.  These quantities measure pressure and support, not gain, and rules with equal KL can have different gains.  We therefore use $A\sqrt d+Bd$ only for raw-score tilt, or for a rule with path-specific calibration, inside a checked and reachable window.  Otherwise the practitioner lowers the pressure (a smaller $N$, mixing with the reference, or a rank-tilt exponent), validates the rule directly, or abstains.

\runhead{Out-of-distribution (OOD) transfer}
For cross-slice transfer, a source prior $\bar\theta_{\rm source}$ over $\theta=(A,B)$ is combined with a 30\% target calibration estimate as $\hat\theta_{\rm shrink}=\lambda\hat\theta_{\rm target}+(1-\lambda)\bar\theta_{\rm source}$, with a fixed target-heavy $\lambda=0.75$ and calibration-only target standardization (\Appref{app:ood}{B}).

\FloatBarrier
\section{Experiments}
\label{sec:experiments}

The experiments ask four questions in turn (Sections~\ref{sec:exp-prepool}--\ref{sec:exp-bridge}); extended results are in \Appref{app:experiments}{B}.

\begin{table}[t]
\centering
\small
\setlength{\tabcolsep}{5pt}
\caption{Prospective prediction on held-out items (lower is better; native units, not comparable across columns).  (a) Endpoint MAE at $d=0.10$ over valid 30/70 splits, with a 95\% path-cluster bootstrap interval for the pooled $A{+}B$ MAE; zero-gain MAE equals the mean realized gain, and curvature sign compares $\operatorname{sign}(\hat B_{\rm cal})$ with a coefficient fitted to the exact evaluation curve.  (b) Pre-pool CurveMAE, a different metric: every method sees only the calibration items, on the common valid cells (11 SummEval cells at $d=0.10$ are unreachable for calibration-pool methods).  Paired intervals are in \Appref{app:flexible}{B}.  GPT-4o has four paths.}
\label{tab:prediction-main}
\begin{adjustbox}{max width=\textwidth}
\begin{tabular}{lrrrr}
\toprule
 & SummEval human & GPT-4o proxy & WMT MQM & OpenMEVA-MAGS \\
\midrule
\multicolumn{5}{@{}l}{\textit{(a) Endpoint MAE at $d=0.10$}} \\
Valid splits & 1,589 & 200 & 7,050 & 500 \\
$A{+}B$ & \textbf{0.0174} & \textbf{0.0118} & \textbf{0.0846} & \textbf{0.0148} \\
\quad 95\% interval & [0.0158, 0.0193] & [0.0104, 0.0125] & [0.0793, 0.0901] & [0.0141, 0.0156] \\
$A$-only & 0.0310 & 0.0165 & 0.1298 & 0.0162 \\
$\Anu$ & 0.0880 & 0.1508 & 0.6230 & 0.0207 \\
Zero gain & 0.2147 & 0.0680 & 0.8524 & 0.3327 \\
Curvature-sign agreement & 90.7\% & 97.5\% & 91.0\% & 91.8\% \\
\midrule
\multicolumn{5}{@{}l}{\textit{(b) Pre-pool CurveMAE over $d\in\{0.01,0.025,0.05,0.075,0.10\}$}} \\
Paths / common cells & 32 / 7,989 & 4 / 1,000 & 141 / 35,250 & 10 / 2,500 \\
$A{+}B$ & \textbf{0.01185} & \textbf{0.00806} & \textbf{0.05921} & \textbf{0.00991} \\
$A$-only & 0.01803 & 0.01030 & 0.08039 & 0.01058 \\
\midrule
Calibration two-term fit & 0.01170 & 0.00805 & 0.05843 & \textbf{0.00984} \\
Stored empirical curve & 0.01171 & 0.00805 & \textbf{0.05842} & \textbf{0.00984} \\
\midrule
Binned & 0.03010 & 0.00878 & 0.07356 & 0.01583 \\
Spline & 0.01275 & \textbf{0.00803} & 0.06255 & 0.01002 \\
Isotonic & \textbf{0.01168} & 0.00827 & 0.05882 & \textbf{0.00984} \\
\bottomrule
\end{tabular}
\end{adjustbox}
\end{table}

\subsection{Prospective prediction before the target pool exists}
\label{sec:exp-prepool}

\runhead{Endpoint prediction}
At $d=0.10$ (Table~\ref{tab:prediction-main}a), the PACE forecast has lower endpoint MAE than $A$-only, zero gain, and $\Anu$ on SummEval, eligible WMT MQM clusters, OpenMEVA-MAGS, and the GPT-4o judge proxy.  In the GPT-4o setting, reweighting held-out items toward higher judge scores yields a mean human-score gain of $0.0680$; no model is trained and no response is generated.  On SummEval and GPT-4o, radius sweeps over $0.005\le d\le0.20$ keep $A{+}B$ below the baselines, and calibration-size curves largely stabilize by 30--50\% calibration (\Appref{app:accept-boost}{B}).

\runhead{Flexible calibration under matched information}
Table~\ref{tab:prediction-main}b compares methods that all see only the calibration split.  $A{+}B$ has lower CurveMAE than binned calibration in all four settings and than splines in three, and its paired intervals against splines exclude zero on SummEval and WMT.  Isotonic regression, the calibration two-term fit, and the stored empirical curve have CurveMAE similar to $A{+}B$, with a significant reduction only on WMT (isotonic minus $A{+}B$: $-0.00039$, 95\% interval $[-0.00070,-0.00009]$).  Gains from flexibility are therefore small and inconsistent: the PACE forecast is competitive rather than dominant, while remaining analytic and interpretable.

\runhead{External proxies and transfer}
With an open reward model (Skywork-Reward-Llama-3.1-8B-v0.2 \citep{liu2024skywork}) or an open judge (Qwen2.5-32B-Instruct \citep{qwen2025qwen25}) as the SummEval proxy, adding $B$ reduces endpoint MAE by 27.1\% and 24.5\% relative to $A$-only, while the calibration two-term fit is marginally lower than $A{+}B$ for both (\Appref{app:open-proxies}{B}).  For three frozen reward models as SummEval proxies, adding $B$ lowers the 50-split endpoint MAE in all 12 model--dimension cells, and the calibration fit remains competitive (\Appref{app:rm-example}{B}).  In each split, $\hat A\sqrt d$ gives the first-order forecast and $\hat Bd$ its curvature correction; Figure~\ref{fig:rm-worked} shows the mean coherence forecasts and prediction errors over all 50 splits.  Across slices, source-only coefficients beat zero gain on all three transfer routes, but target calibration reduces their error on every route (e.g., SummEval$\rightarrow$WMT $0.0436\rightarrow0.0259$), so source coefficients are useful priors rather than invariants (\Appref{app:ood}{B}).

\runhead{Sanity checks}
On the same sample, theory and fitted coefficients correlate at least $0.90$ for $B$ and essentially one for $A$ across five repeated-rating slices (\Appref{app:direct-details}{B}).  SummEval human-proxy paths have median minimum tilted variance $0.631$ and ESS $1{,}338$ on $d\le0.10$ (\Appref{app:stability-bon-filled}{B}).  Random and shuffled proxies drive $A$ and the realized gain close to zero on average, whereas $\Anu$ stays large (\Appref{app:proxy-perturbation-filled}{B}).

\begin{table}[t]
\centering
\small
\caption{Target-pool-aware CurveMAE: the proxy scores of the target pool are available, but its human labels are not.  For each calibrator, the same fitted $\hat h$ is used by projection (the two-term moment formula applied to $(\hat h(r),r)$ on the target pool) and by direct integration (the exact tilt integral of $\hat h$).  The pre-pool $A{+}B$ row, from Table~\ref{tab:prediction-main}b, sees less information and is a reference, not a matched competitor.  Metric, support, and units as in Table~\ref{tab:prediction-main}b; paired intervals are in \Appref{app:flexible}{B}.}
\label{tab:target-aware}
\begin{tabular}{llrrrr}
\toprule
$\hat h$ & Predictor & SummEval & GPT-4o & WMT & OpenMEVA \\
 & Paths / cells & 32 / 7,989 & 4 / 1,000 & 141 / 35,250 & 10 / 2,500 \\
\midrule
 & Pre-pool $A{+}B$ (reference) & \textbf{0.01185} & \textbf{0.00806} & \textbf{0.05921} & \textbf{0.00991} \\
\midrule
Binned & Projection & 0.02782 & 0.00860 & 0.06557 & 0.01330 \\
 & Direct integration & 0.02914 & 0.00865 & 0.06843 & 0.01429 \\
\midrule
Spline & Projection & 0.01040 & 0.00780 & 0.05667 & 0.00726 \\
 & Direct integration & 0.01027 & \textbf{0.00777} & 0.05699 & 0.00728 \\
\midrule
Isotonic & Projection & 0.00959 & 0.00801 & 0.05316 & 0.00712 \\
 & Direct integration & \textbf{0.00909} & 0.00811 & \textbf{0.05211} & \textbf{0.00706} \\
\bottomrule
\end{tabular}

\end{table}

\subsection{Matched information: two-term projection versus direct integration}
\label{sec:exp-matched}

When the proxy scores of the target pool are available (Table~\ref{tab:target-aware}), isotonic and spline calibrators improve on the pre-pool $A{+}B$ reference, which sees less information, except direct isotonic on GPT-4o ($0.00811$ vs.\ $0.00806$); binned calibration does not.  For a fixed calibrator, projection and direct integration differ in aggregate CurveMAE by $0.00002$--$0.00286$ across the 12 dataset--calibrator cells, and their predictions differ pointwise by $0.00041$--$0.00850$ on average.  The choice of calibrator matters more: on SummEval, the binned projection is $0.01824$ above the isotonic projection.  These are descriptive differences, not uniform bounds, and an isotonic $\hat h$ cannot represent a downturn.  Within the checked window, the two moment-derived coefficients are an adequate local summary of a calibrated curve, and replacing the two-term curve by cubic or spline fits changes CurveMAE by at most $2.1\times10^{-5}$ (\Appref{app:flexible}{B}).

\begin{figure}[t]
\centering
\includegraphics[width=\textwidth]{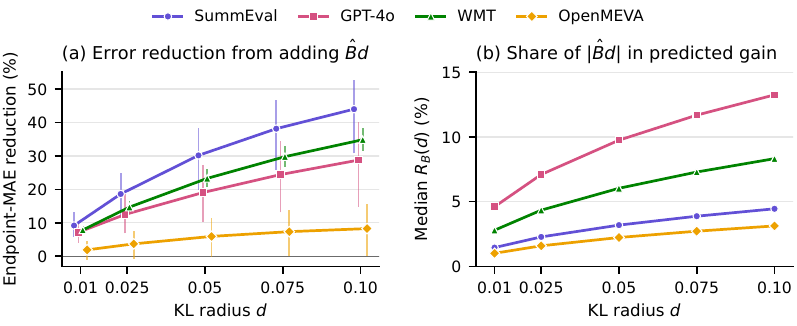}
\caption{What adding $B$ contributes in the pre-pool setting (30\% calibration, 50 splits per path, held-out exact tilt gain; panel (a) uses the common cells of Table~\ref{tab:prediction-main}b).  (a) Relative endpoint-MAE reduction $1-\mathrm{MAE}_{A+B}(d)/\mathrm{MAE}_{A\text{-only}}(d)$ on identical cells, a within-dataset ratio.  Error bars are 95\% path-cluster bootstrap intervals; GPT-4o has 4 paths and OpenMEVA 10, and OpenMEVA's intervals include zero for $d\le0.05$.  (b) Median share $R_B(d)=|\hat Bd|/(|\hat A\sqrt d|+|\hat Bd|)$ of the predicted gain over all splits, a magnitude share rather than net $B$ or an error reduction.  The improvement is aggregate: $A{+}B$ has lower CurveMAE than $A$-only on 55.6--67.0\% of splits.}
\label{fig:ab-contribution}
\end{figure}

\subsection{What \texorpdfstring{$B$}{B} adds and when it is unreliable}
\label{sec:exp-b}

\runhead{Contribution}
In the matched pre-pool comparison, adding $\hat Bd$ to $\hat A\sqrt d$ reduces CurveMAE by 34.3\%, 21.7\%, 26.4\%, and 6.4\% on SummEval, GPT-4o, WMT, and OpenMEVA (Table~\ref{tab:prediction-main}b).  The OpenMEVA paired interval includes zero, so we report both predictors there.  The endpoint reduction grows with $d$ in every setting and reaches 8--44\% at $d=0.10$ (Figure~\ref{fig:ab-contribution}a), whereas the median share of $\hat Bd$ in the predicted gain magnitude is only 1--13\% (Figure~\ref{fig:ab-contribution}b).  The curvature term is thus a small correction to the forecast that removes a systematic error growing with the radius.

\runhead{Reliability of $\hat B$}
With 30\% calibration, the median coefficient-mismatch term $d|\hat B-B_{\rm eval}|$ at $d=0.10$ is $0.0029$--$0.0317$ and shrinks from 5\% to 30\% calibration, and $\hat B$ agrees in sign with the held-out moment coefficient on 89.4--97.0\% of splits.  Within-split bootstrap draws keep the sign of $\hat B$ in 96.7--100\% of draws (dataset medians), yet 25--53\% of split-level intervals include zero, so sign stability does not establish nonzero curvature.  Gold quality matters: on SummEval, a second expert rating reduces the mean absolute error of $\hat B$ from $0.0618$ to $0.0428$, and injected rating noise removes the advantage of $A{+}B$ at 1.0 rating SD for experts and 0.5 SD for crowd workers.  Proxy-side outliers can destabilize $\hat B$; after 1\% positive 8-SD proxy contamination, its correlation with the clean OpenMEVA estimate is $-0.11$.  When the interval for $B$ spans zero or the scores are heavy-tailed, report both $A{+}B$ and $A$-only, recalibrate after robust preprocessing, or narrow the window (\Appref{app:b-stability}{B}).

\begin{table}[t]
\centering
\small
\caption{LoRA-DPO bridge: mean $\pm$ sample SD over three training seeds, after averaging the 4/3/3 checkpoints per seed whose mean prompt KL is at most $0.10$.  Gains are in silver-target units.  \emph{Realized}: self-normalized DPO/reference likelihood-ratio reweighting of fixed pools of 32 candidates per prompt, relative to uniform; it is not a fresh-generation, human-rated, or exact full-policy gain.  \emph{Idealized}: raw-score exponential tilt of the same pool at the same prompt radius.  Predictors use common calibration-only coefficients.  This is a restricted predictor set; all 46 comparator variants are in \Appref{app:dpo}{B}.  Lower MAE is better.}
\label{tab:dpo}
\begin{tabular}{lcc}
\toprule
 & Realized & Idealized \\
Mean $\pm$ SD over 3 seeds & likelihood-ratio gain & tilt gain \\
\midrule
Mean KL radius $d$ & \multicolumn{2}{c}{0.03968 $\pm$ 0.00136} \\
Mean silver gain & 0.02081 $\pm$ 0.00062 & 0.04197 $\pm$ 0.00017 \\
Idealized minus realized gain & \multicolumn{2}{c}{0.02116 $\pm$ 0.00074} \\
\midrule
\multicolumn{3}{l}{\textit{MAE of calibration-only prediction}} \\
$A{+}B$ & \textbf{0.01721} $\pm$ 0.00116 & 0.00395 $\pm$ 0.00047 \\
$A$-only & 0.01905 $\pm$ 0.00120 & \textbf{0.00210} $\pm$ 0.00053 \\
Zero & 0.02081 $\pm$ 0.00062 & 0.04197 $\pm$ 0.00017 \\
$A_\nu$ & 0.02487 $\pm$ 0.00129 & 0.00371 $\pm$ 0.00061 \\
\bottomrule
\end{tabular}

\end{table}

\subsection{Discrete selection and actual optimization}
\label{sec:exp-bridge}

\runhead{Best-of-$N$}
At $N=2$, the median induced radius (Eq.~\eqref{eq:kl-ess}) is already outside the checked window in every setting ($0.166$--$0.193$).  Across 191 paths and eleven values of $N$ from 2 to 64, only 4 of 2,101 BoN points lie inside $d\le0.10$, all at $N=2$ on SummEval expert paths where 85--89\% of items share the top proxy score, which caps the reachable KL.  There, $A{+}B$ overestimates the BoN gain by $0.020$ on average versus $0.088$ for $A$-only (descriptive, $n=4$).  Prompt-conditional BoN at $N=2$ leaves about 88\% of prompts outside the window.  For smooth rules that sweep the window continuously, mixtures with BoN and rank tilts, $A{+}B$ stays within $0.001$ of a calibration fit in every setting and beats $A$-only on three of four datasets but not on OpenMEVA (\Appref{app:stability-bon-filled}{B}).  Because these rules do not follow raw scores, these same-KL comparisons are descriptive scope checks.

\runhead{LoRA-DPO bridge}
We train Qwen2.5-3B-Instruct with LoRA-DPO \citep{hu2022lora,rafailov2023dpo} on 968 preference pairs built from UltraFeedback prompts \citep{cui2024ultrafeedback}, whose chosen and rejected responses are the highest- and lowest-scoring of eight candidates under Skywork-Reward-Llama-3.1-8B-v0.2, the optimized proxy $r$.  A separate frozen reward model, OpenAssistant DeBERTa-v3-large \citep{openassistant2023reward,kopf2023openassistant,he2023debertav3}, supplies a silver target $J$ that is never used in training.  Coefficients are calibrated on 32 reference-policy candidates for each of 99 usable prompts and averaged.  On 200 disjoint evaluation prompts, the policy reweights 32 fixed candidates by the normalized DPO/reference likelihood ratio $w_x$, with radius $d_x=\KL(w_x\Vert\mathrm{Unif}_{32})$.  The realized gain in Table~\ref{tab:dpo} is positive but about half of the idealized tilt gain at the same radii, so the optimizer realizes only part of the local frontier.  Against realized gain, $A{+}B$ has lower MAE than $A$-only in every seed and at all ten checkpoints.  Against the idealized frontier, $A$-only is closer: both forecasts underpredict it, and the negative $B$ moves $A{+}B$ toward the optimizer's shortfall, so part of its realized-gain advantage is error cancellation.  Fourteen of 46 comparator variants, including flexible calibrators, have lower realized MAE (\Appref{app:dpo}{B}).  This is a finite-candidate, model-target bridge with 8.7\% of prompt radii above 0.10, not a human-evaluated RLHF validation.

\section{Discussion and limitations}
\label{sec:discussion}

\runhead{Interpretation}
PACE turns a small repeated-rating calibration matrix into a local forecast of the human-gain curve induced by gently optimizing a fixed proxy: $A$ is first-order alignment and $B$ the leading curvature on the positive tilt branch, and reliability determines neither.  When target-pool proxy scores and a flexible calibrator are available, the same two-term summary can be read off the calibrated curve; when they are not, the PACE coefficients remain a competitive pre-pool forecast.

\runhead{Practical use and choice of window}
Before stronger proxy pressure is applied, a practitioner (i) fixes the proxy, human target, reference pool, planned rule, and minimum useful gain; (ii) estimates $\hat A,\hat B$ and their joint bootstrap uncertainty on repeated-rating calibration items; (iii) computes the induced radius, tilted variance, and RelESS of the planned rule on an unlabeled, proxy-scored pilot; and (iv) chooses the largest radius $D$ that is reachable, inside the checked window, passes the support checks, and keeps a one-sided simultaneous 95\% lower gain band, including a residual allowance, above the minimum useful gain at every checked $d\le D$; we have not validated the coverage of such a band.  If the band crosses the threshold, the result is inconclusive; more ratings help only when sampling error, not shift, dominates.  If reachability or support fails, the pressure is reduced, the rule is validated directly, or the proxy is not used.  The window and tolerances are application-specific; $d\le0.10$ is our checked window, not a constant.

\runhead{Cost and recalibration}
With $n_{\rm cal}$ calibration items rated $k$ times, the overhead over single labeling is $n_{\rm cal}(k-1)$ judgments per target dimension, paid once per calibration.  In the example, a second rating of the 300 items costs 300 extra judgments.  The ratings can be reused across proxies scored on the same items while the human target and item distribution are unchanged, and the benchmarks analyzed here already contain them.  Apart from a verified positive affine rescaling of $r$, a changed proxy, judge prompt, scoring rule, or human target requires fresh calibration.  For a shift in prompts or outputs alone, a small repeated-rating pilot can check whether the source coefficients lead to the same decision.  The protocol is not validated without repeated ratings, and single labels or pairwise preferences are left to future work.

\runhead{Scope and reporting}
A gain forecast is interpretable only relative to the rating matrix, target definition, scalar proxy, item distribution, dimension, and checked KL window.  The theorem concerns a pooled raw-score tilt of a fixed pool; BoN and DPO enter as bridges, and for interactive RLHF or other rules outside the checked setting, re-estimating the coefficients alone is not sufficient.  The judge and reward-model results concern fixed configurations, not general evaluators.  Reproducibility differs by experiment; the GPT-4o API calls and the DPO LoRA weights are not released (\Appref{app:repro-assets}{B}).  \Appref{app:failure-attribution}{B} gives a reporting checklist and follow-up checks for failures.

\section*{Acknowledgments and Disclosure of Funding}
This research was supported by the ``Advanced GPU Utilization Support Program'' funded by the Government of the Republic of Korea (Ministry of Science and ICT) (No.~02-26-01-0303); by a grant of the Korea Health Technology R\&D Project through the Korea Health Industry Development Institute (KHIDI), funded by the Ministry of Health \& Welfare, Republic of Korea (grant number: RS-2025-02307233); and by the Institute of Information \& Communications Technology Planning \& Evaluation (IITP) grant funded by the Korea government (MSIT) (No.~RS-2025-25463302, Global Talent Recruitment for AI-powered Drug Discovery).
The authors declare no competing interests.

\appendix
\setcounter{figure}{0}
\setcounter{table}{0}
\makeatletter
\@addtoreset{figure}{section}
\@addtoreset{table}{section}
\renewcommand{\thefigure}{\Alph{section}.\arabic{figure}}
\renewcommand{\thetable}{\Alph{section}.\arabic{table}}
\renewcommand{\theHfigure}{\Alph{section}.\arabic{figure}}
\renewcommand{\theHtable}{\Alph{section}.\arabic{table}}
\makeatother
\clearpage
\section{Proofs and technical extensions}
\label{app:proofs}

Throughout the appendix, $\mu=\pi_{\mathrm{ref}}$ is a probability measure and
\[
\frac{d\pi_t}{d\mu}(x)=\exp\{t r(x)-\Lambda(t)\},\qquad
\Lambda(t)=\log\E_\mu e^{t r(x)}.
\]
Write $Q(t)=\E_{\pi_t}[J]$, $R(t)=\E_{\pi_t}[r]$, and $d(t)=\KL(\pi_t\Vert\mu)$.  The finite-item case used in the body of the paper is the special case $\mu=\mathrm{Unif}(\calD)$.  All results below specialize directly by replacing $\mu$-integrals with empirical sums.

\subsection{KL-regularized frontier proof}
\label{app:variational-proof}

We prove Proposition~\ref{prop:variational} in the more general appendix notation, replacing the finite-item reference distribution $\piref$ by a probability reference measure $\mu$.

\begin{proof}[Proof of Proposition~\ref{prop:variational}]
Assume $\E_\mu[(1+|r|)e^{tr}]<\infty$ and let $Z_t=\E_\mu e^{tr}$. Then $d\pi_t/d\mu=e^{tr}/Z_t$ has finite $\E_{\pi_t}|r|$ and KL. For any $\pi\ll\mu$ with $\E_\pi|r|<\infty$ and finite KL,
\[
\KL(\pi\Vert\pi_t)
=\E_\pi\log\frac{d\pi}{d\mu}-t\E_\pi r+\log Z_t
=\KL(\pi\Vert\mu)-t\E_\pi r+\log Z_t.
\]
Rearranging gives
\[
t\E_\pi r-\KL(\pi\Vert\mu)=\log Z_t-\KL(\pi\Vert\pi_t)\le \log Z_t,
\]
with equality iff $\KL(\pi\Vert\pi_t)=0$, i.e., $\pi=\pi_t$ almost surely.  In particular, the unconstrained optimizer is strictly unique for finite $t$.

For the constrained-frontier statement, restrict to radii in the image of an interval on which $r$ is non-constant, $\Var_{\pi_t}(r)>0$, and $d(t)$ is strictly increasing.  The constrained primal $\sup_\pi\E_\pi[r]$ subject to $\KL(\pi\Vert\mu)\le D$ has Lagrangian $\E_\pi[r]-\lambda(\KL(\pi\Vert\mu)-D)$ with $\lambda\ge0$, and the scalarized objective in Eq.~\eqref{eq:variational} corresponds to the equivalent rescaling $\E_\pi[r]-t^{-1}\KL(\pi\Vert\mu)$, so the multiplier on the KL constraint is $\lambda=1/t$.  By the standard $I$-projection / Donsker--Varadhan variational argument \citep{csiszar1975idivergence,donsker1976asymptotic}, the constraint is active on this regular interval, and the optimizer agrees with $\pi_t$.  The same uniqueness conclusion as above holds for finite $t$.  Specializing $\mu=\piref$ recovers the main-text statement.  Degeneracy arises only at terminal radii that put all mass on the maximum-score support, which are outside the local regime used in this paper.
\end{proof}

\subsection{Proof of the local coefficient law}
\label{app:local-law-proof}

\begin{theorem}[Local coefficient law, detailed]
Assume $\sigma_r^2=\Var_\mu(r)>0$ and either $J,r$ are bounded or the joint cumulant generating function of $(J,r)$ is $C^4$ in a neighborhood of the origin.  Let $\kappa_r=\kappa_3(r,r,r)$ and $\kappa_{Jrr}=\kappa_3(J,r,r)$.  Then along the positive tilt branch,
\[
Q(d)=Q(0)+\sqrt2\frac{\Cov_\mu(J,r)}{\sigma_r}\sqrt d+
\left(\frac{\kappa_{Jrr}}{\sigma_r^2}-\frac{2\Cov_\mu(J,r)\kappa_r}{3\sigma_r^4}\right)d+O(d^{3/2}).
\]
If proxy and gold share a scale and $r-\E r=(J-\E J)+\varepsilon$, then $\Cov_\mu(J,r)=\sigma_J^2+\Cov_\mu(J,\varepsilon)$, recovering the proxy-error form used for interpretation.
\end{theorem}

\begin{proof}
The proof has four steps: (i) expand the bivariate cumulant generating function to extract $Q(t)$, (ii) expand $d(t)$ as a power series in $t$, (iii) invert $d(t)$ along the positive branch to obtain $t(d)$, and (iv) compose to obtain $Q(d)=Q(t(d))$ and read off the coefficients.

\smallskip
\noindent\emph{Step 1 (bivariate cumulant expansion of $Q$).}  Let $G(s,t)=\log\E_\mu\exp\{sJ+tr\}$.  The boundedness or $C^4$ assumption ensures $G$ is $C^4$ in a neighborhood of $(0,0)$, so its mixed partials $\partial_s^a\partial_t^b G(0,0)$ for $a+b\le3$ exist and equal the joint cumulants $\kappa_{a+b}(J^{\otimes a},r^{\otimes b})$.  Differentiating in $s$ at $s=0$ gives the moment generating identity $\partial_sG(0,t)=\E_{\pi_t}[J]=Q(t)$.  Hence
\begin{align*}
Q(t)
&=\partial_sG(0,t)
=\partial_sG(0,0)+t\,\partial_t\partial_sG(0,0)+\tfrac{t^2}{2}\,\partial_t^2\partial_sG(0,0)+R_Q(t)\\
&=\E_\mu[J]+t\,\Cov_\mu(J,r)+\tfrac{t^2}{2}\,\kappa_3(J,r,r)+R_Q(t),
\end{align*}
where $|R_Q(t)|\le C_Q t^3$ for some $C_Q$ depending only on the local fourth-order cumulants of $(J,r)$ (an explicit Taylor remainder bound under the $C^4$ assumption).  Since $Q(0)=\E_\mu[J]$, this gives
\[
Q(t)-Q(0)=\Cov_\mu(J,r)\,t+\tfrac12\kappa_{Jrr}\,t^2+O(t^3).
\]

\smallskip
\noindent\emph{Step 2 (power series for $d(t)$).}  Apply the same Taylor expansion to $\Lambda(t)=G(0,t)$ in $t$ alone:
\[
\Lambda(t)=\E_\mu[r]\,t+\tfrac{\sigma_r^2}{2}t^2+\tfrac{\kappa_r}{6}t^3+R_\Lambda(t),\qquad |R_\Lambda(t)|\le C_\Lambda t^4.
\]
Then $d(t)=t\Lambda'(t)-\Lambda(t)$ gives, after combining terms,
\[
d(t)=\tfrac{\sigma_r^2}{2}\,t^2+\tfrac{\kappa_r}{3}\,t^3+R_d(t),\qquad |R_d(t)|\le C_d t^4.
\]
Since $\sigma_r^2>0$, $d(t)>0$ for $t>0$ small, and $d'(t)=t\Var_{\pi_t}(r)>0$ on a positive interval.  In particular $d$ is strictly increasing there.

\smallskip
\noindent\emph{Step 3 (inversion of $d$ along the positive branch).}  Because $d(t)\sim(\sigma_r^2/2)t^2$ with $\sigma_r>0$, on the positive branch $d^{1/2}\sim(\sigma_r/\sqrt2)\,t$, suggesting the ansatz
\[
t(d)=a\sqrt d+b\,d+R_t(d),\qquad |R_t(d)|\le C_t d^{3/2},
\]
with $a>0$.  Substituting into the expansion for $d(t)$ and matching coefficients of $d$ and $d^{3/2}$:
\begin{align*}
\text{coefficient of }d:\quad &\tfrac{\sigma_r^2}{2}\,a^2=1
\quad\Longrightarrow\quad a=\frac{\sqrt2}{\sigma_r}\ \ \text{(positive root)},\\
\text{coefficient of }d^{3/2}:\quad &\sigma_r^2\,a\,b+\tfrac{\kappa_r}{3}a^3=0
\quad\Longrightarrow\quad b=-\frac{\kappa_r}{3}\frac{a^2}{\sigma_r^2}=-\frac{2\kappa_r}{3\sigma_r^4}.
\end{align*}
The negative root $a=-\sqrt2/\sigma_r$ corresponds to the negative tilt branch (decreasing $E_\pi[r]$), which is excluded by the positive-branch hypothesis.  This determines $a,b$ uniquely.  Set $x(t)=t\sqrt{d(t)/t^2}$, extending $d(t)/t^2$ continuously at zero. Then $x(t)=\sigma_rt/\sqrt2+O(t^2)$ has nonzero derivative at zero. Taylor inversion in $x=\sqrt d$ yields $R_t(d)=O(x^3)=O(d^{3/2})$; one does not apply the implicit-function theorem directly to $d-d(t)$ at zero.

\smallskip
\noindent\emph{Step 4 (composition).}  Substitute $t=t(d)$ into the expansion for $Q(t)-Q(0)$.  The leading two terms give
\[
\Cov_\mu(J,r)\Big[\tfrac{\sqrt2}{\sigma_r}\sqrt d-\tfrac{2\kappa_r}{3\sigma_r^4}d\Big]
+\tfrac{\kappa_{Jrr}}{2}\Big[\tfrac{\sqrt2}{\sigma_r}\sqrt d\Big]^2+O(d^{3/2}).
\]
The squared term is $\tfrac{\kappa_{Jrr}}{2}\cdot\tfrac{2}{\sigma_r^2}\,d=\tfrac{\kappa_{Jrr}}{\sigma_r^2}d$.  Cross-terms involving $\sqrt d\cdot d$ are already $O(d^{3/2})$ and absorbed into the remainder.  The cubic remainder $|R_Q(t(d))|\le C_Q t(d)^3=O(d^{3/2})$ since $t(d)=O(\sqrt d)$.  Collecting,
\[
Q(d)-Q(0)
=\frac{\sqrt2\,\Cov_\mu(J,r)}{\sigma_r}\sqrt d
+\Big[\frac{\kappa_{Jrr}}{\sigma_r^2}-\frac{2\Cov_\mu(J,r)\,\kappa_r}{3\sigma_r^4}\Big]d
+O(d^{3/2}),
\]
which gives the stated $A$ and $B$.
\end{proof}

\subsection{Standardized coefficient decomposition}
\label{app:standardized-decomposition}

Let $Z=(r-\E r)/\sigma_r$ and $\bar J=J-\E J$.  Since $\Cov(J,r)=\sigma_r\E[\bar JZ]$, the alignment coefficient satisfies
\[
A=\sqrt2\frac{\Cov(J,r)}{\sigma_r}=\sqrt2\,\E[\bar JZ].
\]
Moreover,
\[
\kappa_{Jrr}=\E[(J-\E J)(r-\E r)^2]=\sigma_r^2\E[\bar JZ^2],
\qquad
\kappa_r=\sigma_r^3\E[Z^3].
\]
Substituting these identities into Theorem~\ref{thm:local} gives
\[
B=\E[\bar JZ^2]-\frac23\E[\bar JZ]\E[Z^3].
\]
Thus $A$ is first-order standardized proxy--gold alignment and $B$ decomposes into a two-sided proxy-tail human-quality moment and a proxy-skewness correction.

\subsection{Skewness correction as a change of coordinate}
\label{app:skewness}

The skewness term in $B$ arises because the tilt path is indexed by the KL radius.  Let $s=t\sigma_r$ be the standardized tilt coordinate, with $Z$ and $\bar J$ as in Appendix~\ref{app:standardized-decomposition} and all expectations under $\mu$.  Steps~1 and~2 of Appendix~\ref{app:local-law-proof}, rewritten in $s$, give
\[
Q(s)-Q(0)=\E[\bar JZ]\,s+\tfrac12\E[\bar JZ^2]\,s^2+O(s^3),
\qquad
d(s)=\tfrac12s^2+\tfrac13\E[Z^3]\,s^3+O(s^4).
\]
Inverting the second expansion on the positive branch gives $s(d)=\sqrt{2d}-\tfrac23\E[Z^3]\,d+O(d^{3/2})$, and substituting into the first gives $A=\sqrt2\,\E[\bar JZ]$ and $B=\E[\bar JZ^2]-\tfrac23\E[\bar JZ]\E[Z^3]$.  The first term of $B$ is the second-order quality response in the tilt coordinate.  It equals $\Cov(J,Z^2)$, a two-sided moment to which large-$|Z|$ items on both sides contribute, so it is not an upper-tail-only quantity.  The second term converts from $s$ to $d$.  With positive alignment, right skew makes a given positive $s$ consume more KL, so a fixed $d$ permits a smaller $s$ and the correction lowers $B$ relative to its first term; left skew gives a positive correction.  The correction vanishes when either alignment or skewness is zero.

Table~\ref{tab:skewness} and Figure~\ref{fig:skewness} measure when the correction matters.  The reported share, $|{\rm corr}|/(|\E[\bar JZ^2]|+|{\rm corr}|)$ with ${\rm corr}=-\tfrac23\E[\bar JZ]\E[Z^3]$, divides by the sum of the absolute components rather than by $|B|$, which avoids shares above 100\% when the components cancel.  In a synthetic family with a standardized skew-normal proxy and a fixed conditional mean, the share is 0.25\% in the symmetric case and rises to about 29\% for strong right skew and 52\% for strong left skew.  For the full SummEval coherence pool, it is 26\% with ROUGE-1 \citep{lin2004rouge} and 34\% with GPT-4o as the proxy.  Changing the skew-normal shape also changes higher moments, so the synthetic family is an illustration, not an isolated causal effect of skewness, and two real proxies do not establish the range of skewness found in practice.

\begin{table}[H]
\centering
\small
\caption{Size of the skewness correction in $B$.  Synthetic rows: $Z$ is a standardized skew-normal variable with shape $a$ (20,000 items each), with conditional mean $\E[J\mid Z]=1.3Z+0.6(Z^2-1)$ and proxy $r=5+2Z$.  SummEval rows: the full pool of 1,600 coherence items with ROUGE-1 or GPT-4o as the proxy (descriptive full-pool coefficients, not calibration predictors).  Share is $|{\rm corr}|/(|\E[\bar JZ^2]|+|{\rm corr}|)$, not a fraction of $B$.}
\label{tab:skewness}
\begin{adjustbox}{max width=\textwidth}
\begin{tabular}{lrrrrr}
\toprule
Setting & $\mathbb{E}[Z^3]$ & $\mathbb{E}[\bar J Z^2]$ & $-\tfrac23\mathbb{E}[\bar J Z]\mathbb{E}[Z^3]$ & $B$ & Share (\%) \\
\midrule
\multicolumn{6}{l}{\textit{Synthetic skew-normal family, shape $a$ (20,000 items each)}} \\
$a=-30$ & $-$0.992 & 0.4308 & 0.4661 & 0.8969 & 52.0 \\
$a=-8$ & $-$0.938 & 0.4692 & 0.4609 & 0.9301 & 49.6 \\
$a=-3$ & $-$0.650 & 0.6095 & 0.3944 & 1.0039 & 39.3 \\
$a=-1$ & $-$0.130 & 1.0605 & 0.1062 & 1.1667 & 9.1 \\
$a=0$ & 0.004 & 1.2256 & $-$0.0031 & 1.2225 & 0.3 \\
$a=+1$ & 0.138 & 1.4217 & $-$0.1270 & 1.2947 & 8.2 \\
$a=+3$ & 0.679 & 2.3698 & $-$0.7732 & 1.5966 & 24.6 \\
$a=+8$ & 0.946 & 2.9262 & $-$1.1780 & 1.7482 & 28.7 \\
$a=+30$ & 0.958 & 2.8736 & $-$1.1974 & 1.6762 & 29.4 \\
\midrule
\multicolumn{6}{l}{\textit{SummEval coherence, full pool (1,600 items), proxy:}} \\
ROUGE-1 & 0.082 & 0.0304 & $-$0.0108 & 0.0196 & 26.2 \\
GPT-4o & $-$0.240 & $-$0.1533 & 0.0802 & $-$0.0731 & 34.4 \\
\bottomrule
\end{tabular}

\end{adjustbox}
\end{table}

\begin{figure}[H]
\centering
\includegraphics[width=\linewidth]{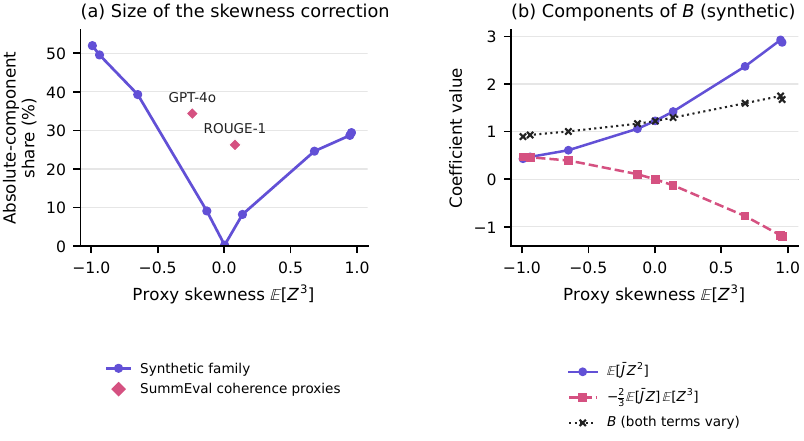}
\caption{When the skewness correction matters.  (a) Correction share as a function of proxy skewness $\E[Z^3]$ for the synthetic skew-normal family (line) and the two SummEval coherence proxies (points).  (b) The two components of $B$ and their sum for the synthetic family: with positive alignment, right skew contributes a negative correction and left skew a positive one. Other moments also vary across this family, so total $B$ need not decrease with skewness.}
\label{fig:skewness}
\end{figure}

\subsection{Finite-range natural coordinate and global slope}
\label{app:finite-range-natural}

\begin{proposition}[Quadratic comparability]
\label{prop:quad}
If $0<v_-\le\Var_{\pi_t}(r)\le v_+<\infty$ for all $t\in[0,T]$, then
\[
\frac{v_-}{2}t^2\le d(t)\le\frac{v_+}{2}t^2.
\]
Consequently, $\sqrt{2/v_+}\sqrt d\le t(d)\le\sqrt{2/v_-}\sqrt d$ for attainable $d\in[0,d(T)]$.
\end{proposition}

\begin{proof}
Since $d(t)=t\Lambda'(t)-\Lambda(t)$ and $\Lambda''(t)=\Var_{\pi_t}(r)$, we have $d'(t)=t\Var_{\pi_t}(r)$ and $d(0)=0$.  Integration from $0$ to $t$ proves the bounds.
\end{proof}

\begin{proposition}[Exact global derivative identity]
\label{prop:global}
Assume $\E_\mu[(1+|J|)(1+r^2)e^{tr}]<\infty$ on $(0,t_{\max})$.  Then $\Lambda(t)<\infty$ and
\[
Q'(t)=\Cov_{\pi_t}(J,r),\quad R'(t)=\Var_{\pi_t}(r),\quad d'(t)=t\Var_{\pi_t}(r).
\]
If $h(u)=\E_\mu[J\mid r=u]$, then $\E_{\pi_t}[J\mid r]=h(r)$ and, for $t>0$,
\[
\frac{dQ}{dd}(d(t))=\frac{\Cov_{\pi_t}(h(r),r)}{t\Var_{\pi_t}(r)}
\]
whenever the denominator is nonzero.
\end{proposition}

\begin{proof}
The integrability assumption majorizes $|J e^{tr}|$, $|r e^{tr}|$, $|Jr e^{tr}|$, and $r^2 e^{tr}$ uniformly on a neighborhood of $t$, justifying differentiation under the integral.  For integrable $f$, $\E_{\pi_t}f=\E_\mu[f e^{tr}]/\E_\mu[e^{tr}]$, so $d\E_{\pi_t}f/dt=\Cov_{\pi_t}(f,r)$.  Taking $f=J$ and $f=r$ gives the first identities.  Since $d(t)=t\Lambda'(t)-\Lambda(t)$, differentiating gives $d'(t)=t\Lambda''(t)=t\Var_{\pi_t}(r)$.  The density ratio $e^{tr}/Z_t$ is a function of $r$ alone, so tilting reweights $r$-fibers without changing the conditional law of $J$ given $r$.
\end{proof}

\begin{theorem}[Finite-range residual bound]
\label{thm:finite}
Fix $D>0$, and let $T_D$ satisfy $d(T_D)=D$.  Suppose the tilt family is well-defined on $[0,T_D]$, proxy variance is bounded above and below there, and fourth cumulants $|\kappa_4^{\pi_t}(J,r,r,r)|$ and $|\kappa_4^{\pi_t}(r,r,r,r)|$ are uniformly bounded.  Let $A,B$ be the coefficients in Theorem~\ref{thm:local}.  Then there exists $C_D<\infty$ such that
\[
|Q(d)-Q(0)-A\sqrt d-Bd|\le C_D d^{3/2}\qquad\forall d\in[0,D].
\]
\end{theorem}

\begin{proof}
The result extends the local expansion to a finite checked range by combining the local residual with continuity away from zero.  We spell out the argument because it clarifies why the constant is range-dependent.

\smallskip
\noindent\emph{Step 1 (local residual near zero).}  Let $K=\sup_{t\in[0,T_D]}\max\{|\kappa_4^{\pi_t}(J,r,r,r)|,|\kappa_4^{\pi_t}(r,r,r,r)|\}<\infty$.  Taylor's theorem with remainder applies to the marginal expansions of $Q(t)$ and $\Lambda(t)$ on a small interval $[0,t_0]\subset[0,T_D]$.  The third and fourth derivatives are the corresponding tilted cumulants, bounded by $K$ on that interval.  Hence
\[
Q(t)-\Big[Q(0)+\Cov_\mu(J,r)\,t+\tfrac12\kappa_{Jrr}\,t^2\Big]
=O(t^3),
\]
\[
\Lambda(t)-\Big[\E_\mu[r]\,t+\tfrac{\sigma_r^2}{2}t^2+\tfrac{\kappa_r}{6}t^3\Big]
=O(t^4).
\]
Using $d(t)=t\Lambda'(t)-\Lambda(t)$ gives
\[
d(t)=\tfrac{\sigma_r^2}{2}t^2+\tfrac{\kappa_r}{3}t^3+O(t^4).
\]
Since $\sigma_r>0$, the positive branch is locally invertible and
\[
t(d)=\frac{\sqrt2}{\sigma_r}\sqrt d-\frac{2\kappa_r}{3\sigma_r^4}d+O(d^{3/2}).
\]
Composing the expansion for $Q(t)$ with this inverse gives, on $0\le d\le d_0:=d(t_0)$,
\[
|Q(d)-Q(0)-A\sqrt d-Bd|\le C_0 d^{3/2}
\]
for some finite $C_0$.

\smallskip
\noindent\emph{Step 2 (compact tail of the checked range).}  If $d_0\ge D$, Step 1 proves the claim.  Otherwise, proxy variance is bounded below on $[0,T_D]$, so $d(t)$ is strictly increasing and the inverse $t(d)$ is continuous on $[d_0,D]$.  Therefore
\[
S(d)=Q(t(d))-Q(0)-A\sqrt d-Bd
\]
is continuous on the compact interval $[d_0,D]$.  Since $d^{3/2}$ is bounded away from zero there,
\[
C_1=\sup_{d\in[d_0,D]} \frac{|S(d)|}{d^{3/2}}<\infty.
\]
Taking $C_D=\max\{C_0,C_1\}$ yields
\[
|Q(d)-Q(0)-A\sqrt d-Bd|\le C_D d^{3/2}\qquad\forall d\in[0,D].
\]
The constant $C_D$ depends on the data distribution and on the checked range.  The theorem therefore does not assert that the two-coefficient approximation is numerically tight for a chosen deployment radius.  Appendix~\ref{app:stability-bon-filled} reports finite-range diagnostics that measure this tightness empirically.
\end{proof}

\subsection{Calibration prediction error and finite-sample concentration}
\label{app:prediction-concentration}

The main text states the prediction-error decomposition as Lemma~\ref{lem:prediction-error}.  The proof is immediate but useful for interpreting failures.  If
\[
\Delta_{\rm eval}(d)=A_{\rm eval}\sqrt d+B_{\rm eval}d+R_{\rm eval}(d),
\qquad |R_{\rm eval}(d)|\le C_{\rm eval}d^{3/2},
\]
and $\widehat\Delta(d)=\hat A\sqrt d+\hat B d$, then
\[
\widehat\Delta(d)-\Delta_{\rm eval}(d)
=(\hat A-A_{\rm eval})\sqrt d+(\hat B-B_{\rm eval})d-R_{\rm eval}(d).
\]
Taking absolute values gives Eq.~\eqref{eq:prediction-error}.  The additional split into estimation error and calibration--evaluation shift is the triangle inequality, e.g.
\[
|\hat A-A_{\rm eval}|\le |\hat A-A_{\rm cal}|+|A_{\rm cal}-A_{\rm eval}|.
\]
The sign-margin statement follows because if the error bound is smaller than $|\Delta_{\rm eval}(d_0)|$, the prediction cannot cross zero.

\paragraph{Plug-in coefficient estimators.}  For a calibration sample $(J_i,r_i)_{i=1}^n$, define empirical moments $\bar J=n^{-1}\sum_i J_i$, $\bar r=n^{-1}\sum_i r_i$, $\hat\sigma_r^2=n^{-1}\sum_i(r_i-\bar r)^2$, and $\hat Z_i=(r_i-\bar r)/\hat\sigma_r$, $\hat{\bar J}_i=J_i-\bar J$.  The plug-in estimators corresponding to Eq.~\eqref{eq:standardized-coefficients} are
\[
\hat A=\sqrt2\,\frac{1}{n}\sum_i\hat{\bar J}_i\hat Z_i,
\qquad
\hat B=\frac{1}{n}\sum_i\hat{\bar J}_i\hat Z_i^2-\frac23\!\left(\frac{1}{n}\sum_i\hat{\bar J}_i\hat Z_i\right)\!\!\left(\frac{1}{n}\sum_i\hat Z_i^3\right).
\]
Equivalently, $\hat A=\sqrt2\,\widehat{\Cov}(J,r)/\hat\sigma_r$ and $\hat B=\widehat{\kappa_{Jrr}}/\hat\sigma_r^2-\tfrac23\,\widehat{\Cov}(J,r)\widehat{\kappa_r}/\hat\sigma_r^4$, where all covariances and third-order cumulants are the same $n$-normalized empirical central moments displayed above.  The implementation uses these plug-in empirical moments, not finite-sample unbiased $k$-statistics; this convention matches the finite-item target population used for the exact tilt curves.

\begin{proposition}[Coefficient concentration]
\label{prop:coefficient-concentration}
In an idealized i.i.d. calibration model, let $(J_i,r_i)_{i=1}^n$ be i.i.d., with $|J_i|\le M_J$, $|r_i|\le M_r$, and population proxy standard deviation $\sigma_r\ge s>0$.  Use the plug-in estimators of Eq.~\eqref{eq:standardized-coefficients} when $\hat\sigma_r>0$, and set $\hat A=\hat B=0$ on a zero-variance sample for this probability statement. The empirical protocol instead excludes such samples.  Then there exist constants $C_1,C_2>0$, depending only on $(M_J,M_r,s)$, such that for sufficiently small fixed $u$,
\[
\Pr(|\hat A-A_{\rm cal}|>u)\le C_1e^{-C_2 n u^2},
\qquad
\Pr(|\hat B-B_{\rm cal}|>u)\le C_1e^{-C_2 n u^2}.
\]
In particular $|\hat A-A_{\rm cal}|,|\hat B-B_{\rm cal}|=O_p(n^{-1/2})$.
\end{proposition}

\begin{proof}
The coefficients in Eq.~\eqref{eq:standardized-coefficients} are smooth functions of bounded empirical moments as long as the proxy standard deviation is bounded away from zero.  Hoeffding concentration controls each required empirical moment at rate $n^{-1/2}$, and a delta-method/Lipschitz argument on the event $\{\hat\sigma_r\ge s/2\}$ gives the displayed bounds, with constants depending on $(M_J,M_r,s)$.  The probability of the complementary event is itself exponentially small: applying Hoeffding separately to $n^{-1}\sum_i r_i^2$ and $\bar r$, using $\hat\sigma_r^2=n^{-1}\sum_i r_i^2-\bar r^2$ and $\sigma_r^2\ge s^2$, gives $\Pr(\hat\sigma_r<s/2)\le C_3 e^{-C_4 n}$ for constants $C_3,C_4$ depending on $(M_J,M_r,s)$, which can be absorbed into $C_1,C_2$ to yield the unconditional bounds stated.
\end{proof}

\begin{corollary}[Uniform concentration of the local prediction curve]
For fixed $D>0$,
\[
\sup_{0\le d\le D}|(\hat A-A_{\rm cal})\sqrt d+(\hat B-B_{\rm cal})d|
=O_p\left(\frac{\sqrt D+D}{\sqrt n}\right).
\]
\end{corollary}

\paragraph{Boundedness and tail robustness.}  Proposition~\ref{prop:coefficient-concentration} is a finite-sample concentration statement under bounded scores and a nondegenerate proxy variance.  The local coefficient law itself allows a local cumulant-generating-function condition, but the plug-in estimator can be sensitive to extreme proxy values because $B$ contains third-order moments.  In finite empirical rating matrices with nonzero proxy variance the coefficients are defined after the reported preprocessing, but the diagnostic should still be accompanied by tail checks when raw proxy scores are unbounded or heavy-tailed.  Appendix~\ref{app:accept-boost} reports a winsorization, rank-normalization, and controlled-outlier audit: mild winsorization and gold-side outliers largely preserve $B$ correlation, but sign agreement can change. Proxy outliers can perturb curvature, motivating robust preprocessing and revalidation.

\subsection{Best-of-\texorpdfstring{$N$}{N} as rank-score tilt}
\label{app:bon-rank}

This appendix records the rank-score tilt formulation of best-of-$N$ used as a scope check in the body.  Recent theoretical and algorithmic analyses of BoN policies provide complementary perspectives \citep{beirami2025theoretical,jinnai2025regularized}.

Let $F(u)=\mu(r\le u)$.  If $r$ is atomless under $\mu$, the probability that a sample with score $r(x)$ is selected as the maximum among $N$ iid draws is proportional to $F(r(x))^{N-1}$.  Therefore
\[
\frac{d\pi_N^{\rm BoN}}{d\mu}(x)=NF(r(x))^{N-1}
=\exp\{(N-1)\log F(r(x))+\log N\}.
\]
Thus BoN is an exponential tilt on the rank score $s(x)=\log F(r(x))$, not on the raw score $r(x)$.

For the atomless KL calculation, under $X\sim\mu$ the random variable $U=F(r(X))$ is uniform on $[0,1]$.  Hence
\[
\KL(\pi_N^{\rm BoN}\Vert\mu)
=\E_{\pi_N^{\rm BoN}}\log\{NU^{N-1}\}
=\log N+(N-1)\E_{U\sim\mathrm{Beta}(N,1)}\log U.
\]
Since $\E_{\mathrm{Beta}(N,1)}\log U=-1/N$, we get
\[
 d_N^{\rm BoN}=\log N-\frac{N-1}{N}.
\]
In particular $d_2^{\rm BoN}=\log 2-1/2\approx0.193$.

\paragraph{Tie-corrected discrete form.}  For discrete or tied proxy scores, let $F_-(u)=\mu(r<u)$.  We derive the selection density under uniform tie-breaking among maximizers.  Draw $X_1,\ldots,X_N$ i.i.d.\ from $\mu$, and select $X_1$ as the BoN winner whenever (i) no other draw exceeds $r(X_1)$, and (ii) among draws tied with $r(X_1)$, $X_1$ is chosen uniformly.  Conditional only on $r(X_1)=u$, write $p:=F(u)-F_-(u)$ for the atom mass.  The joint event that exactly $j$ of the $N-1$ competitors tie at score $u$, the remaining $N-1-j$ are below $u$, and none exceed $u$ has probability $\binom{N-1}{j}F_-(u)^{N-1-j}p^j$.  Uniform tie-breaking then selects $X_1$ with probability $1/(j+1)$ on that joint event.  Summing these unconditional joint probabilities gives
\[
\Pr(X_1\text{ selected}\mid r(X_1)=u)
=\sum_{j=0}^{N-1}\binom{N-1}{j}F_-(u)^{N-1-j}p^{j}\cdot\frac{1}{j+1}
=\frac{F(u)^{N}-F_-(u)^{N}}{N\,p},
\]
using the binomial identity $\sum_{j=0}^{N-1}\binom{N-1}{j}a^{N-1-j}b^{j}/(j+1)=((a+b)^{N}-a^{N})/(Nb)$ with $a=F_-(u)$ and $b=p$.  Multiplying by $N$ (Bayes' rule with $\Pr(X_1\text{ selected})=1/N$) gives the density ratio $w_N(u)=\big(F(u)^N-F_-(u)^N\big)/\big(F(u)-F_-(u)\big)$ on the atom case $p>0$.  Letting $p\downarrow0$ recovers $w_N(u)=NF(u)^{N-1}$.  Hence
\[
 w_N(u)=
 \begin{cases}
 \dfrac{F(u)^N-F_-(u)^N}{F(u)-F_-(u)}, & F(u)>F_-(u),\\[0.6em]
 NF(u)^{N-1}, & F(u)=F_-(u),
 \end{cases}
\]
and $d\pi_N^{\rm BoN}/d\mu(x)=w_N(r(x))$.  The atomless formula is recovered when $\mu$ has no atoms in $r$.

\paragraph{Relative ESS, softened rules, and saturation.}
For a rule-induced distribution $p$ on a pool with reference $p_0$, $\mathrm{RelESS}(p,p_0)=(\sum_ip_i^2/p_{0,i})^{-1}=1/\{1+\chi^2(p\Vert p_0)\}\in(0,1]$.  For a uniform reference on $m$ items it equals $\mathrm{ESS}/m$ with $\mathrm{ESS}=1/\sum_ip_i^2$, so the effective number of candidates is $m\,\mathrm{RelESS}$.  For atomless scores, $\E_\mu[w_N^2]=\int_0^1N^2u^{2N-2}\,du=N^2/(2N-1)$, so $\mathrm{RelESS}_N=(2N-1)/N^2$, which is $0.750$, $0.438$, and $0.234$ at $N=2,4,8$, where $d_N=0.193$, $0.636$, and $1.204$.  For atomless scores, $q_\lambda\propto\mu F(r)^\lambda$ has density $(\lambda+1)F^\lambda=w_{\lambda+1}$ and equals BoN when $N=\lambda+1$ is an integer. The finite-pool experiments instead use $U_i=\operatorname{rank}_{\rm avg}(r_i)/m$, with one-based average ranks for ties, and $q_{\lambda,i}\propto p_{0,i}U_i^\lambda$. This empirical rank tilt need not equal finite-pool BoN, even without ties.  The mixture $p_\alpha=(1-\alpha)\mu+\alpha\pi^{\rm BoN}_N$ reaches arbitrarily small radii and satisfies $\chi^2(p_\alpha\Vert\mu)=\alpha^2\chi^2(\pi^{\rm BoN}_N\Vert\mu)$, hence $\mathrm{RelESS}(p_\alpha)=1/\{1+\alpha^2\chi^2(\pi^{\rm BoN}_N\Vert\mu)\}$; its KL has no closed form and is computed exactly.  Ties also cap the radius.  As $N\to\infty$, BoN and the pooled tilt both converge to the uniform distribution on the $n_{\rm top}$ items tied at the maximum score, so $d\le\log(m/n_{\rm top})$, and a proxy with many tied top scores keeps BoN at small radii even for moderate $N$ (Appendix~\ref{app:stability-bon-filled}).  All of these quantities are computable from unlabeled proxy scores.  Lying inside the checked window is necessary but not sufficient, because rules with equal KL can have different gains.

\paragraph{Pooled versus prompt-conditional optimization.}  The exponential-tilt target in the main theorem is a pooled finite-item tilt over the reference distribution supplied to the diagnostic.  If items are pairs $(x,y)$ with prompt $x$ and candidate response $y$, the pooled tilt
\[
\pi_t(x,y)\propto \pi_{\rm ref}(x,y)\exp\{t r(x,y)\}
\]
can change the prompt marginal because prompts with larger partition functions receive more mass.  Standard RLHF, reranking, and BoN deployments more often apply proxy pressure conditionally on each prompt,
\[
\pi_t(y\mid x)\propto \pi_{\rm ref}(y\mid x)\exp\{t r(x,y)\},\qquad \pi_t(x)=\pi_{\rm ref}(x),
\]
which equals the pooled tilt only when the prompt-level normalizer is constant across prompts or when the reference distribution is already defined on a fixed prompt-conditional slice.  The paper therefore treats the smooth pooled tilt as the theorem target and reports prompt-conditional BoN as a scope audit rather than as a direct theorem validation; Table~\ref{tab:prompt-conditional-bon-scope} and Figure~\ref{fig:prompt-conditional-bon-scope} quantify this distinction.

\subsection{Measurement aggregation toy theorem}
\label{app:aggregation}

Let $Y_{ij}=J_i+\eta_{ij}$ with $\eta_{ij}$ i.i.d.\ across $j$, independent of $J_i$, with $\E\eta=0$ and $\Var(\eta)=\tau^2$. Assume finite third moments and $\sigma_J^2+\tau^2>0$.  Let the $k$-rater proxy be $r_k=J+\bar\eta_k$, where $\bar\eta_k=k^{-1}\sum_{j=1}^k\eta_{ij}$.

\begin{proposition}[First and second moments under aggregation]
\label{prop:aggregation-A}
For this additive-noise model with $\sigma_J^2+\tau^2>0$,
\[
\sigma_{r_k}^2=\sigma_J^2+\tau^2/k,
\qquad
\Cov(J,r_k)=\sigma_J^2.
\]
Therefore
\[
A_k=\sqrt2\frac{\sigma_J^2}{\sqrt{\sigma_J^2+\tau^2/k}}
\]
is nondecreasing in $k$, strictly increasing when both $\tau^2>0$ and $\sigma_J^2>0$, and converges to $\sqrt2\,\sigma_J$.
\end{proposition}

\begin{proposition}[Third cumulants under aggregation]
\label{prop:aggregation-B}
Let $\kappa_J=\kappa_3(J,J,J)$ and $\kappa_\eta=\kappa_3(\eta,\eta,\eta)$.  Under the i.i.d.\ rater-noise assumption above,
\[
\kappa_3(J,r_k,r_k)=\kappa_J,
\qquad
\kappa_3(r_k,r_k,r_k)=\kappa_J+\kappa_\eta/k^2.
\]
Consequently
\[
B_k=\frac{\kappa_J}{\sigma_{r_k}^2}
-\frac{2\sigma_J^2(\kappa_J+\kappa_\eta/k^2)}{3\sigma_{r_k}^4},
\]
which need not improve monotonically with $k$.
\end{proposition}

The toy calculation matches the intervention experiments: rater aggregation reliably improves reliability and $A$, but local headroom also depends on curvature.

\subsection{ICC random-effects clarification}
\label{app:icc-random-effects}

We follow the standard intraclass-correlation conventions of \citet{shrout1979intraclass} and \citet{mcgraw1996forming}.  A standard random-effects model writes
\[
Y_{ij}=\mu+U_i+V_j+W_{ij},
\]
where $U_i\sim(0,\sigma_U^2)$ is item-level quality, $V_j\sim(0,\sigma_V^2)$ is rater bias, and $W_{ij}\sim(0,\sigma_W^2)$ is residual interaction/noise.  The three components are mutually independent.  Two single-rater reliability conventions appear in the body:
\[
\mathrm{ICC}(A,1)=\frac{\sigma_U^2}{\sigma_U^2+\sigma_V^2+\sigma_W^2},
\qquad
\mathrm{ICC}(C,1)=\frac{\sigma_U^2}{\sigma_U^2+\sigma_W^2},
\]
that is, ICC(A,1) is the absolute-agreement form (treating $V_j$ as variance to be explained), and ICC(C,1) is the consistency form (rater bias absorbed/conditioned out).  The corresponding $k$-rater versions replace $\sigma_W^2$ by $\sigma_W^2/k$ (and, for ICC(A,1), $\sigma_V^2$ by $\sigma_V^2/k$ when raters are crossed and averaged).  Thus, schematically,
\[
\mathrm{ICC}(C,k)=\frac{\sigma_U^2}{\sigma_U^2+\sigma_W^2/k},
\qquad
\mathrm{ICC}(A,k)=\frac{\sigma_U^2}{\sigma_U^2+\sigma_V^2/k+\sigma_W^2/k}.
\]
We use $\nu_k^2\approx\mathrm{Rel}_k$ as a generic reliability anchor.  In concrete tables, $\mathrm{Rel}_k$ is taken to be the reported ICC variant.  ICC(A,1) is the default absolute-agreement reliability anchor.  ICC(C,1) is reported separately to diagnose rater-bias or personalization gaps, not as the default $\Anu$ estimator.

In the prospective direct-proxy tables, the operational reliability-only baseline is the plug-in surrogate
\[
\Anu=\sqrt2\,\hat\sigma_J(\hat\sigma_J/\hat\sigma_r),
\]
computed on the calibration split for each proxy--gold path.  This is the empirical analogue of $\nu^2=\sigma_J^2/\sigma_r^2$ in the simple noise-only model $r=J+\eta$ with $\eta\perp J$, which corresponds to ICC(C,1) after identifying $\sigma_U^2$ with $\sigma_J^2$ and $\sigma_W^2$ with $\Var(\eta)$.  When an ICC estimate is converted to a reliability amplitude, we use the square root of its nonnegative part.

\subsection{Qualitative coefficient readout}
\label{app:regime-readout}

Table~\ref{tab:regime-main} lists qualitative interpretations of the local coefficient law in Eq.~\eqref{eq:local-law} for representative coefficient patterns.  The readout is local.  Finite-range deployment still requires the reachability, tilted-variance, and ESS checks of Section~\ref{sec:protocol}.

\begin{table}[tbp]
\centering
\small
\caption{Qualitative readout of the local coefficient law.  The table is only a local interpretation of Eq.~\eqref{eq:local-law}.  Finite-range deployment requires the reachability, tilted-variance, and ESS checks of Section~\ref{sec:protocol}.}
\label{tab:regime-main}
\begin{adjustbox}{max width=\textwidth}
\begin{tabular}{lll}
\toprule
Coefficient pattern & Local gain-curve interpretation & Practical implication \\
\midrule
$A>0,\ B<0$ & initial human gain with leading saturation & optimize only within a checked local window \\
$A>0,\ B\ge0$ & initially favorable proxy pressure & still verify finite-range tail behavior \\
$A\approx0$ & weak first-order alignment & sign claims need a margin or more calibration \\
$A<0$ & proxy pressure initially lowers human gold & do not use this proxy direction without redesign \\
large $|B|$ & strong tail or skewness effect & inspect selected items and reachable radii \\
\bottomrule
\end{tabular}
\end{adjustbox}
\end{table}

\FloatBarrier
\section{Experimental details and additional results}
\label{app:experiments}
\setlength{\floatsep}{6pt plus 2pt minus 2pt}
\setlength{\textfloatsep}{8pt plus 2pt minus 2pt}
\setlength{\intextsep}{8pt plus 2pt minus 2pt}

This appendix gives the experiment-level details behind the main tables and figures.  The goal is not to add new headline claims, but to make the repeated-rating construction, admissibility filters, robustness checks, and failure analyses inspectable.

\subsection{Reproducibility package and asset provenance}
\label{app:repro-assets}

Code, prepared inputs, and derived results are released at \url{https://github.com/pacifiic/local-proxy-diagnostics}.  The release regenerates the tables and figures of this paper from saved results; large files are stored compressed with recorded checksums, and a verification script checks the saved numerical outputs.  Reproducibility differs by experiment, and we distinguish three levels.
(i) \emph{GPT-4o judge.}  The released aggregate proxy scores reproduce all downstream analyses, but not the original API calls: an exact rerun would require raw per-call logs with the model snapshot, decoding parameters, and invalid-output and retry counts, which the aggregate does not contain (Appendix~\ref{app:gpt-repro}).
(ii) \emph{Open reward models and open judge} (Appendices~\ref{app:open-proxies} and~\ref{app:rm-example}).  The scored files reproduce all downstream analyses, and the model identifiers, resolved revisions, and scoring settings are recorded; we do not claim byte-level replay of the upstream scoring runs.
(iii) \emph{LoRA-DPO bridge} (Appendix~\ref{app:dpo}).  The pinned candidate pools and scores, code, seed schedule, and derived summaries are released, but the LoRA weights were not retained.

Large public source datasets are not redistributed when their terms require users to obtain them from the original providers; the release lists the download commands, pinned source revisions, and preparation steps.  All resampling uses explicit seeds.  Each experiment family records its numerical environment.  Results are reproduced in that environment; for example, the cross-validation tie order of a few DPO comparators can change selected hyperparameters on other hardware (Appendix~\ref{app:dpo}).

\paragraph{Compute.}  The experimental node had two Intel Xeon 6960P processors (144 physical cores / 288 threads), approximately 2.2 TiB of system RAM, and eight NVIDIA B200 GPUs (approximately 179 GiB of device memory each).  Human-rating analyses, including flexible calibration, coefficient stability, and selection-rule audits, are CPU-only.  The logged endpoint, transfer, and validation reproduction took approximately 4.25 wall-clock hours, with another 0.5 hours for the robustness checks of Appendix~\ref{app:accept-boost}.  These mostly single-process Python jobs used about 5--10 CPU-hours in total and less than 32 GB of RAM.  Candidate generation, reward-model and judge scoring, and LoRA-DPO training used one GPU per worker, allowing independent jobs to run in parallel.  Each of the three reward-model scorers in Appendix~\ref{app:rm-example} used one B200.

\paragraph{GPU runtimes.}  Generating the shared 17,600 DPO candidates with Qwen2.5-3B-Instruct took approximately 3.5 minutes on one GPU (0.058 GPU-hours).  The three-seed DPO sweep used six separately trained step budgets per seed (64, 90, 120, 150, 180, and 200), totaling 1.216 GPU-hours of training-loop time across 18 jobs, including those later excluded by the KL gate.  This excludes model loading and likelihood-ratio evaluation.  The five-pass Qwen2.5-32B-Instruct judge run produced 32,000 SummEval judgments in approximately 18.5 minutes on one GPU (0.309 GPU-hours, tensor parallelism 1).  Generation and judge timings include engine initialization and shutdown.  These logged timings cover the stated stages, rather than the full project GPU cost.

Table~\ref{tab:asset-provenance} lists the source, license or access terms where verified, and redistribution status for each external asset; when we could not verify a public license page for the source used here, the table marks the terms as not specified.

\begin{table}[H]
\centering
\scriptsize
\caption{Asset provenance, license/terms, and redistribution status.  Raw external datasets are cited in the paper and obtained from their original sources unless the reproducibility package explicitly includes a derived data product.  Users should follow the source license or access terms for each external asset.}
\label{tab:asset-provenance}
\begin{adjustbox}{max width=\textwidth}
\begin{tabular}{@{}>{\raggedright\arraybackslash}p{0.19\textwidth}>{\raggedright\arraybackslash}p{0.17\textwidth}>{\raggedright\arraybackslash}p{0.22\textwidth}>{\raggedright\arraybackslash}p{0.19\textwidth}>{\raggedright\arraybackslash}p{0.17\textwidth}@{}}
\toprule
Asset & Role in paper & Source or pinned access & License / terms & Redistribution status \\
\midrule
SummEval annotations & main repeated-rating slice & public SummEval source used by parser & MIT & raw data not redistributed \\
WMT MQM resources & eligible MQM repeated-rating clusters & \href{https://github.com/google/wmt-mqm-human-evaluation}{official repository}; revision in package notes & Apache-2.0 & raw data not redistributed \\
OpenMEVA-MAGS ROC/WP & story-generation repeated-rating slice & pinned source revision listed in package notes & license/terms not specified & raw data prepared locally \\
QGEval human files & direct coefficient-law check & user-supplied raw spreadsheets & license/terms not specified & raw data not redistributed \\
RewardBench margins & auxiliary geometry check & pinned public source revision listed in package notes & repo Apache-2.0; dataset card ODC-BY & generated locally from source \\
MT-Bench judgments & auxiliary split-half Bradley--Terry check & pinned public source revision listed in package notes & FastChat repo Apache-2.0; dataset card CC-BY-4.0 & generated locally from source \\
GPT-4o aggregate proxy scores & external-proxy experiment & derived aggregate included in the release & authors' derived artifact; model-service terms apply & derived artifact included \\
Skywork-Reward-Llama-3.1-8B-v0.2 & open reward-model proxy; DPO preference labels & Hugging Face, pinned revision & Skywork Community License (Llama~3.1 base) & weights not redistributed; derived scores included \\
ArmoRM-Llama3-8B-v0.1 & reward-model proxy in the three-model evaluation & Hugging Face, pinned revision & Llama~3 license & weights not redistributed; derived scores included \\
OpenAssistant reward-model-deberta-v3-large-v2 & reward-model proxy; DPO silver target & Hugging Face, pinned revision & MIT & weights not redistributed; derived scores included \\
Qwen2.5-32B-Instruct & open judge proxy & Hugging Face, pinned revision & Apache-2.0 & weights not redistributed; derived scores included \\
Qwen2.5-3B-Instruct & DPO reference and trained policy & Hugging Face & Qwen Research License & weights not redistributed; LoRA weights not retained \\
UltraFeedback (binarized) & DPO prompts & Hugging Face, pinned revision & MIT & prompt subset and derived candidate pools included \\
Analysis code, tables, and figures & reproduction and inspection & public code release & authors' release (MIT) & included with the release \\
\bottomrule
\end{tabular}
\end{adjustbox}
\end{table}

\subsection{Datasets, repeated-rating construction, and admissibility}
\label{app:data-admissibility}

The clean theorem-validation experiments require repeated ratings on the same item.  For an item $i$ and held-out rater $j$, the proxy is $r_i=Y_{ij}$ and the gold target is the leave-one-rater-out (LOO) mean $J_i=\frac{1}{m_i-1}\sum_{j'\ne j}Y_{ij'}$.  A \emph{path} is the exact tilt curve for one fixed proxy, human target, and analysis group, such as one held-out expert rater on one SummEval dimension; the 50 calibration/evaluation splits of a path are resamples of it, not additional paths.  This rule is deliberately restrictive: splits without repeated segment--system ratings are excluded because they cannot support a human proxy and a disjoint human gold signal on the same item.  Table~\ref{tab:protocol-main} summarizes the analysis-eligible slices, the subsets of each dataset that retain only the systems, language pairs, story sources, raters, and dimensions that support finite-item path construction, leave-one-rater-out or disjoint human gold, and repeated calibration/evaluation resampling.  The 11 skipped SummEval splits are those whose exact tilt curve cannot reach $d=0.10$ (Appendix~\ref{app:flexible}), a reachability failure rather than post-hoc outcome filtering.

\begin{table}[H]
\centering
\scriptsize
\setlength{\tabcolsep}{4pt}
\renewcommand{\arraystretch}{1.08}
\caption{Experimental protocol and split accounting for the endpoint and transfer analyses.  Paths are proxy--gold tilt paths, and valid splits are the calibration/evaluation resamples used for prospective prediction at $d=0.10$.  Human-proxy rows use repeated ratings to construct both the proxy view and the human target, and the GPT-4o row uses an external judge score as the proxy and repeated SummEval human ratings as the target.  Labels name analysis-eligible slices or transfer routes, not entire source benchmarks.}
\label{tab:protocol-main}
\begin{adjustbox}{max width=\textwidth}
\begin{tabular}{@{}>{\raggedright\arraybackslash}p{0.24\textwidth}>{\raggedright\arraybackslash}p{0.20\textwidth}>{\raggedright\arraybackslash}p{0.31\textwidth}>{\raggedright\arraybackslash}p{0.22\textwidth}@{}}
\toprule
Source and role & Paths / valid splits & Proxy and human target & Split or transform \\
\midrule
SummEval (prospective) & 32 paths, 1,589/1,600 valid & held-out human rater $\rightarrow$ leave-one-rater-out mean & 30\% item calibration. Native units \\
QGEval-human (validation) & 63 eligible paths & held-out human rater $\rightarrow$ LOO human mean & Raw scores; same-item direct validation \\
OpenMEVA-MAGS ROC/WP \newline(prospective + validation) & 10 prospective paths, 500/500 valid. 15 direct paths & held-out human rater $\rightarrow$ disjoint human aggregate & 30\% item calibration. Prospective: native units \\
WMT MQM/pSQM \newline(prospective + validation) & 141 MQM prospective paths, 7,050/7,050 valid. 73 direct paths & held-out rater $\rightarrow$ LOO segment--system mean & Repeated-rating filter. Prospective: native units \\
GPT-4o judge proxy \newline(external proxy) & 4 paths, 200/200 valid & GPT-4o judge score $\rightarrow$ SummEval human aggregate & 30\% item calibration. Native units \\
OOD transfer (cross-slice) & three source--target routes & source coefficient prior $\rightarrow$ target human gold & 30\% target calibration. Calibration-split $z$-scores \\
\bottomrule
\end{tabular}
\end{adjustbox}
\end{table}

\paragraph{Pairwise-preference data.}  Pairwise-preference datasets do not directly satisfy this requirement.  They provide comparisons rather than repeated scalar observations of the same item-level quality, so using them would require an additional utility model, such as a Bradley--Terry or Thurstone construction, before $J_i$ and $r_i$ can be put on a scalar finite-item tilt path.  Such utilities are useful engineering summaries, but they add structural assumptions that are outside the nonparametric repeated-rating validation protocol.  We therefore use pairwise-derived resources only as auxiliary geometry checks unless repeated scalar human ratings are available.

\paragraph{SummEval.} SummEval \citep{fabbri2021summeval} contains 1,600 evaluated summaries, four dimensions, and two rater pools: five crowd raters and three expert raters.  Table~\ref{tab:summeval-icc-detail} reports the ICC-style reliability decomposition used as the verifiability anchor.  Reliability varies sharply by dimension and pool.  Expert reliability is much higher than crowd reliability for consistency and fluency, moderately higher for coherence, and slightly lower for relevance under ICC(A,1).  This heterogeneity is why the paper treats verifiability as an observation-process quantity rather than as a benchmark constant.

\begin{table}[tbp]
\centering
\scriptsize
\caption{SummEval reliability and score statistics.  $G_{\rm pers}$ is the gap between consistency-style and absolute-agreement reliability.}
\label{tab:summeval-icc-detail}
\begin{adjustbox}{max width=\textwidth}
\begin{tabular}{llrrrrrrr}
\toprule
Dimension & Pool & Items & Raters & ICC(A,1) & ICC(C,1) & $G_{\rm pers}$ & Mean & Std. \\
\midrule
coherence & all & 1600 & 8 & 0.228 & 0.246 & 0.0178 & 3.662 & 1.070 \\
coherence & crowd & 1600 & 5 & 0.494 & 0.494 & 0.0000 & 3.811 & 0.926 \\
coherence & expert & 1600 & 3 & 0.573 & 0.631 & 0.0579 & 3.412 & 1.234 \\
consistency & all & 1600 & 8 & 0.228 & 0.276 & 0.0474 & 4.136 & 1.037 \\
consistency & crowd & 1600 & 5 & 0.495 & 0.495 & 0.0001 & 3.822 & 0.956 \\
consistency & expert & 1600 & 3 & 0.899 & 0.902 & 0.0027 & 4.660 & 0.950 \\
fluency & all & 1600 & 8 & 0.208 & 0.274 & 0.0660 & 4.075 & 0.997 \\
fluency & crowd & 1600 & 5 & 0.483 & 0.483 & 0.0000 & 3.716 & 0.925 \\
fluency & expert & 1600 & 3 & 0.727 & 0.734 & 0.0069 & 4.673 & 0.806 \\
relevance & all & 1600 & 8 & 0.235 & 0.244 & 0.0093 & 3.887 & 0.930 \\
relevance & crowd & 1600 & 5 & 0.478 & 0.478 & -0.0001 & 3.953 & 0.878 \\
relevance & expert & 1600 & 3 & 0.468 & 0.512 & 0.0437 & 3.777 & 1.001 \\
\bottomrule
\end{tabular}
\end{adjustbox}
\end{table}

\begin{table}[tbp]
\centering
\small
\caption{Expert--crowd reliability lift on SummEval.  The expert lift is dimension-dependent rather than uniform.}
\label{tab:expert-crowd-lift-detail}
\begin{adjustbox}{max width=0.88\textwidth}
\begin{tabular}{lrrrrr}
\toprule
Dimension & crowd ICC(A,1) & expert ICC(A,1) & lift & CI lower & CI upper \\
\midrule
coherence & 0.494 & 0.573 & 0.079 & 0.036 & 0.121 \\
consistency & 0.495 & 0.899 & 0.405 & 0.362 & 0.446 \\
fluency & 0.483 & 0.727 & 0.244 & 0.189 & 0.294 \\
relevance & 0.478 & 0.468 & -0.010 & -0.056 & 0.034 \\
\bottomrule
\end{tabular}
\end{adjustbox}
\end{table}

\begin{figure}[tbp]
\centering
\includegraphics[width=0.72\linewidth]{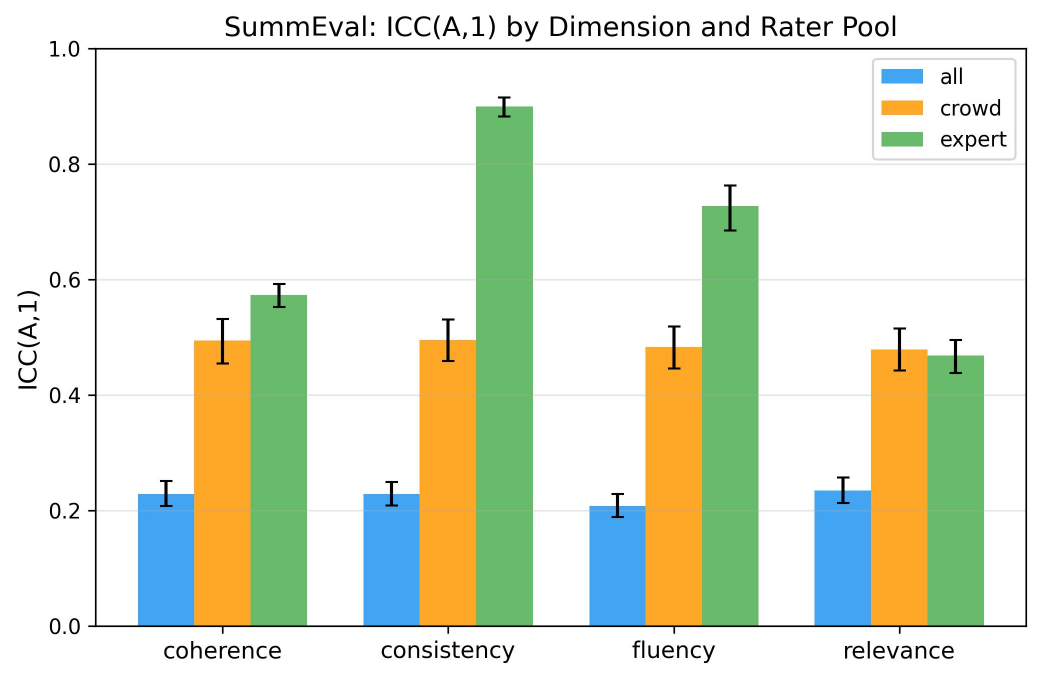}
\caption{SummEval reliability by dimension and rater pool.  The plot visualizes Table~\ref{tab:summeval-icc-detail}: reliability is a property of the observation protocol, not a single scalar attached to the benchmark.  Error bars, where shown, are bootstrap intervals for the ICC estimate under the split construction.}
\label{fig:icc-app-filled}
\end{figure}

\begin{figure}[tbp]
\centering
\includegraphics[width=0.72\linewidth]{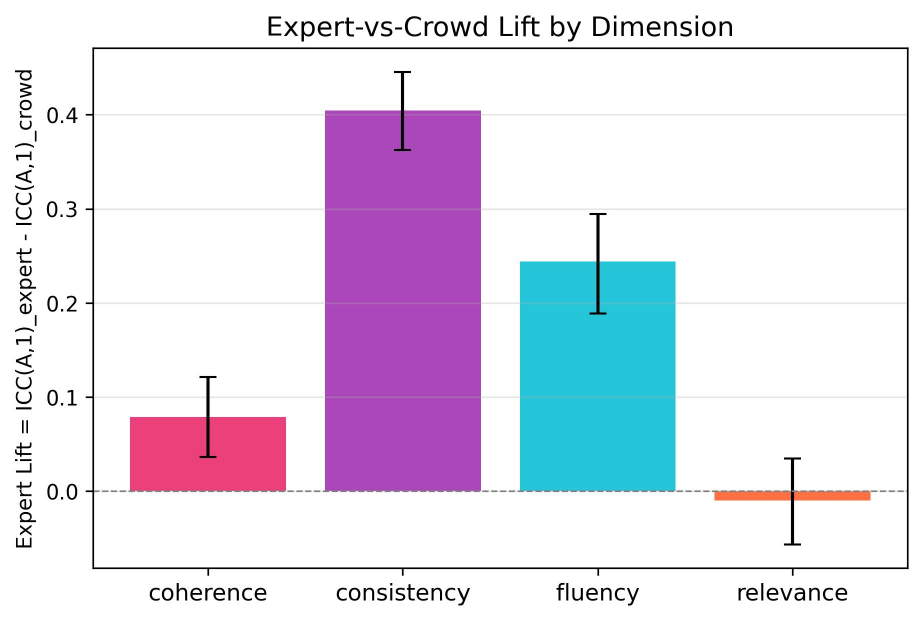}
\caption{Expert--crowd reliability lift.  Consistency and fluency show large measurement gains from expert raters.  Relevance does not, which motivates separating measurement improvement from preference heterogeneity.  Error bars show bootstrap intervals for the expert-minus-crowd ICC difference.}
\label{fig:expert-lift-app}
\end{figure}

\paragraph{WMT MQM/pSQM admissibility.} The raw WMT inventory \citep{freitag2021large,lommel2014mqm} contains many one-rating-per-item splits.  These are useful for benchmark evaluation, but not for our direct human tilt construction.  Table~\ref{tab:wmt-admissible-detail} lists the WMT splits that pass the repeated-rating filter and therefore enter the direct-tilt analysis.  For MQM, 8 of 19 inspected splits provide repeated ratings and yield 61 held-out-rater paths.  For pSQM, both inspected 2020 language-pair splits pass the filter and yield 12 paths.  The main prospective WMT row uses a stricter prepared MQM matrix with 47 clusters, each a set of segment--system items rated by the same complete set of at least three raters, and 141 held-out-rater paths.  It should be read as eligible WMT MQM clusters, not all WMT MQM.

\paragraph{WMT preprocessing correction.} The prospective cluster preparation treats quotation marks literally in the tab-delimited source files and includes document IDs in segment keys, which prevents dropped or merged records and collisions between document-local segment IDs.  An earlier preparation lacked both safeguards and produced 45 clusters and 139 paths instead of 47 and 141; every WMT result in this paper uses the corrected input, which lowers the endpoint $A{+}B$ MAE from 0.0958 to 0.0846 and preserves the ordering of the predictors.  The separate raw-data direct-validation loader already used document-aware keys and literal-tab parsing, so its 61 MQM and 12 pSQM paths are unaffected.  The two pipelines also use different scoring conventions: the prepared clusters use severity weights, whereas raw MQM validation additionally distinguishes error categories.  Their path counts therefore describe different analysis slices.

\begin{table}[H]
\centering
\scriptsize
\caption{WMT repeated-rating splits eligible to the direct-tilt analysis.  Splits with one rating per segment--system item are excluded by construction.  Split names follow the release file names, and \texttt{3r} marks files with three ratings per segment.  For the WMT 2023 entries, \texttt{3r} and \texttt{sxs\_3r} identify single-sided and side-by-side MQM annotations, respectively; side-by-side MQM presents two translations together while retaining separate error annotations for each \citep{song2025enhancing}.}
\label{tab:wmt-admissible-detail}
\begin{adjustbox}{max width=\textwidth}
\begin{tabular}{llrrrrrr}
\toprule
Family & Split & Paths & corr($A$) & corr($B$) & sign agree & median items & ICC(1,1) \\
\midrule
MQM & \texttt{generalMT2022\_\_ende\_\_mqm\_\_3r} & 7 & 0.9999 & 0.9627 & 7/7 & 8460 & 0.488 \\
MQM & \texttt{generalMT2022\_\_enzh\_\_mqm\_\_raw} & 6 & 0.9999 & 0.9460 & 6/6 & 13240 & 0.417 \\
MQM & \texttt{generalMT2023\_\_ende\_\_mqm\_\_3r} & 10 & 0.9995 & 0.9606 & 10/10 & 1755 & 0.638 \\
MQM & \texttt{generalMT2023\_\_ende\_\_mqm\_\_sxs\_3r} & 10 & 0.9998 & 0.9564 & 10/10 & 310 & 0.499 \\
MQM & \texttt{generalMT2023\_\_zhen\_\_mqm\_\_3r} & 8 & 0.9993 & 0.9061 & 8/8 & 7056 & 0.579 \\
MQM & \texttt{generalMT2023\_\_zhen\_\_mqm\_\_sxs\_3r} & 8 & 0.9998 & 0.9769 & 8/8 & 1480 & 0.457 \\
MQM & \texttt{newstest2020\_\_ende\_\_mqm\_\_raw} & 6 & 0.9969 & 0.9107 & 6/6 & 7110 & 0.410 \\
MQM & \texttt{newstest2020\_\_zhen\_\_mqm\_\_raw} & 6 & 1.0000 & 0.9783 & 6/6 & 9945 & 0.556 \\
\midrule
pSQM & \texttt{newstest2020\_\_ende\_\_psqm\_\_raw} & 6 & 1.0000 & 0.9237 & 6/6 & 7030 & 0.535 \\
pSQM & \texttt{newstest2020\_\_zhen\_\_psqm\_\_raw} & 6 & 1.0000 & 0.9997 & 6/6 & 9925 & 0.098 \\
\bottomrule
\end{tabular}
\end{adjustbox}
\end{table}

\paragraph{OpenMEVA, QGEval, and auxiliary sources.} For the OpenMEVA-MAGS annotations of stories from ROCStories (ROC) and WritingPrompts (WP) \citep{guan2021openmeva}, we use only the manually annotated human-rating files because the diagnostic needs repeated scalar ratings on the same story--model item.  Other OpenMEVA benchmark components are outside this analysis.  Within OpenMEVA, the prospective, OOD-transfer, and calibration-size analyses use the 10-path ROC/WP slice, while the direct-validation appendix tables report pooled, ROC, and WP scalar-rating paths where noted.  QGEval-human \citep{fu2024qgeval} is used only where repeated human dimensions permit direct coefficient validation.  RewardBench and MT-Bench are auxiliary geometry checks: RewardBench uses cross-model score margins, and MT-Bench uses split-half Bradley--Terry utilities rather than repeated scalar human ratings.

\begin{figure}[H]
\centering
\includegraphics[width=0.90\linewidth]{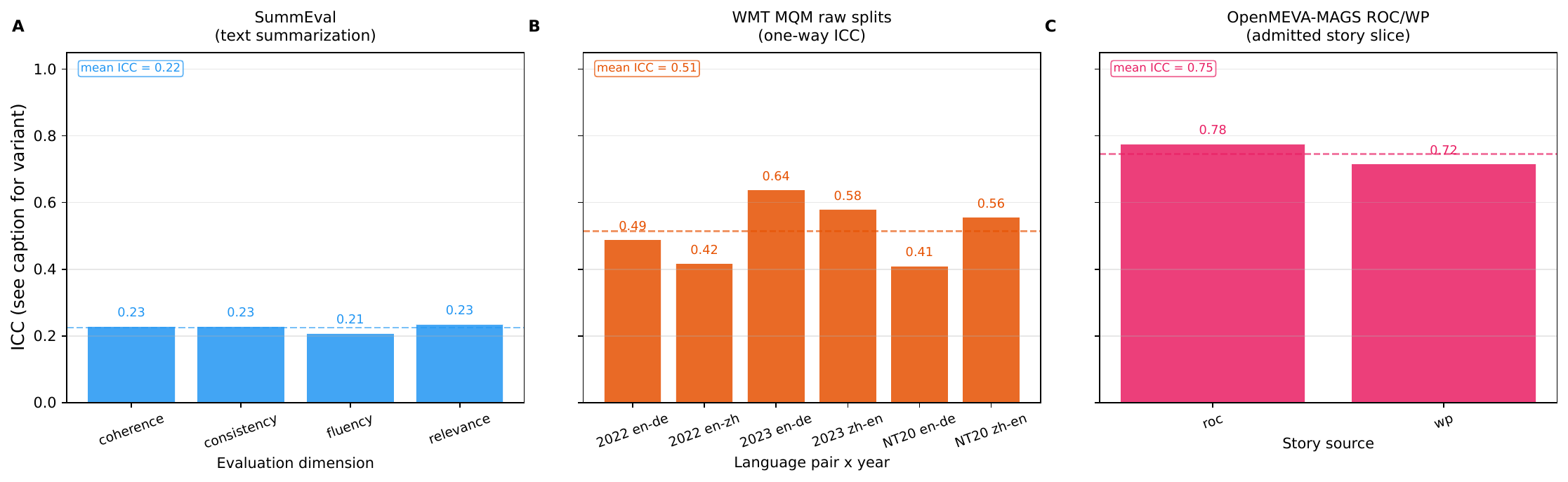}
\caption{ICC reliability profiles across the repeated-rating analysis slices.  Panel A shows SummEval dimension-level ICC(A,1), with mean ICC $\approx0.22$.  Panel B shows one-way ICC(1,1) for eligible raw WMT MQM splits, with mean ICC $\approx0.51$; these are the direct-validation splits, not the prepared prospective clusters.  Panel C shows OpenMEVA-MAGS reliability, with mean ICC $\approx0.75$.  The dashed line in each panel marks the slice mean.}
\label{fig:dataset-replication-icc}
\end{figure}

\subsection{Direct coefficient-law validation and robustness}
\label{app:direct-details}

The main text reports pooled direct-tilt validation.  This subsection expands the direct SummEval and robustness results.  Table~\ref{tab:direct-validation} gives the rendered multi-dataset validation row that supports the direct-validation claim, including the QGEval-human repeated scalar-rating slice.  For OpenMEVA-MAGS, the direct-validation row pools the pooled, ROC, and WP scalar-rating splits, with 5 paths from each split for 15 total paths.  In the default SummEval LOO construction, all 32 held-out-rater paths have negative theoretical and fitted $B$, with mean local $R^2$ essentially equal to one.  The strongest expert dimensions, consistency and fluency, also have the most negative curvature, so improved measurement does not automatically imply larger local headroom.

\begin{table}[tbp]
\centering
\scriptsize
\caption{Human direct-tilt validation of the local coefficient law. All fits use exact exponential tilt and LOO/disjoint human gold where applicable. The QGEval-human row is the direct scalar-human validation slice listed in Table~\ref{tab:protocol-main}.}
\label{tab:direct-validation}
\begin{adjustbox}{max width=\textwidth}
\begin{tabular}{llrrrrrrr}
\toprule
Dataset & Setting & Paths & corr($A$) & corr($B$) & sign($B$) agree & $B_{\rm th}<0$ & $B_{\rm fit}<0$ & median $R^2$ \\
\midrule
SummEval & human scalar LOO & 32 & 0.9997 & 0.9934 & 32/32 & 100.0\% & 100.0\% & 1.000000 \\
QGEval-human & eligible human scalar paths & 63 & 0.9998 & 0.9302 & 63/63 & 98.4\% & 98.4\% & 0.999874 \\
OpenMEVA-MAGS pooled/ROC/WP & eligible human scalar paths & 15 & 1.0000 & 0.9975 & 13/15 & 13.3\% & 26.7\% & 1.000000 \\
WMT MQM eligible paths & human scalar LOO & 61 & 0.9999 & 0.9816 & 61/61 & 100.0\% & 100.0\% & 0.999999 \\
WMT pSQM & human scalar LOO & 12 & 1.0000 & 0.9928 & 12/12 & 91.7\% & 91.7\% & 1.000000 \\
\bottomrule
\end{tabular}
\end{adjustbox}
\end{table}

\begin{figure}[tbp]
\centering
\includegraphics[width=0.72\linewidth]{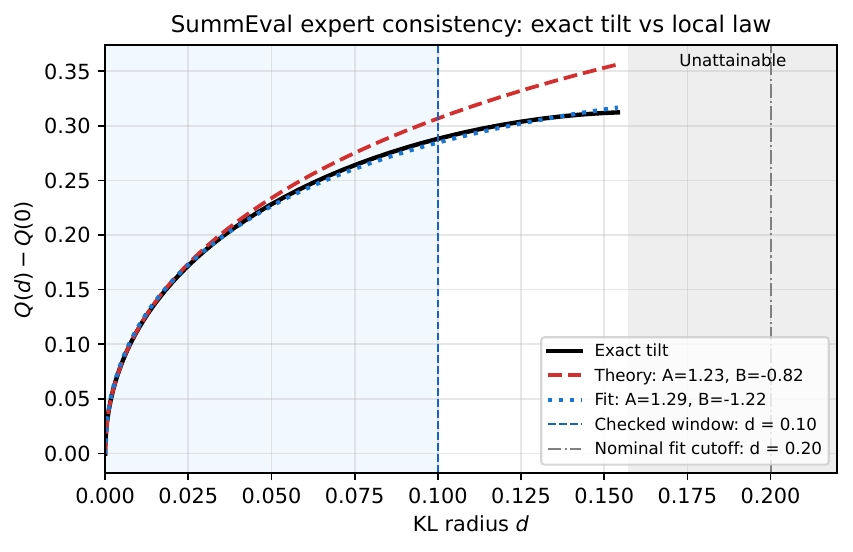}
\caption{SummEval exact-tilt validation example.  One human rater defines the proxy and LOO ratings define $J$. \reviewedit{The checked window is shaded through $d=0.10$; the dash-dot line marks the nominal fitting cutoff $d=0.20$. Curves stop at the computed exact path near $d=0.15$, and the unattainable region is shaded gray.}{The vertical dash-dot line marks the nominal fitting cutoff $d=0.20$, not the main endpoint $0.10$. The exact path ends near $d=0.15$; curve extensions beyond its attainable range are extrapolations, not validation.}}
\label{fig:summeval-tilt-app}
\end{figure}

\begin{table}[tbp]
\centering
\scriptsize
\caption{SummEval LOO direct-tilt coefficient summary by rater pool and dimension at $d\le0.1$.}
\label{tab:summeval-loo-by-group}
\begin{adjustbox}{max width=\textwidth}
\begin{tabular}{llrrrrrrr}
\toprule
Pool & Dimension & Raters & mean $A_{\rm th}$ & mean $B_{\rm th}$ & mean $A_{\rm fit}$ & mean $B_{\rm fit}$ & mean $\rho$ & mean $\sigma_J/\sigma_r$ \\
\midrule
crowd & coherence & 5 & 0.647 & -0.091 & 0.647 & -0.096 & -0.108 & 0.788 \\
crowd & consistency & 5 & 0.669 & -0.044 & 0.671 & -0.070 & -0.115 & 0.788 \\
crowd & fluency & 5 & 0.632 & -0.061 & 0.634 & -0.078 & -0.111 & 0.783 \\
crowd & relevance & 5 & 0.594 & -0.079 & 0.596 & -0.098 & -0.100 & 0.781 \\
expert & coherence & 3 & 1.090 & -0.094 & 1.093 & -0.130 & -0.264 & 0.953 \\
expert & consistency & 3 & 1.221 & -0.812 & 1.250 & -1.093 & -0.044 & 0.991 \\
expert & fluency & 3 & 0.840 & -0.630 & 0.851 & -0.745 & -0.086 & 0.942 \\
expert & relevance & 3 & 0.706 & -0.208 & 0.706 & -0.211 & -0.230 & 0.886 \\
\bottomrule
\end{tabular}
\end{adjustbox}
\end{table}

The coefficient recovery is not a fragile consequence of choosing $d\le0.1$.  Table~\ref{tab:fit-range-sensitivity} shows the exact-tilt validation over wider fit windows on the eight SummEval pool--dimension groups.  Using reachable points below each nominal cutoff, the $A$ correlation is at least $0.9986$ and the $B$ correlation exceeds $0.982$, with $8/8$ sign agreement. Only 2 of 32 paths can reach $d=2.0$. The main window $d\le0.1$ lies below the atomless $N=2$ BoN radius; tied finite pools can have smaller radii.

\begin{table}[tbp]
\centering
\small
\caption{SummEval fit-range sensitivity for direct exact-tilt validation.}
\label{tab:fit-range-sensitivity}
\begin{adjustbox}{max width=0.82\textwidth}
\begin{tabular}{rrrrr}
\toprule
Fit cutoff $D$ & corr($A$) & corr($B$) & sign agree & mean $R^2$ \\
\midrule
0.05 & 1.0000 & 0.9983 & 8/8 & 0.999999 \\
0.10 & 0.9997 & 0.9953 & 8/8 & 0.999993 \\
0.20 & 0.9986 & 0.9909 & 8/8 & \reviewedit{0.999968}{0.999967} \\
0.50 & 0.9991 & 0.9891 & 8/8 & \reviewedit{0.999962}{0.999961} \\
1.00 & 0.9995 & 0.9846 & 8/8 & \reviewedit{0.999933}{0.999932} \\
2.00 & 0.9990 & 0.9825 & 8/8 & \reviewedit{0.999912}{0.999911} \\
\bottomrule
\end{tabular}
\end{adjustbox}
\end{table}

The default estimator is also robust to score standardization and to using the median of the other raters as gold.  Table~\ref{tab:estimator-sensitivity} fixes the fit window at $d\le0.1$.  Same-mean gold is included only as a mechanical sanity check.  It is not the preferred validation because it reuses the same mean structure on both sides.  The LOO and median-other variants give nearly identical conclusions.

\begin{table}[tbp]
\centering
\scriptsize
\caption{Estimator sensitivity at $d\le0.1$ on SummEval.}
\label{tab:estimator-sensitivity}
\begin{adjustbox}{max width=\textwidth}
\begin{tabular}{llrrrrrr}
\toprule
Score transform & Gold estimator & paths & corr($A$) & corr($B$) & sign match & $B_{\rm th}<0$ & mean $R^2$ \\
\midrule
rater-$z$ & LOO mean & 32 & 0.9998 & 0.9936 & 100\% & 100\% & 0.999995 \\
rater-$z$ & median others & 32 & 0.9998 & 0.9935 & 100\% & 100\% & 0.999995 \\
rater-$z$ & same mean & 32 & 0.9997 & 0.9955 & 100\% & 100\% & 0.999995 \\
raw & LOO mean & 32 & 0.9997 & 0.9934 & 100\% & 100\% & 0.999995 \\
raw & median others & 32 & 0.9997 & 0.9933 & 100\% & 100\% & 0.999995 \\
raw & same mean & 32 & 0.9995 & 0.9954 & 100\% & 100\% & 0.999995 \\
\bottomrule
\end{tabular}
\end{adjustbox}
\end{table}

\paragraph{Ordinal-artifact check.} A possible concern is that discretizing continuous scores into a small number of rating levels could itself induce negative $B$.  The ordinal simulation does not support this as the main explanation.  Across 5-point and 9-point synthetic scales with noise standard deviations from 0.25 to 1.25, the continuous baseline keeps $\Pr(B_{\rm fit}<0)$ near one half, whereas the ordinal transformation raises it only to about $0.49$--$0.68$, depending on scale and noise.  Thus, discretization can shift curvature modestly, but the effect is far weaker than the near-universal negative-$B$ pattern in SummEval, QGEval-human, and eligible WMT MQM paths.

\subsection{Cross-slice and feedback-geometry extensions}
\label{app:feedback-extensions-filled}

The main text summarizes the repeated scalar human-rating validation as the cleanest theorem-checking regime.  Table~\ref{tab:feedback-extensions-filled} records broader extensions that are useful but should not be read as equally controlled theorem validation.  MT-Bench \citep{zheng2023judging} uses split-half Bradley--Terry (BT) utilities from pairwise judgments \citep{bradley1952rank}.  QGEval metric-proxy paths use automatic metrics as $r$ against human gold \citep{fu2024qgeval}.  RewardBench uses continuous score margins from a pinned snapshot of the public \texttt{allenai/reward-bench-results} repository rather than human repeated ratings \citep{lambert2025rewardbench}.  Each RewardBench model path uses every item for which that model has a score margin, instead of dropping a whole subset when another model lacks a margin.  The common pattern is that coefficient recovery remains strong, but the sign regime is more heterogeneous outside the repeated scalar human setup.

\begin{table}[tbp]
\centering
\small
\caption{Feedback-geometry and proxy-regime extensions.  These are breadth checks, not the main human direct-rating validation regime.}
\label{tab:feedback-extensions-filled}
\begin{adjustbox}{max width=\textwidth}
\begin{tabular}{llrrrrrr}
\toprule
Dataset & Setting & Paths & corr($A$) & corr($B$) & sign agree & $B_{\rm th}<0$ & median $R^2$ \\
\midrule
MT-Bench BT split-half & split-half BT utilities & 200 & 0.9999 & 0.9831 & 186/200 & 93.0\% & 1.0000 \\
QGEval metrics & automatic metrics vs human gold & 315 & 0.9991 & 0.6915 & 301/315 & 77.5\% & 1.0000 \\
RewardBench margins & cross-model margin geometry & 2515 & 0.9999 & 0.9585 & 2359/2515 & 45.7\% & 1.0000 \\
\bottomrule
\end{tabular}
\end{adjustbox}
\end{table}

OpenMEVA-MAGS \citep{guan2021openmeva} is the clearest human-rating heterogeneity case among the eligible story-generation annotations.  The pooled split has 4/5 sign agreement with no negative theoretical $B$.  In the ROC split all 5 paths have non-negative theoretical and fitted $B$, while the WP split is mixed: $2/5$ theoretical paths and $3/5$ fitted paths have negative $B$.  This supports the paper's central wording: the coefficient law transfers, but the local sign regime is empirical.

\begin{table}[tbp]
\centering
\small
\caption{OpenMEVA-MAGS pooled/ROC/WP direct-tilt summary by split.  The table does not cover the full OpenMEVA benchmark suite.}
\label{tab:openmeva-split-detail}
\begin{adjustbox}{max width=0.9\textwidth}
\begin{tabular}{lrrrrrr}
\toprule
Split & Paths & corr($A$) & corr($B$) & sign agree & $B_{\rm th}<0$ & median $R^2$ \\
\midrule
pooled & 5 & 1.0000 & 0.9993 & 4/5 & 0.0\% & 1.0000 \\
ROC & 5 & 0.9996 & 0.9411 & 5/5 & 0.0\% & 1.0000 \\
WP & 5 & 1.0000 & 0.9989 & 4/5 & 40.0\% & 1.0000 \\
\bottomrule
\end{tabular}
\end{adjustbox}
\end{table}

\subsection{Exact \texorpdfstring{$A$}{A} versus normalized \texorpdfstring{$A_\nu$}{A-nu}}
\label{app:scaling-a-vs-nu-filled}

The exact local law predicts $A=\sqrt2(\sigma_J^2+\rho)/\sigma_r$.  In the direct human-proxy experiments, the simplified reliability-only coefficient $A_\nu=\sqrt2\sigma_J\nu$ is valid only when the correlated proxy-error term $\rho$ is negligible and the reliability quantity $\nu$ coincides with the trained-proxy signal ratio $\sigma_J/\sigma_r$ under the same measurement model.  Table~\ref{tab:scaling-a-vs-nu-filled} shows that this extra condition is not empirically satisfied: the median $|\rho|/\sigma_J^2$ is $0.212$ on SummEval and $0.338$ on eligible WMT MQM clusters.  The exact coefficient still matches the fitted coefficient essentially perfectly, while $A_\nu$ is noticeably rougher and induces large surrogate-headroom error.

For the GPT-4o external-proxy setting, the reported $A_\nu$ row is not an independently estimated proxy-rater ICC baseline.  The GPT scoring passes are first averaged into one item-level proxy score $r_i$.  On each calibration split, we then compute the same variance-ratio surrogate used in the direct-proxy experiments,
\[
\widehat A_\nu=\sqrt2\,\hat\sigma_J(\hat\sigma_J/\hat\sigma_r)=\sqrt2\,\hat\sigma_J^2/\hat\sigma_r,
\]
where $\hat\sigma_J^2$ and $\hat\sigma_r^2$ are the calibration-split item-level variances of the human aggregate target and the aggregated GPT-4o proxy score.  The stored number of GPT scoring passes ($n_{\rm judges}$ in Table~\ref{tab:gpt4o-repro-record}) is used only to form the item-level mean proxy score and does not enter the $A_\nu$ baseline separately.  Thus, in the external-proxy experiment, $A_\nu$ should be interpreted as a scale/reliability-style surrogate rather than as a measured GPT-rater ICC.

\begin{table}[H]
\centering
\scriptsize
\caption{Exact $A$ coefficient versus normalized $A_\nu$ approximation.  The ``small $\rho$'' threshold is $|\rho|/\sigma_J^2\le0.1$.  Headroom error is the median relative error of the vertex height $H_*=-A^2/(4B)$ (Appendix~\ref{app:measurement-intervention-filled}) when $A_\nu$ replaces exact $A$ with the same $B$.}
\label{tab:scaling-a-vs-nu-filled}
\begin{adjustbox}{max width=\textwidth}
\begin{tabular}{lrrrrrrr}
\toprule
Dataset & Paths & median $|\rho|/\sigma_J^2$ & small $\rho$ & headroom error & corr($A_{\rm exact},A_{\rm fit}$) & corr($A_\nu,A_{\rm fit}$) & corr($B_{\rm th},B_{\rm fit}$) \\
\midrule
SummEval & 32 & 0.212 & 3.1\% & 22.2\% & 0.9997 & 0.9821 & 0.9934 \\
WMT MQM clusters & 141 & 0.338 & 13.5\% & 31.0\% & 1.0000 & 0.9814 & 0.9768 \\
\bottomrule
\end{tabular}
\end{adjustbox}
\end{table}

\subsection{Prospective prediction by group}
\label{app:prediction-by-group-filled}

Table~\ref{tab:prediction-baselines-ci-filled} compares the coefficient predictor with an empirical two-parameter fit and other baselines.  The empirical fit estimates $A$ and $B$ by least squares on the exact calibration-split gain curve over $0<d\le0.10$, then predicts the evaluation gain without using evaluation labels.  It is a usable data-driven baseline under the same calibration protocol.  The oracle instead uses evaluation moment coefficients and is only a retrospective diagnostic of finite-range approximation error.  The moment-based predictor has similar endpoint MAE to the empirical fit while providing closed-form coefficient estimates with an analytic interpretation.

Table~\ref{tab:prediction-multidataset-app} gives the compact multi-slice view used in the main text.  Prospective gain and MAE use each human target's native score units, including MQM penalties for WMT.  Absolute MAEs are therefore not directly comparable across datasets.  Within the eligible WMT slice, the coefficient predictor has lower MAE than $A$-only, zero-gain, and reliability-only baselines.

\begin{table}[H]
\centering
\scriptsize
\caption{Prospective prediction baselines at $d=0.10$, including the calibration-only empirical-fit baseline and retrospective oracle.  Brackets are 95\% path-cluster bootstrap intervals for pooled-split MAE (1,000 draws, seed 0). Each draw retains all valid splits of its sampled paths.  Split rows sharing paths/items make these descriptive uncertainty summaries rather than iid split-level intervals.  For the GPT-4o row, $A_\nu$ denotes the calibration-split variance-ratio surrogate computed after averaging GPT scoring passes to one item-level proxy score; no GPT-rater ICC is estimated.}
\label{tab:prediction-baselines-ci-filled}
\begin{adjustbox}{max width=\textwidth}
\begin{tabular}{lllrr}
\toprule
Setting & Predictor & Resamples & MAE [95\% CI] & mean $\Delta(0.10)$ \\
\midrule
Human proxy & coefficient $A{+}B$ & 1589 & 0.0174 [0.0158, 0.0193] & 0.2147 \\
Human proxy & empirical two-parameter fit & 1589 & \textbf{0.0168} [0.0155, 0.0183] & 0.2147 \\
Human proxy & $A$-only & 1589 & 0.0310 [0.0229, 0.0409] & 0.2147 \\
Human proxy & zero gain & 1589 & 0.2147 [0.2002, 0.2314] & 0.2147 \\
Human proxy & $A_\nu$ & 1589 & 0.0880 [0.0698, 0.1115] & 0.2147 \\
Human proxy & oracle evaluation coefficients & 1589 & 0.0041 [0.0023, 0.0064] & 0.2147 \\
GPT-4o proxy & coefficient $A{+}B$ & 200 & 0.0118 [0.0104, 0.0125] & 0.0680 \\
GPT-4o proxy & empirical two-parameter fit & 200 & \textbf{0.0117} [0.0103, 0.0126] & 0.0680 \\
GPT-4o proxy & $A$-only & 200 & 0.0165 [0.0130, 0.0200] & 0.0680 \\
GPT-4o proxy & zero gain & 200 & 0.0680 [0.0517, 0.0818] & 0.0680 \\
GPT-4o proxy & $A_\nu$ & 200 & 0.1508 [0.1234, 0.1782] & 0.0680 \\
GPT-4o proxy & oracle evaluation coefficients & 200 & 0.0005 [0.0003, 0.0009] & 0.0680 \\
\bottomrule
\end{tabular}
\end{adjustbox}
\end{table}

\begin{table}[H]
\centering
\scriptsize
\caption{Prospective prediction summary across the eligible analysis slices used in the main text.  WMT means complete-rater MQM clusters.  OpenMEVA means MAGS ROC/WP human annotations.}
\label{tab:prediction-multidataset-app}
\begin{adjustbox}{max width=\textwidth}
\begin{tabular}{llllllll}
\toprule
Setting & Splits & A+B MAE @0.10 & A-only MAE & Zero MAE & $A_\nu$ MAE & Sign acc. & Mean gain \\
\midrule
SummEval human & 1589 & \textbf{0.0174} [0.0158, 0.0193] & 0.0310 [0.0229, 0.0409] & 0.2147 [0.2002, 0.2314] & 0.0880 [0.0698, 0.1115] & 90.7\% & 0.2147 \\
WMT MQM clusters & 7050 & \textbf{0.0846} [0.0793, 0.0901] & 0.1298 [0.1198, 0.1406] & 0.8524 [0.7818, 0.9252] & 0.6230 [0.5394, 0.7180] & 91.0\% & 0.8524 \\
OpenMEVA-MAGS & 500 & \textbf{0.0148} [0.0141, 0.0156] & 0.0162 [0.0148, 0.0175] & 0.3327 [0.3083, 0.3556] & 0.0207 [0.0162, 0.0264] & 91.8\% & 0.3327 \\
GPT-4o proxy & 200 & \textbf{0.0118} [0.0104, 0.0125] & 0.0165 [0.0130, 0.0200] & 0.0680 [0.0517, 0.0818] & 0.1508 [0.1234, 0.1782] & 97.5\% & 0.0680 \\
\bottomrule
\end{tabular}

\end{adjustbox}
\end{table}

The path/repeat/valid-split accounting for the headline prospective rows is in Table~\ref{tab:protocol-main}.  Figure~\ref{fig:main-results} summarizes the endpoint comparison and one transfer route graphically.

\begin{figure}[H]
\centering
\includegraphics[width=0.94\textwidth]{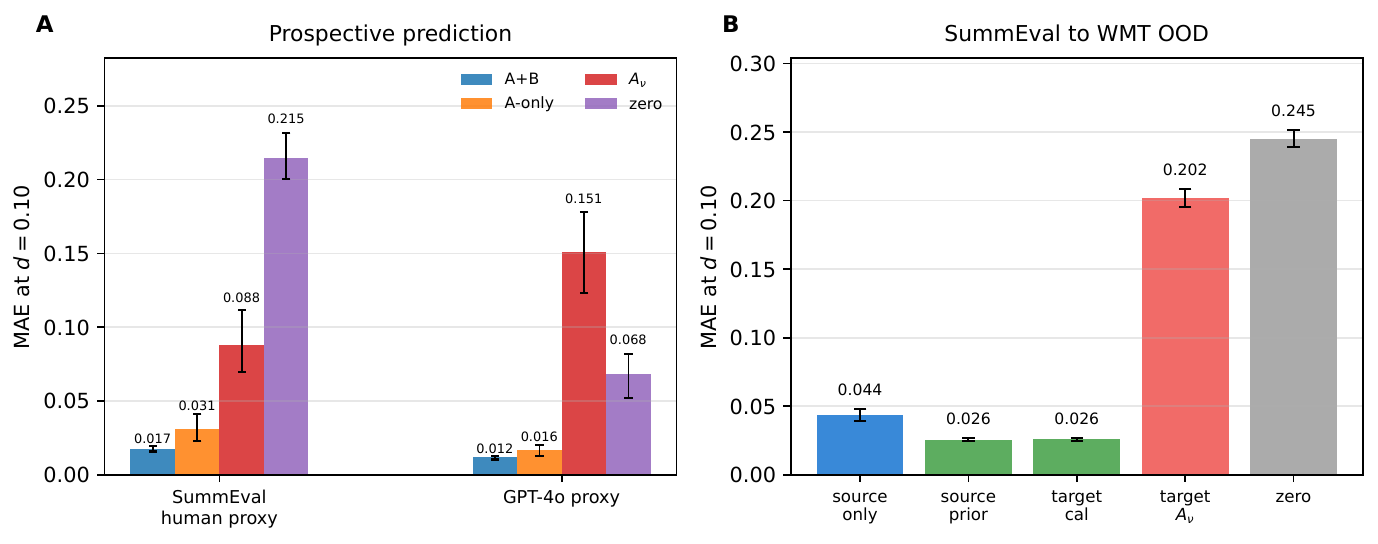}
\caption{Overview of endpoint prediction and OOD transfer. Panel A compares endpoint gain-prediction error at $d=0.10$ for SummEval human-rater proxy paths and the GPT-4o judge proxy experiment. Panel B shows the SummEval $\rightarrow$ WMT OOD transfer comparison under calibration-only target standardization (Appendix~\ref{app:ood}).  Source-prior denotes the fixed target-heavy shrinkage estimator. Error bars are 95\% path-cluster bootstrap intervals for pooled-split MAE; GPT-4o has only four dimension paths. Lower MAE is better.}
\label{fig:main-results}
\end{figure}

Table~\ref{tab:prediction-by-group-filled} localizes the prospective result using $B$-sign agreement, not fixed-gain sign accuracy.  The human-proxy expert groups have perfect $B$-sign agreement, while crowd consistency and fluency are harder coefficient-sign cases.  Their magnitude MAE remains small.  In GPT-4o, all four dimensions retain positive mean gain at $d=0.10$, with coherence largest at $0.0865$ and fluency smallest at $0.0446$.  $B$-sign agreement is $90$--$100\%$ by dimension, and $A{+}B$ MAE stays between $0.0097$ and $0.0126$.

\begin{table}[H]
\centering
\scriptsize
\caption{Prospective prediction diagnostics by group.  Resamples are valid item splits.  ``Theory'' denotes the calibration moment coefficients ($A{+}B$) and ``emp-fit'' the calibration empirical two-parameter fit.  Both $B$-sign columns use the evaluation curve-fit coefficient as their reference.}
\label{tab:prediction-by-group-filled}
\begin{adjustbox}{max width=\textwidth}
\begin{tabular}{llrrrrrr}
\toprule
Setting & Group & Resamples & $B$-sign theory & $B$-sign emp-fit & MAE theory @\textbf{0.10} & MAE $A_\nu$ @\textbf{0.10} & mean $\Delta(0.10)$ \\
\midrule
Human proxy & coherence/expert & 150 & 100.0\% & 100.0\% & \textbf{0.0137} & 0.1574 & 0.3319 \\
Human proxy & coherence/crowd & 250 & 91.6\% & 95.6\% & \textbf{0.0164} & 0.0630 & 0.1949 \\
Human proxy & consistency/expert & 139 & 100.0\% & 100.0\% & \textbf{0.0325} & 0.1250 & 0.2817 \\
Human proxy & consistency/crowd & 250 & 76.0\% & 95.6\% & \textbf{0.0148} & 0.0594 & 0.2052 \\
Human proxy & fluency/expert & 150 & 100.0\% & 100.0\% & \textbf{0.0229} & 0.1212 & 0.1943 \\
Human proxy & fluency/crowd & 250 & 82.0\% & 93.2\% & \textbf{0.0154} & 0.0597 & 0.1925 \\
Human proxy & relevance/expert & 150 & 100.0\% & 100.0\% & \textbf{0.0161} & 0.1346 & 0.2016 \\
Human proxy & relevance/crowd & 250 & 91.6\% & 99.2\% & \textbf{0.0142} & 0.0596 & 0.1787 \\
GPT-4o proxy & coherence & 50 & 90.0\% & 92.0\% & \textbf{0.0125} & 0.1691 & 0.0865 \\
GPT-4o proxy & consistency & 50 & 100.0\% & 100.0\% & \textbf{0.0123} & 0.1356 & 0.0731 \\
GPT-4o proxy & fluency & 50 & 100.0\% & 100.0\% & \textbf{0.0126} & 0.1873 & 0.0446 \\
GPT-4o proxy & relevance & 50 & 100.0\% & 100.0\% & \textbf{0.0097} & 0.1113 & 0.0676 \\
\bottomrule
\end{tabular}
\end{adjustbox}
\end{table}

\FloatBarrier

\begin{table}[H]
\centering
\scriptsize
\caption{Audit separating the two sign metrics used in the paper.  $B$-sign agreement compares the calibration moment coefficient sign with the sign of $B_{\rm eval}^{\rm fit}$, fitted to the exact evaluation gain curve over $0<d\le0.10$.  Fixed-gain sign accuracy compares the predicted and actual gain signs at $d=0.10$.}
\label{tab:sign-metric-audit}
\begin{adjustbox}{max width=\textwidth}
\begin{tabular}{lllll}
\toprule
Setting & Metric & Rows & Accuracy & Interpretation \\
\midrule
SummEval human & B-sign agreement & 1589 & 90.7\% & sign(calibration B) equals sign(evaluation fitted B) \\
WMT MQM clusters & B-sign agreement & 7050 & 91.0\% & sign(calibration B) equals sign(evaluation fitted B) \\
OpenMEVA-MAGS & B-sign agreement & 500 & 91.8\% & sign(calibration B) equals sign(evaluation fitted B) \\
GPT-4o proxy & B-sign agreement & 200 & 97.5\% & sign(calibration B) equals sign(evaluation fitted B) \\
SummEval human & fixed-gain sign at $d=0.10$ & 1589 & 100.0\% & sign(predicted Delta) equals sign(actual Delta) \\
GPT-4o proxy & fixed-gain sign at $d=0.10$ & 200 & 100.0\% & sign(predicted Delta) equals sign(actual Delta) \\
\bottomrule
\end{tabular}

\end{adjustbox}
\end{table}

\begin{figure}[tbp]
\centering
\includegraphics[width=0.88\linewidth]{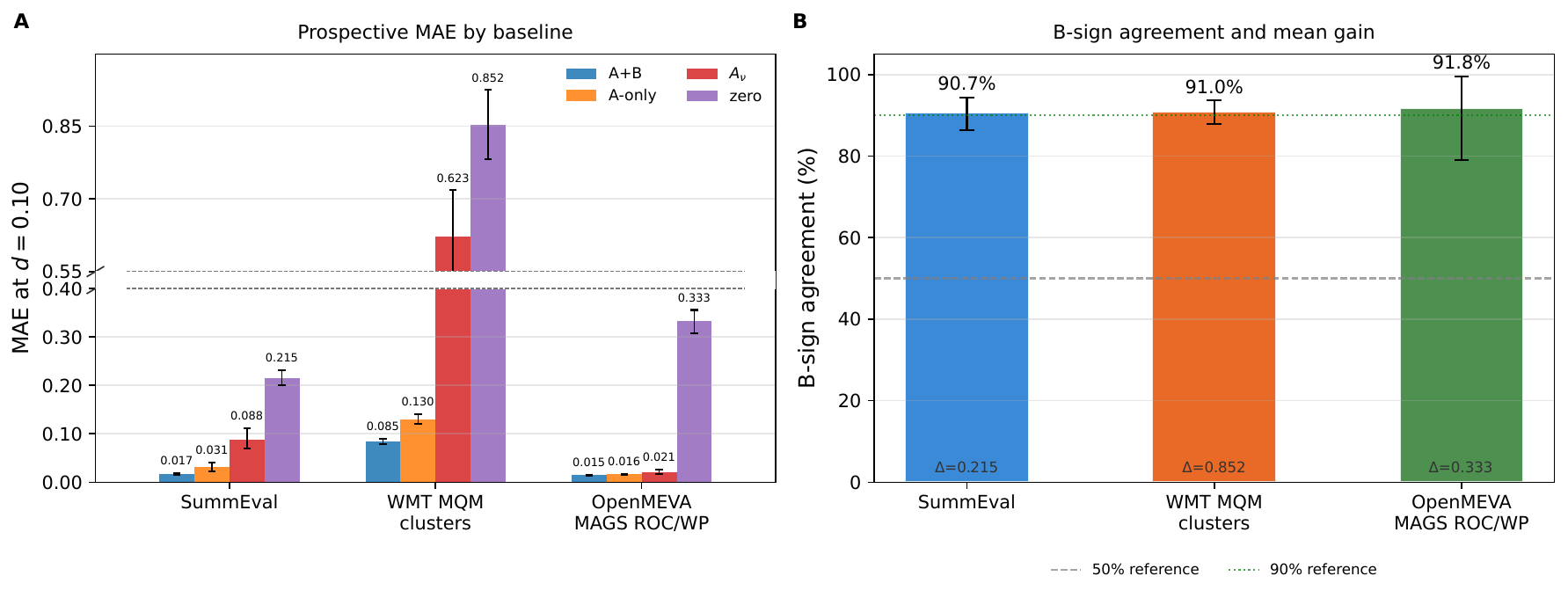}
\caption{Prospective prediction diagnostics across eligible human-proxy analysis slices.  Panel A compares $A{+}B$ MAE against the reliability-only baseline ($A_\nu$) at $d=0.10$ on SummEval, WMT MQM clusters, and OpenMEVA-MAGS ROC/WP; the broken $y$-axis separates the high-error $A_\nu$/zero baselines from the lower coefficient errors.  Panel B reports agreement between the calibration moment coefficient sign and the evaluation curve-fit coefficient sign, $\operatorname{sign}(\hat B_{\rm cal})=\operatorname{sign}(B_{\rm eval}^{\rm fit})$, together with mean realized gain $\Delta(0.10)$ for the same human-proxy slices.  The evaluation coefficient is fitted over $0<d\le0.10$ and is used only for retrospective scoring.  The 50\% and 90\% lines are visual references, not calibrated acceptance thresholds. GPT-4o results are in Tables~\ref{tab:prediction-baselines-ci-filled} and~\ref{tab:prediction-by-group-filled}.}
\label{fig:prediction-mae-multidataset}
\end{figure}

\begin{table}[H]
\centering
\scriptsize
\caption{Margin-conditioned fixed-gain sign accuracy at $d=0.10$.  The Condition column thresholds the actual held-out gain margin $|\Delta(0.10)|$, so this is an oracle diagnostic of sign-margin sensitivity rather than a deployable selection rule.  The human-proxy case has positive gains throughout.  GPT-4o proxy retains perfect sign accuracy across the realized-margin sweep until the most aggressive filter leaves no retained cases.}
\label{tab:margin-sweep-filled}
\begin{adjustbox}{max width=\textwidth}
\begin{tabular}{llllll}
\toprule
Setting & Condition & Coverage & $n$ & Fixed-gain sign acc. & MAE \\
\midrule
SummEval human & all & 100.0\% & 1589 & 100.0\% & 0.0174 \\
SummEval human & $|\Delta|>0.001$ & 100.0\% & 1589 & 100.0\% & 0.0174 \\
SummEval human & $|\Delta|>0.005$ & 100.0\% & 1589 & 100.0\% & 0.0174 \\
SummEval human & $|\Delta|>0.01$ & 100.0\% & 1589 & 100.0\% & 0.0174 \\
SummEval human & $|\Delta|>0.02$ & 100.0\% & 1589 & 100.0\% & 0.0174 \\
SummEval human & $|\Delta|>0.05$ & 100.0\% & 1589 & 100.0\% & 0.0174 \\
SummEval human & $|\Delta|>0.1$ & 100.0\% & 1589 & 100.0\% & 0.0174 \\
GPT-4o proxy & all & 100.0\% & 200 & 100.0\% & 0.0118 \\
GPT-4o proxy & $|\Delta|>0.001$ & 100.0\% & 200 & 100.0\% & 0.0118 \\
GPT-4o proxy & $|\Delta|>0.005$ & 100.0\% & 200 & 100.0\% & 0.0118 \\
GPT-4o proxy & $|\Delta|>0.01$ & 100.0\% & 200 & 100.0\% & 0.0118 \\
GPT-4o proxy & $|\Delta|>0.02$ & 100.0\% & 200 & 100.0\% & 0.0118 \\
GPT-4o proxy & $|\Delta|>0.05$ & 78.5\% & 157 & 100.0\% & 0.0121 \\
GPT-4o proxy & $|\Delta|>0.1$ & 0.0\% & 0 & -- & -- \\
\bottomrule
\end{tabular}

\end{adjustbox}
\end{table}

\FloatBarrier
\subsection{Flexible calibration under matched information}
\label{app:flexible}

\paragraph{Calibrators and tuning.}  All calibrators are fitted on the calibration items of each split.  The two information regimes share the split plan, seeds, and code, and differ only in the pool over which the fitted $\hat h$ is integrated.  Isotonic regression is nondecreasing and untuned.  Binned calibration uses equal-frequency proxy bins, with candidate bin counts $b\in\{3,5,8,12\}$ restricted by $3b\le n_{\rm cal}$ (ties can leave smaller bins), and shrinks bin means toward the calibration mean with pseudo-count $k\in\{0,5,20\}$; both are chosen by five-fold cross-validated squared error.  The spline uses a cubic B-spline basis with constant extrapolation and $\{4,6,8\}$ knots chosen by five-fold cross-validation, with a ridge penalty chosen by efficient leave-one-out over 13 log-spaced values in $[10^{-3},10^{3}]$.  A small multilayer perceptron with two hidden ReLU layers of 16 units is pre-specified, untuned, and reported only as exploratory.  The calibration two-term fit samples the exact calibration-pool tilt curve at 200 equally spaced tilt values up to KL $\min(0.10,\,0.95\,d_{\max})$, where $d_{\max}$ is the largest KL reachable by the calibration pool, and fits $a\sqrt d+bd$ without intercept.

\paragraph{Support.}  Each path has 50 random 30/70 splits (seed 0), with 480/1,120 items for SummEval and GPT-4o, 300/700 for OpenMEVA, and a median of 267/623 for the WMT clusters.  WMT uses the corrected 141-path input throughout.  Methods that integrate over the calibration pool are undefined when that pool cannot reach the requested radius.  This happens only for 11 SummEval cells at $d=0.10$, all on one expert consistency path whose calibration splits reach at most KL $0.080$--$0.098$.  Every comparison uses the cells valid for all compared methods (7,989 of 8,000 for SummEval).  With this support, the SummEval endpoint row at $d=0.10$ of Table~\ref{tab:flex-radius} reproduces the endpoint MAE in Table~\ref{tab:prediction-main}a.

\paragraph{Paired comparisons.}  Table~\ref{tab:flex-paired} reports the pre-pool differences from $A{+}B$.  The two-term fit, the stored curve, and isotonic regression are significantly better than $A{+}B$ only on WMT; on the four-path GPT-4o slice, $A{+}B$ is lower than isotonic with an interval that barely excludes zero.  Table~\ref{tab:flex-radius} gives endpoint MAE at each radius, and Table~\ref{tab:flex-target-paired} gives the target-pool-aware differences, of which only projection minus direct integration is information-matched.  Across the 12 dataset--calibrator cells, the aggregate projection--direct gap is $0.00002$--$0.00286$ and the mean pointwise gap is $0.00041$--$0.00850$.  These are different quantities, and neither is a uniform bound from the theorem.

\paragraph{Functional form.}  In a separate run on all cells of each slice, replacing the fitted two-term calibration curve by cubic, six-knot spline, or residual-spline curves fitted to the same calibration data changes CurveMAE by at most $2.1\times10^{-5}$ on every dataset.  This supports the empirical adequacy of the two-term form on this grid, not its theoretical dominance.

\begin{table}[H]
\centering
\scriptsize
\caption{Pre-pool CurveMAE differences from $A{+}B$ (method minus $A{+}B$; positive values mean that $A{+}B$ has lower error), on the common valid cells of Table~\ref{tab:prediction-main}b.  Intervals are 95\% percentile path-clustered bootstrap intervals (2,000 draws, paths resampled, seed 0), not adjusted for multiplicity.  $^{\dagger}$Recomputed with the same protocol because the saved interval used all 8,000 SummEval cells.  GPT-4o has 4 paths, so its intervals are descriptive.  The MLP is exploratory.}
\label{tab:flex-paired}
\begin{adjustbox}{max width=\textwidth}
\begin{tabular}{llrrl}
\toprule
Dataset & Method & CurveMAE & Method $-$ $A{+}B$ & 95\% CI \\
\midrule
SummEval human & $A{+}B$ (reference) & 0.01185 & -- & -- \\
 & $A$-only & 0.01803 & 0.00618 & [0.00324, 0.00962]$^{\dagger}$ \\
 & Calibration two-term fit & 0.01170 & $-$0.00015 & [$-$0.00043, 0.00004]$^{\dagger}$ \\
 & Stored empirical curve & 0.01171 & $-$0.00014 & [$-$0.00044, 0.00004] \\
 & Binned & 0.03010 & 0.01825 & [0.01548, 0.02091] \\
 & Spline & 0.01275 & 0.00090 & [0.00036, 0.00136] \\
 & Isotonic & 0.01168 & $-$0.00017 & [$-$0.00045, 0.00002] \\
 & MLP (exploratory) & 0.01182 & $-$0.00003 & [$-$0.00032, 0.00016] \\
\midrule
GPT-4o proxy & $A{+}B$ (reference) & 0.00806 & -- & -- \\
 & $A$-only & 0.01030 & 0.00224 & [0.00100, 0.00364] \\
 & Calibration two-term fit & 0.00805 & $-$0.00001 & [$-$0.00006, 0.00006] \\
 & Stored empirical curve & 0.00805 & $-$0.00001 & [$-$0.00006, 0.00007] \\
 & Binned & 0.00878 & 0.00072 & [0.00043, 0.00101] \\
 & Spline & 0.00803 & $-$0.00002 & [$-$0.00017, 0.00014] \\
 & Isotonic & 0.00827 & 0.00022 & [0.00001, 0.00042] \\
 & MLP (exploratory) & 0.00814 & 0.00008 & [$-$0.00009, 0.00025] \\
\midrule
WMT MQM & $A{+}B$ (reference) & 0.05921 & -- & -- \\
 & $A$-only & 0.08039 & 0.02118 & [0.01826, 0.02451] \\
 & Calibration two-term fit & 0.05843 & $-$0.00078 & [$-$0.00099, $-$0.00058] \\
 & Stored empirical curve & 0.05842 & $-$0.00079 & [$-$0.00100, $-$0.00059] \\
 & Binned & 0.07356 & 0.01435 & [0.01220, 0.01680] \\
 & Spline & 0.06255 & 0.00334 & [0.00248, 0.00427] \\
 & Isotonic & 0.05882 & $-$0.00039 & [$-$0.00070, $-$0.00009] \\
 & MLP (exploratory) & 0.05916 & $-$0.00005 & [$-$0.00047, 0.00041] \\
\midrule
OpenMEVA-MAGS & $A{+}B$ (reference) & 0.00991 & -- & -- \\
 & $A$-only & 0.01058 & 0.00067 & [$-$0.00001, 0.00136] \\
 & Calibration two-term fit & 0.00984 & $-$0.00007 & [$-$0.00020, 0.00006] \\
 & Stored empirical curve & 0.00984 & $-$0.00007 & [$-$0.00021, 0.00006] \\
 & Binned & 0.01583 & 0.00593 & [0.00299, 0.00986] \\
 & Spline & 0.01002 & 0.00011 & [$-$0.00013, 0.00032] \\
 & Isotonic & 0.00984 & $-$0.00007 & [$-$0.00021, 0.00006] \\
 & MLP (exploratory) & 0.00985 & $-$0.00006 & [$-$0.00024, 0.00013] \\
\bottomrule
\end{tabular}

\end{adjustbox}
\end{table}

\begin{table}[H]
\centering
\scriptsize
\caption{Pre-pool endpoint MAE at each radius, pooled over the common valid cells at that radius (SummEval uses 1,589 of 1,600 cells at $d=0.10$ for every method).  This is endpoint MAE, not CurveMAE; its $d=0.10$ column is comparable to Table~\ref{tab:prediction-main}a.  Native units; lower is better.  The MLP is exploratory.}
\label{tab:flex-radius}
\begin{adjustbox}{max width=\textwidth}
\begin{tabular}{llrrrrr}
\toprule
Dataset & Method & $d=\textbf{0.010}$ & $d=\textbf{0.025}$ & $d=\textbf{0.050}$ & $d=\textbf{0.075}$ & $d=\textbf{0.100}$ \\
\midrule
SummEval human & $A{+}B$ & \textbf{0.0057} & \textbf{0.0088} & \textbf{0.0123} & 0.0150 & 0.0174 \\
 & $A$-only & 0.0063 & 0.0108 & 0.0176 & 0.0243 & 0.0310 \\
 & Calibration two-term fit & \textbf{0.0057} & \textbf{0.0088} & \textbf{0.0123} & \textbf{0.0149} & \textbf{0.0168} \\
 & Stored empirical curve & \textbf{0.0057} & \textbf{0.0088} & \textbf{0.0123} & \textbf{0.0149} & \textbf{0.0168} \\
 & Binned & 0.0129 & 0.0211 & 0.0310 & 0.0392 & 0.0465 \\
 & Spline & 0.0062 & 0.0097 & 0.0134 & 0.0162 & 0.0182 \\
 & Isotonic & \textbf{0.0057} & \textbf{0.0088} & \textbf{0.0123} & \textbf{0.0149} & \textbf{0.0168} \\
 & MLP (exploratory) & \textbf{0.0057} & 0.0089 & 0.0124 & 0.0150 & 0.0169 \\
\midrule
GPT-4o proxy & $A{+}B$ & \textbf{0.0039} & \textbf{0.0060} & \textbf{0.0084} & \textbf{0.0102} & 0.0118 \\
 & $A$-only & 0.0042 & 0.0069 & 0.0104 & 0.0135 & 0.0165 \\
 & Calibration two-term fit & \textbf{0.0039} & \textbf{0.0060} & \textbf{0.0084} & \textbf{0.0102} & \textbf{0.0117} \\
 & Stored empirical curve & \textbf{0.0039} & \textbf{0.0060} & \textbf{0.0084} & \textbf{0.0102} & \textbf{0.0117} \\
 & Binned & 0.0043 & 0.0067 & 0.0092 & 0.0111 & 0.0126 \\
 & Spline & \textbf{0.0039} & \textbf{0.0060} & \textbf{0.0084} & \textbf{0.0102} & \textbf{0.0117} \\
 & Isotonic & 0.0040 & 0.0062 & 0.0086 & 0.0105 & 0.0120 \\
 & MLP (exploratory) & \textbf{0.0039} & 0.0061 & 0.0085 & 0.0103 & 0.0119 \\
\midrule
WMT MQM & $A{+}B$ & \textbf{0.0297} & 0.0454 & 0.0620 & 0.0743 & 0.0846 \\
 & $A$-only & 0.0323 & 0.0532 & 0.0808 & 0.1058 & 0.1298 \\
 & Calibration two-term fit & \textbf{0.0297} & \textbf{0.0453} & \textbf{0.0615} & \textbf{0.0731} & \textbf{0.0825} \\
 & Stored empirical curve & \textbf{0.0297} & \textbf{0.0453} & \textbf{0.0615} & \textbf{0.0731} & \textbf{0.0825} \\
 & Binned & 0.0373 & 0.0567 & 0.0771 & 0.0921 & 0.1046 \\
 & Spline & 0.0319 & 0.0485 & 0.0658 & 0.0782 & 0.0882 \\
 & Isotonic & 0.0300 & 0.0456 & 0.0619 & 0.0736 & 0.0829 \\
 & MLP (exploratory) & 0.0301 & 0.0459 & 0.0623 & 0.0740 & 0.0835 \\
\midrule
OpenMEVA-MAGS & $A{+}B$ & \textbf{0.0045} & \textbf{0.0072} & 0.0103 & 0.0127 & 0.0148 \\
 & $A$-only & 0.0046 & 0.0075 & 0.0109 & 0.0137 & 0.0162 \\
 & Calibration two-term fit & \textbf{0.0045} & \textbf{0.0072} & \textbf{0.0102} & \textbf{0.0126} & \textbf{0.0146} \\
 & Stored empirical curve & \textbf{0.0045} & \textbf{0.0072} & \textbf{0.0102} & \textbf{0.0126} & \textbf{0.0146} \\
 & Binned & 0.0065 & 0.0108 & 0.0162 & 0.0207 & 0.0249 \\
 & Spline & 0.0046 & 0.0073 & 0.0104 & 0.0128 & 0.0149 \\
 & Isotonic & \textbf{0.0045} & \textbf{0.0072} & \textbf{0.0102} & \textbf{0.0126} & \textbf{0.0146} \\
 & MLP (exploratory) & 0.0046 & \textbf{0.0072} & \textbf{0.0102} & \textbf{0.0126} & \textbf{0.0146} \\
\bottomrule
\end{tabular}

\end{adjustbox}
\end{table}

\begin{table}[H]
\centering
\scriptsize
\caption{Target-pool-aware paired CurveMAE differences (first-named minus second-named; negative values favor the first).  Only projection minus direct integration compares the same $\hat h$ with the same information; comparisons with the pre-pool $A{+}B$ anchor mix information regimes.  The mean pointwise gap is the path-mean of $|\text{projection}-\text{direct}|$ predictions, a disagreement rather than an error.  Intervals as in Table~\ref{tab:flex-paired}; $^{\dagger}$computed with the same protocol from row-level predictions.}
\label{tab:flex-target-paired}
\begin{adjustbox}{max width=\textwidth}
\begin{tabular}{llrrrr}
\toprule
Dataset & $h$ & Projection $-$ direct & Projection $-$ anchor & Direct $-$ anchor$^{\dagger}$ & Mean pointwise gap \\
\midrule
SummEval human & Binned & $-$0.00132 & 0.01598 & 0.01729 & 0.00210 \\
 &  & [$-$0.00156, $-$0.00108] & [0.01250, 0.01909] & [0.01388, 0.02039] &  \\
 & Spline & 0.00013 & $-$0.00145 & $-$0.00158 & 0.00185 \\
 &  & [$-$0.00023, 0.00058] & [$-$0.00279, $-$0.00025] & [$-$0.00332, $-$0.00007] &  \\
 & Isotonic & 0.00050 & $-$0.00226 & $-$0.00276 & 0.00171 \\
 &  & [0.00008, 0.00103] & [$-$0.00336, $-$0.00132] & [$-$0.00432, $-$0.00148] &  \\
\midrule
GPT-4o proxy & Binned & $-$0.00004 & 0.00055 & 0.00059 & 0.00041 \\
 &  & [$-$0.00009, 0.00000] & [0.00024, 0.00075] & [0.00027, 0.00084] &  \\
 & Spline & 0.00003 & $-$0.00026 & $-$0.00029 & 0.00042 \\
 &  & [$-$0.00005, 0.00013] & [$-$0.00108, 0.00037] & [$-$0.00120, 0.00037] &  \\
 & Isotonic & $-$0.00010 & $-$0.00005 & 0.00005 & 0.00044 \\
 &  & [$-$0.00022, $-$0.00003] & [$-$0.00094, 0.00048] & [$-$0.00072, 0.00053] &  \\
\midrule
WMT MQM & Binned & $-$0.00286 & 0.00637 & 0.00922 & 0.00850 \\
 &  & [$-$0.00358, $-$0.00223] & [0.00419, 0.00874] & [0.00673, 0.01191] &  \\
 & Spline & $-$0.00031 & $-$0.00253 & $-$0.00222 & 0.00570 \\
 &  & [$-$0.00056, $-$0.00004] & [$-$0.00426, $-$0.00099] & [$-$0.00404, $-$0.00057] &  \\
 & Isotonic & 0.00105 & $-$0.00605 & $-$0.00710 & 0.00456 \\
 &  & [0.00083, 0.00128] & [$-$0.00750, $-$0.00467] & [$-$0.00860, $-$0.00573] &  \\
\midrule
OpenMEVA-MAGS & Binned & $-$0.00099 & 0.00339 & 0.00438 & 0.00180 \\
 &  & [$-$0.00125, $-$0.00075] & [0.00008, 0.00795] & [0.00093, 0.00904] &  \\
 & Spline & $-$0.00002 & $-$0.00265 & $-$0.00263 & 0.00124 \\
 &  & [$-$0.00015, 0.00011] & [$-$0.00336, $-$0.00184] & [$-$0.00342, $-$0.00177] &  \\
 & Isotonic & 0.00006 & $-$0.00279 & $-$0.00285 & 0.00120 \\
 &  & [$-$0.00007, 0.00019] & [$-$0.00342, $-$0.00210] & [$-$0.00353, $-$0.00212] &  \\
\bottomrule
\end{tabular}

\end{adjustbox}
\end{table}

\FloatBarrier
\subsection{Contribution and reliability of \texorpdfstring{$B$}{B}}
\label{app:b-stability}

\paragraph{Contribution.}  Tables~\ref{tab:ab-reduction} and~\ref{tab:ab-bshare} give the numbers behind Figure~\ref{fig:ab-contribution}.  \reviewedit{Both use the pre-pool protocol with 30\% calibration. Error reductions use the common valid cells of Table~\ref{tab:prediction-main}b; magnitude shares use all path-splits (1,600 for SummEval, versus 1,589 common cells at $d=0.10$).}{Both use the pre-pool protocol with 30\% calibration and the common valid cells of Table~\ref{tab:prediction-main}b.}  The endpoint reduction grows with the radius in every dataset; OpenMEVA's interval includes zero for $d\le0.05$, and its CurveMAE interval includes zero.  The share $R_B(d)=|\hat Bd|/(|\hat A\sqrt d|+|\hat Bd|)$ is a magnitude share of the predicted gain.  It grows like $\sqrt d$, because $|\hat Bd|/|\hat A\sqrt d|=(|\hat B|/|\hat A|)\sqrt d$, and it says nothing about accuracy.  The aggregate improvement coexists with split-level heterogeneity: $A{+}B$ has lower CurveMAE than $A$-only on 55.6--67.0\% of path-split units.

\paragraph{Coefficient error and sign stability.}  Table~\ref{tab:b-stability}a decomposes the error at $d=0.10$ with the terms of Lemma~\ref{lem:prediction-error}.  Here $A_{\rm eval}$ and $B_{\rm eval}$ are moment coefficients computed on the held-out split and used only retrospectively, so $\sqrt d|\hat A-A_{\rm eval}|$ and $d|\hat B-B_{\rm eval}|$ combine calibration estimation error with calibration--evaluation shift, and the remainder is the finite-range error of the held-out curve's own two-term approximation.  The sign agreement in this panel is \emph{moment-based}: it compares $\hat B_{\rm cal}$ with $B_{\rm eval}$ and gives 89.4--97.0\% at 30\% calibration.  The 90.7--97.5\% in Table~\ref{tab:prediction-main}a instead compares $\hat B_{\rm cal}$ with a coefficient fitted to the exact held-out curve.  Table~\ref{tab:b-stability}b uses 300 cluster-bootstrap draws within each calibration split.  The sign of $\hat B$ is kept in a dataset median of 96.7--100\% of draws, yet 25.0--52.6\% of split-level 95\% intervals include zero, so sign stability does not establish nonzero population curvature.  With an operational threshold of $0.1\,\hat\sigma_{J}$ on the lower bootstrap percentile of the predicted gain, $A$-only and $A{+}B$ make different proceed decisions on 25.5\% of GPT-4o splits and on at most 2.5\% of splits elsewhere.

\paragraph{Rater count and rating noise.}  Table~\ref{tab:rater-noise} and Figure~\ref{fig:rater-noise} vary only the calibration gold on SummEval: the proxy is a fixed held-out human rater, and evaluation gold is always the clean mean of all non-proxy raters.  One to two expert gold raters reduce the mean absolute $B$ error from $0.0618$ to $0.0428$ and CurveMAE from $0.0269$ to $0.0144$; one to four crowd raters reduce them from $0.0659$ to $0.0590$ and from $0.0122$ to $0.0104$.  Under injected Gaussian or variance-normalized Student-$t_3$ noise, scaled by the group rating SD and clipped to the observed range, the sign of $\hat B$ agrees with the clean estimate in at least 99.3\% (experts) and 93.0\% (crowd) of splits.  Its predictive usefulness is less stable: the advantage of $A{+}B$ over $A$-only disappears at 1.0 rating SD for experts and 0.5 SD for crowd workers.  These CurveMAE values come from their own split streams and are not comparable with Table~\ref{tab:prediction-main}b.

\paragraph{Heavy tails.}  Table~\ref{tab:tail-outlier-robustness} in Appendix~\ref{app:accept-boost} audits preprocessing and contamination with full-pool coefficients.  Injecting 1\% positive 8-SD outliers into the gold scores leaves $B$ correlated 0.955--0.997 with the clean estimates (sign agreement 93.0--98.6\%), but contaminating the proxy destabilizes OpenMEVA most (correlation $-0.107$, sign agreement 20.0\%).  This bounded stress test does not establish robustness to unbounded heavy-tailed scores outside the regularity assumptions of Theorem~\ref{thm:local}.  When tails are heavy, we recommend inspecting them, recalibrating after robust preprocessing, narrowing the window, or abstaining.

\begin{table}[H]
\centering
\scriptsize
\caption{$A$-only versus $A{+}B$ in the pre-pool protocol (common valid cells).  (a) Endpoint MAE at each radius, with the relative reduction $1-\mathrm{MAE}_{A+B}/\mathrm{MAE}_{A\text{-only}}$ and its 95\% path-cluster bootstrap interval (2,000 draws, seed 0).  (b) CurveMAE with the paired difference $A{+}B$ minus $A$-only (negative favors $A{+}B$) and the share of path-split units won by $A{+}B$; paths/cells are 32/7,989 (SummEval), 4/1,000 (GPT-4o), 141/35,250 (WMT), and 10/2,500 (OpenMEVA).  MAEs are in native units and are not comparable across datasets; GPT-4o intervals rest on four paths.}
\label{tab:ab-reduction}
\begin{adjustbox}{max width=\textwidth}
\begin{tabular}{@{}c@{}}
\textit{(a) Endpoint MAE by radius}\\[2pt]
\begin{tabular}{llrrrrrr}
\toprule
Dataset & $d$ & Cells & $A$-only & $A{+}B$ & $A{+}B-A$-only & Red.\ (\%) & 95\% CI (\%) \\
\midrule
SummEval & 0.010 & 1,600 & 0.0063 & \textbf{0.0057} & $-$0.0006 & 9.2 & [5.9, 13.3] \\
 & 0.025 & 1,600 & 0.0108 & \textbf{0.0088} & $-$0.0020 & 18.6 & [12.0, 24.9] \\
 & 0.050 & 1,600 & 0.0176 & \textbf{0.0123} & $-$0.0053 & 30.2 & [19.8, 38.2] \\
 & 0.075 & 1,600 & 0.0243 & \textbf{0.0150} & $-$0.0093 & 38.1 & [25.9, 46.7] \\
 & 0.100 & 1,589 & 0.0310 & \textbf{0.0174} & $-$0.0136 & 44.0 & [30.8, 52.8] \\
\midrule
GPT-4o & 0.010 & 200 & 0.0042 & \textbf{0.0039} & $-$0.0003 & 7.2 & [4.1, 10.7] \\
 & 0.025 & 200 & 0.0069 & \textbf{0.0060} & $-$0.0009 & 12.5 & [6.8, 18.5] \\
 & 0.050 & 200 & 0.0104 & \textbf{0.0084} & $-$0.0020 & 19.0 & [10.3, 27.4] \\
 & 0.075 & 200 & 0.0135 & \textbf{0.0102} & $-$0.0033 & 24.4 & [13.1, 34.5] \\
 & 0.100 & 200 & 0.0165 & \textbf{0.0118} & $-$0.0047 & 28.7 & [14.9, 40.0] \\
\midrule
WMT & 0.010 & 7,050 & 0.0323 & \textbf{0.0297} & $-$0.0025 & 7.9 & [6.9, 8.9] \\
 & 0.025 & 7,050 & 0.0532 & \textbf{0.0454} & $-$0.0078 & 14.7 & [13.0, 16.5] \\
 & 0.050 & 7,050 & 0.0808 & \textbf{0.0620} & $-$0.0188 & 23.3 & [20.7, 25.9] \\
 & 0.075 & 7,050 & 0.1058 & \textbf{0.0743} & $-$0.0315 & 29.8 & [26.7, 32.9] \\
 & 0.100 & 7,050 & 0.1298 & \textbf{0.0846} & $-$0.0452 & 34.8 & [31.5, 38.2] \\
\midrule
OpenMEVA & 0.010 & 500 & 0.0046 & \textbf{0.0045} & $-$0.0001 & 1.9 & [$-$1.1, 4.6] \\
 & 0.025 & 500 & 0.0075 & \textbf{0.0072} & $-$0.0003 & 3.7 & [$-$0.7, 7.7] \\
 & 0.050 & 500 & 0.0109 & \textbf{0.0103} & $-$0.0006 & 5.9 & [$-$0.2, 11.3] \\
 & 0.075 & 500 & 0.0137 & \textbf{0.0127} & $-$0.0010 & 7.4 & [0.1, 13.8] \\
 & 0.100 & 500 & 0.0162 & \textbf{0.0148} & $-$0.0013 & 8.3 & [0.1, 15.5] \\
\bottomrule
\end{tabular}
\\[6pt]
\textit{(b) CurveMAE over $d\in\{0.01,0.025,0.05,0.075,0.10\}$}\\[2pt]
\begin{tabular}{lrrrcrr}
\toprule
 & \multicolumn{2}{c}{CurveMAE} & \multicolumn{2}{c}{$A{+}B-A$-only} &  & Units \\
\cmidrule(lr){2-3}\cmidrule(lr){4-5}Dataset & $A$-only & $A{+}B$ & Diff. & 95\% CI & Red.\ (\%) & won (\%) \\
\midrule
SummEval & 0.01803 & \textbf{0.01185} & $-$0.00618 & [$-$0.00962, $-$0.00324] & 34.3 & 67.0 \\
GPT-4o & 0.01030 & \textbf{0.00806} & $-$0.00224 & [$-$0.00364, $-$0.00100] & 21.7 & 65.5 \\
WMT & 0.08039 & \textbf{0.05921} & $-$0.02118 & [$-$0.02451, $-$0.01826] & 26.4 & 65.2 \\
OpenMEVA & 0.01058 & \textbf{0.00991} & $-$0.00067 & [$-$0.00136, 0.00001] & 6.4 & 55.6 \\
\bottomrule
\end{tabular}

\end{tabular}
\end{adjustbox}
\end{table}

\begin{table}[H]
\centering
\scriptsize
\caption{Median magnitude share $R_B(d)=|\hat Bd|/(|\hat A\sqrt d|+|\hat Bd|)$ of the calibration-predicted gain over path-splits (30\% calibration, 50 splits per path), with the interquartile range at $d=0.10$.  $R_B$ is not net $B$, not a share of the realized gain, and not an error reduction.}
\label{tab:ab-bshare}
\begin{adjustbox}{max width=\textwidth}
\begin{tabular}{lrrrrrrc}
\toprule
\multicolumn{2}{c}{} & \multicolumn{5}{c}{Median $R_B(d)$ (\%) at $d=$} & \multicolumn{1}{c}{IQR (\%)} \\
\cmidrule(lr){3-7}Dataset & Path-splits & 0.010 & 0.025 & 0.050 & 0.075 & 0.100 & at $d=0.100$ \\
\midrule
SummEval & 1,600 & 1.5 & 2.3 & 3.2 & 3.9 & 4.4 & [2.2, 8.4] \\
GPT-4o & 200 & 4.6 & 7.1 & 9.7 & 11.7 & 13.3 & [7.2, 18.6] \\
WMT & 7,050 & 2.8 & 4.3 & 6.0 & 7.3 & 8.3 & [4.8, 12.3] \\
OpenMEVA & 500 & 1.0 & 1.6 & 2.2 & 2.7 & 3.1 & [1.8, 4.0] \\
\bottomrule
\end{tabular}

\end{adjustbox}
\end{table}

\begin{table}[H]
\centering
\scriptsize
\caption{Reliability of $\hat B$.  (a) Median error terms at $d=0.10$ by calibration fraction (50 nested splits per path).  $A_{\rm eval}$ and $B_{\rm eval}$ are held-out moment coefficients used only retrospectively, and the sign agreement is moment-based.  (b) Within-split cluster bootstrap (30\% calibration, 300 draws; clusters are SummEval articles, OpenMEVA stories, and WMT source segments): median bootstrap SE of $\hat B$, median share of draws keeping the sign of $\hat B$, share of split-level 95\% intervals that include zero, and share of splits on which $A$-only and $A{+}B$ make different proceed decisions at the threshold $0.1\,\hat\sigma_J$.  Panels (a) and (b) use different split schedules.  Error columns are in native units.}
\label{tab:b-stability}
\begin{adjustbox}{max width=\textwidth}
\begin{tabular}{@{}c@{}}
\textit{(a) Coefficient error by calibration size}\\[2pt]
\begin{tabular}{lrrrrr}
\toprule
Dataset & Cal. & $\sqrt{d}\,|\hat A-A_{\rm eval}|$ & $d\,|\hat B-B_{\rm eval}|$ & Remainder & Sign agr.\ (\%) \\
\midrule
SummEval & 5\% & 0.0318 & 0.0093 & 0.0017 & 78.4 \\
 & 10\% & 0.0226 & 0.0068 & 0.0017 & 81.9 \\
 & 30\% & 0.0157 & 0.0044 & 0.0018 & 89.4 \\
\midrule
GPT-4o & 5\% & 0.0214 & 0.0069 & 0.0003 & 86.0 \\
 & 10\% & 0.0160 & 0.0049 & 0.0003 & 93.5 \\
 & 30\% & 0.0116 & 0.0037 & 0.0004 & 97.0 \\
\midrule
WMT & 5\% & 0.1621 & 0.0546 & 0.0082 & 79.5 \\
 & 10\% & 0.1187 & 0.0431 & 0.0083 & 84.0 \\
 & 30\% & 0.0777 & 0.0317 & 0.0087 & 89.7 \\
\midrule
OpenMEVA & 5\% & 0.0238 & 0.0056 & 0.0027 & 82.6 \\
 & 10\% & 0.0188 & 0.0044 & 0.0027 & 85.2 \\
 & 30\% & 0.0126 & 0.0029 & 0.0027 & 90.0 \\
\bottomrule
\end{tabular}
\\[6pt]
\textit{(b) Bootstrap sign and interval stability}\\[2pt]
\begin{tabular}{lrrrrr}
\toprule
Dataset & Path-splits & Median SE$(\hat B)$ & Sign kept (\%) & CI $\ni 0$ (\%) & Disagree (\%) \\
\midrule
SummEval & 1,600 & 0.0570 & 96.7 & 52.6 & 0.0 \\
GPT-4o & 200 & 0.0376 & 100.0 & 25.0 & 25.5 \\
WMT & 7,050 & 0.3870 & 97.3 & 49.8 & 2.5 \\
OpenMEVA & 500 & 0.0333 & 100.0 & 35.2 & 0.0 \\
\bottomrule
\end{tabular}

\end{tabular}
\end{adjustbox}
\end{table}

\begin{table}[H]
\centering
\scriptsize
\caption{Rater count and rating noise on SummEval (30\% calibration).  The proxy is a fixed held-out human rater; calibration gold averages $k$ of the $K$ available non-proxy raters, with optional per-rating noise at nominal scales of the group rating SD, clipped to the observed range; evaluation gold is always the clean mean of all non-proxy raters.  $\hat B$ error is the mean $|\hat B-B_{\rm eval}|$.  The sign reference differs by block, as each block header states.  The last two columns give the paired CurveMAE difference and its 5,000-draw path-cluster bootstrap interval.  Values are not comparable with Table~\ref{tab:prediction-main}b.}
\label{tab:rater-noise}
\begin{adjustbox}{max width=\textwidth}
\begin{tabular}{lrrrrc}
\toprule
 &  &  & \multicolumn{3}{c}{CurveMAE} \\
\cmidrule(lr){4-6}Setting & $\hat B$ error & Sign (\%) & $A{+}B$ & $A{+}B-A$-only & 95\% CI \\
\midrule
\multicolumn{6}{l}{\textit{Clean ratings, $k$ of $K$ gold raters; sign vs.\ $B_{\rm eval}$}} \\
Expert, 1 of 2 & 0.0618 & 97.3 & 0.02687 & $-$0.01246 & [$-$0.01814, $-$0.00693] \\
Expert, 2 of 2 & 0.0428 & 100.0 & 0.01437 & $-$0.01343 & [$-$0.02012, $-$0.00708] \\
Crowd, 1 of 4 & 0.0659 & 81.6 & 0.01224 & $-$0.00127 & [$-$0.00160, $-$0.00096] \\
Crowd, 2 of 4 & 0.0612 & 84.0 & 0.01083 & $-$0.00135 & [$-$0.00174, $-$0.00093] \\
Crowd, 3 of 4 & 0.0596 & 85.8 & 0.01059 & $-$0.00146 & [$-$0.00180, $-$0.00112] \\
Crowd, 4 of 4 & 0.0590 & 85.4 & 0.01037 & $-$0.00149 & [$-$0.00181, $-$0.00117] \\
\midrule
\multicolumn{6}{l}{\textit{Noise (rating SD), all $K$ gold raters; sign vs.\ clean $\hat B$}} \\
Expert, Gauss. 0.25 & 0.0454 & 100.0 & 0.01469 & $-$0.00865 & [$-$0.01379, $-$0.00354] \\
Expert, Gauss. 0.50 & 0.0541 & 99.7 & 0.01826 & $-$0.00343 & [$-$0.00731, 0.00031] \\
Expert, Gauss. 1.00 & 0.0835 & 99.5 & 0.03253 & 0.00657 & [0.00465, 0.00857] \\
Expert, $t_3$ 0.25 & 0.0446 & 99.8 & 0.01444 & $-$0.00967 & [$-$0.01545, $-$0.00415] \\
Expert, $t_3$ 0.50 & 0.0508 & 99.5 & 0.01662 & $-$0.00551 & [$-$0.01011, $-$0.00113] \\
Expert, $t_3$ 1.00 & 0.0709 & 99.3 & 0.02540 & 0.00236 & [0.00041, 0.00427] \\
Crowd, Gauss. 0.25 & 0.0585 & 96.6 & 0.01082 & $-$0.00031 & [$-$0.00068, 0.00008] \\
Crowd, Gauss. 0.50 & 0.0592 & 94.2 & 0.01304 & 0.00110 & [0.00065, 0.00163] \\
Crowd, Gauss. 1.00 & 0.0585 & 93.0 & 0.02189 & 0.00335 & [0.00289, 0.00384] \\
Crowd, $t_3$ 0.25 & 0.0581 & 96.9 & 0.01049 & $-$0.00062 & [$-$0.00097, $-$0.00023] \\
Crowd, $t_3$ 0.50 & 0.0576 & 95.3 & 0.01186 & 0.00054 & [0.00010, 0.00101] \\
Crowd, $t_3$ 1.00 & 0.0569 & 93.3 & 0.01760 & 0.00263 & [0.00215, 0.00316] \\
\bottomrule
\end{tabular}

\end{adjustbox}
\end{table}

\begin{figure}[H]
\centering
\includegraphics[width=\linewidth]{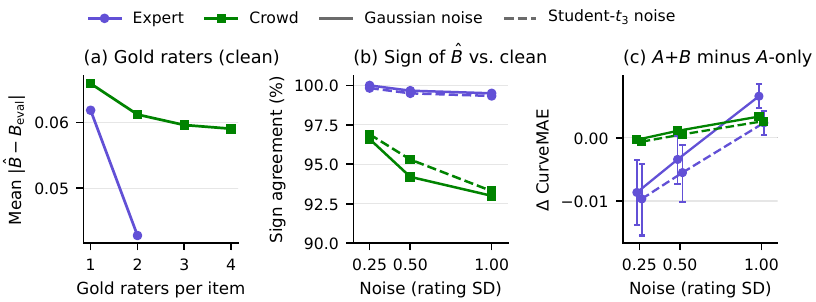}
\caption{Sensitivity of $\hat B$ to the calibration gold on SummEval.  (a) Mean $|\hat B-B_{\rm eval}|$ against the number of clean gold raters.  (b) Sign agreement of $\hat B$ with the clean estimate at the same rater count under injected noise (nominal scale, all gold raters).  (c) Paired CurveMAE difference $A{+}B$ minus $A$-only with 95\% path-cluster intervals: the sign of $\hat B$ stays stable, but its predictive benefit disappears at 1.0 SD for experts and 0.5 SD for crowd workers.}
\label{fig:rater-noise}
\end{figure}

\FloatBarrier
\subsection{Open reward-model and judge proxies}
\label{app:open-proxies}

To check that the external-proxy result is not specific to one judge, we score the same 1,600 SummEval summaries with an open reward model and an open LLM judge and keep the SummEval human ratings as calibration and evaluation targets.  The reward model is Skywork-Reward-Llama-3.1-8B-v0.2 \citep{liu2024skywork} (revision \texttt{d4117fbf}, bfloat16, maximum length 4,096), which gives one scalar score per summary.  The judge is Qwen2.5-32B-Instruct \citep{qwen2025qwen25} (revision \texttt{5ede1c97}), which rates each dimension with a rubric prompt that elicits a rationale and a real-valued score; five passes at temperature 0.3 are averaged.  The protocol is the prospective protocol of Section~\ref{sec:protocol}: per dimension, 480 calibration and 1,120 evaluation items, 50 repeats, and endpoint MAE at $d=0.10$ against the exact held-out raw-score tilt gain.

Table~\ref{tab:open-proxies} shows that adding $B$ reduces the endpoint MAE of $A$-only by 27.1\% and 24.5\%.  The calibration two-term fit is slightly lower than $A{+}B$ for both proxies and wins about half of the splits; for Qwen, its path-bootstrap interval relative to $A{+}B$, $[-0.00050,-0.00003]$, lies entirely below zero.  $A{+}B$ beats $A$-only in 66.0\% and 60.5\% of individual splits, and Qwen coherence is the only one of eight proxy--dimension cells in which $A$-only is lower (Table~\ref{tab:open-proxies-dim}).  These checks are scoped to SummEval and show that the second-order benefit is not an artifact of one judge; they do not establish broad proxy generality.  Training-data overlap between these models and SummEval is not excluded.

\begin{table}[H]
\centering
\scriptsize
\caption{Open proxies on SummEval: endpoint MAE at $d=0.10$, averaged over four dimensions and 50 item splits (480/1,120).  The paired interval for $A$-only minus $A{+}B$ is a path-clustered bootstrap over only four dimension paths (500 draws), computed from the saved split rows.}
\label{tab:open-proxies}
\begin{adjustbox}{max width=\textwidth}
\begin{tabular}{lrrrrrrc}
\toprule
 &  &  & Emp. &  &  & $A{+}B$ & $A$-only $-$ $A{+}B$ \\
Proxy & $A$-only & $A{+}B$ & fit & $A_\nu$ & Zero & reduction & [path-bootstrap 95\%] \\
\midrule
Skywork RM v0.2 & 0.0155 & 0.0113 & \textbf{0.0111} & 0.0438 & 0.0665 & 27.1\% & 0.00419 [0.00106, 0.00732] \\
Qwen2.5-32B judge & 0.0148 & 0.0112 & \textbf{0.0110} & 0.0964 & 0.0817 & 24.5\% & 0.00364 [0.00026, 0.00992] \\
\bottomrule
\end{tabular}

\end{adjustbox}
\end{table}

\begin{table}[H]
\centering
\scriptsize
\caption{Open proxies on SummEval by dimension (endpoint MAE at $d=0.10$, 50 splits each).}
\label{tab:open-proxies-dim}
\begin{adjustbox}{max width=\textwidth}
\begin{tabular}{llrrrrr}
\toprule
Proxy & Dimension & $A$-only & $A{+}B$ & Emp. fit & $A_\nu$ & Zero \\
\midrule
Skywork RM v0.2 & coherence & 0.0137 & 0.0131 & \textbf{0.0126} & 0.0527 & 0.0789 \\
 & consistency & 0.0223 & 0.0133 & \textbf{0.0132} & 0.0435 & 0.0674 \\
 & fluency & 0.0161 & 0.0105 & \textbf{0.0104} & 0.0198 & 0.0403 \\
 & relevance & 0.0097 & \textbf{0.0081} & 0.0082 & 0.0590 & 0.0792 \\
\midrule
Qwen2.5-32B judge & coherence & \textbf{0.0120} & 0.0124 & 0.0123 & 0.1107 & 0.0973 \\
 & consistency & 0.0266 & 0.0138 & \textbf{0.0131} & 0.0737 & 0.0874 \\
 & fluency & 0.0116 & 0.0104 & \textbf{0.0102} & 0.1442 & 0.0490 \\
 & relevance & 0.0092 & \textbf{0.0082} & 0.0083 & 0.0568 & 0.0931 \\
\bottomrule
\end{tabular}

\end{adjustbox}
\end{table}

\FloatBarrier
\subsection{Three reward models: gain forecasts and aggregate errors}
\label{app:rm-example}

This evaluation shows how $\hat A$ and $\hat B$ are used for fixed reward-model proxies; it is not a reward-model leaderboard.  The proxies are three frozen scalar reward models: Skywork-Reward-Llama-3.1-8B-v0.2 (revision \texttt{d4117fbf}), ArmoRM-Llama3-8B-v0.1 \citep{wang2024armorm} (revision \texttt{eb2676d2}, final gated scalar), and OpenAssistant reward-model-deberta-v3-large-v2 \citep{openassistant2023reward,kopf2023openassistant,he2023debertav3} (revision \texttt{c355404e}).  There is no fine-tuning and no new annotation.  The target for each SummEval dimension is the mean of three expert and five crowd ratings.  Every model and summary for an article receive the same article prefix.  The supplied articles are already capped at 2,000 characters, and prefixes are further shortened when needed to fit all three tokenizers (4,096, 4,096, and 512 tokens) while keeping each summary in full; the mean retained fraction is 96.08\% of the supplied prefixes.  Predictions use 480 calibration and 1,120 evaluation items, 50 repeats (seed 0), and radii $0.025$, $0.05$, $0.075$, and $0.10$, with splits identical across models.  The splits hold out items, not articles.

Figure~\ref{fig:rm-worked} summarizes coherence across all 50 matched splits, showing mean gain forecasts and mean within-split absolute errors separately.  Over all 50 repeats and four dimensions (Table~\ref{tab:rm-endpoint}), adding $B$ lowers the endpoint MAE in all 12 model--dimension cells and in all 48 cells across the four radii, while the calibration two-term fit is lower than $A{+}B$ in 7 of 12 endpoint cells.  The repeated splits reuse the same 1,600 items, so split-to-split variation should not be interpreted as independent-sample uncertainty.  Table~\ref{tab:rm-coefficients} reports full-pool descriptive coefficients with paired article-cluster bootstrap intervals.  These differ from the calibration-only coefficients used for prediction because the estimation sample differs.  The Skywork coherence interval for $B$ includes zero, and the other eleven $B$ intervals are negative.

The coefficients describe this proxy, target, candidate distribution, and context; they are not an intrinsic quality score of a reward model.  They are invariant to positive affine changes of the proxy but not to rescaling the target or shifting the population, and a negative $B$ is a local correction rather than a turning point.  The reliability-only surrogate is omitted, because its variance ratio depends on proxy scale (ArmoRM's scalar reward has standard deviation 0.022).  The article context was preprocessed on this benchmark's texts before splitting, and training-data overlap with SummEval is not excluded.  These shared-context Skywork scores differ from the open-proxy scores of Appendix~\ref{app:open-proxies}.

\begin{figure}[H]
\centering
\includegraphics[width=\linewidth]{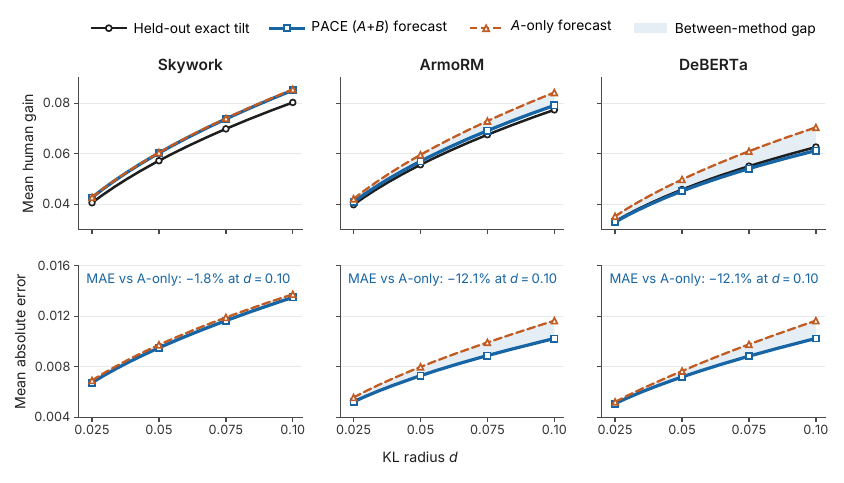}
\caption{SummEval coherence forecasts across 50 matched item splits (480 calibration / 1,120 held-out items).  Top: mean exact held-out tilt gain (black), calibration-only PACE ($A{+}B$, blue), and $A$-only forecasts (orange dashed).  Bottom: absolute error computed within each split, then averaged.  Percentages give PACE's relative MAE change against $A$-only at $d=0.10$.  Shading marks the between-method gap, not uncertainty: the mean $\hat B d$ correction above and the MAE difference below.  Axes are zoomed to show $d\in[0.025,0.10]$; neither vertical axis starts at zero.  The repeated splits reuse the same 1,600 items.}
\label{fig:rm-worked}
\end{figure}

\begin{table}[H]
\centering
\scriptsize
\caption{Three reward models: endpoint MAE at $d=0.10$ (not CurveMAE), averaged equally over four SummEval dimensions and 50 item splits per dimension (480/1,120), in human-score units.  Point estimates only.  The lower block lists each dimension.}
\label{tab:rm-endpoint}
\begingroup
\let\bottomrule\midrule
\begin{adjustbox}{max width=\textwidth}
\begin{tabular}{lrrrr}
\toprule
 &  &  & Calibration & $A{+}B$ reduction \\
Reward model & $A$-only & $A{+}B$ & empirical fit & vs $A$-only \\
\midrule
Skywork & 0.0155 & 0.0113 & \textbf{0.0111} & 27.5\% \\
ArmoRM & 0.0136 & \textbf{0.0105} & \textbf{0.0105} & 23.1\% \\
DeBERTa & 0.0151 & \textbf{0.0113} & \textbf{0.0113} & 25.1\% \\
\bottomrule
\end{tabular}

\end{adjustbox}
\endgroup
\par\vspace{10pt}
\begingroup
\let\toprule\midrule
\begin{adjustbox}{max width=\textwidth}
\begin{tabular}{llrrrr}
\toprule
 &  & Mean held-out &  &  & Calibration \\
Dimension & Reward model & gain & $A$-only & $A{+}B$ & empirical fit \\
\midrule
coherence & Skywork & 0.0803 & 0.0137 & 0.0135 & \textbf{0.0129} \\
 & ArmoRM & 0.0773 & 0.0116 & 0.0102 & \textbf{0.0101} \\
 & DeBERTa & 0.0626 & 0.0116 & 0.0102 & \textbf{0.0101} \\
\midrule
consistency & Skywork & 0.0668 & 0.0225 & \textbf{0.0134} & \textbf{0.0134} \\
 & ArmoRM & 0.0576 & 0.0166 & \textbf{0.0122} & 0.0124 \\
 & DeBERTa & 0.0454 & 0.0199 & 0.0135 & \textbf{0.0133} \\
\midrule
fluency & Skywork & 0.0406 & 0.0162 & 0.0103 & \textbf{0.0102} \\
 & ArmoRM & 0.0482 & 0.0151 & \textbf{0.0097} & 0.0098 \\
 & DeBERTa & 0.0330 & 0.0172 & \textbf{0.0111} & 0.0112 \\
\midrule
relevance & Skywork & 0.0803 & 0.0097 & \textbf{0.0079} & \textbf{0.0079} \\
 & ArmoRM & 0.0778 & 0.0111 & \textbf{0.0097} & \textbf{0.0097} \\
 & DeBERTa & 0.0592 & 0.0116 & \textbf{0.0104} & \textbf{0.0104} \\
\bottomrule
\end{tabular}

\end{adjustbox}
\endgroup
\end{table}

\begin{table}[H]
\centering
\scriptsize
\caption{Full-pool descriptive coefficients for the three reward models (all 1,600 items), with 95\% percentile intervals from 1,000 paired article-cluster bootstrap draws (seed 1; the same articles are drawn for every model and dimension).  $B$ equals the tail term plus the skewness term.  These are not the calibration-only coefficients used in Figure~\ref{fig:rm-worked} and Table~\ref{tab:rm-endpoint}, and marginal intervals are not a test of differences between models.}
\label{tab:rm-coefficients}
\begin{adjustbox}{max width=\textwidth}
\begin{tabular}{llrrrr}
\toprule
Reward model & Human dimension & $A$ [95\% CI] & $B$ [95\% CI] & Tail term & Skewness term \\
\midrule
Skywork & coherence & 0.2602 [0.2087, 0.3154] & 0.0086 [$-$0.0403, 0.0623] & $-$0.0816 & 0.0902 \\
ArmoRM & coherence & 0.2615 [0.2105, 0.3118] & $-$0.0505 [$-$0.0972, $-$0.0037] & $-$0.1153 & 0.0649 \\
DeBERTa & coherence & 0.2232 [0.1645, 0.2813] & $-$0.0908 [$-$0.1375, $-$0.0376] & $-$0.1230 & 0.0322 \\
\midrule
Skywork & consistency & 0.2722 [0.2230, 0.3223] & $-$0.1991 [$-$0.2560, $-$0.1384] & $-$0.2934 & 0.0944 \\
ArmoRM & consistency & 0.2226 [0.1662, 0.2811] & $-$0.1322 [$-$0.1906, $-$0.0745] & $-$0.1875 & 0.0552 \\
DeBERTa & consistency & 0.1906 [0.1349, 0.2493] & $-$0.1558 [$-$0.2129, $-$0.0992] & $-$0.1833 & 0.0275 \\
\midrule
Skywork & fluency & 0.1694 [0.1193, 0.2257] & $-$0.1229 [$-$0.1905, $-$0.0528] & $-$0.1816 & 0.0587 \\
ArmoRM & fluency & 0.1877 [0.1416, 0.2393] & $-$0.1142 [$-$0.1722, $-$0.0579] & $-$0.1607 & 0.0465 \\
DeBERTa & fluency & 0.1410 [0.0883, 0.1956] & $-$0.1145 [$-$0.1742, $-$0.0564] & $-$0.1348 & 0.0203 \\
\midrule
Skywork & relevance & 0.2750 [0.2296, 0.3257] & $-$0.0644 [$-$0.1139, $-$0.0161] & $-$0.1598 & 0.0953 \\
ArmoRM & relevance & 0.2675 [0.2241, 0.3118] & $-$0.0749 [$-$0.1174, $-$0.0338] & $-$0.1412 & 0.0663 \\
DeBERTa & relevance & 0.2115 [0.1625, 0.2628] & $-$0.0745 [$-$0.1145, $-$0.0378] & $-$0.1049 & 0.0305 \\
\bottomrule
\end{tabular}

\end{adjustbox}
\end{table}

\FloatBarrier
\subsection{Radius, calibration-size, and selective-use robustness checks}
\label{app:accept-boost}

This subsection reports robustness checks of the local diagnostic.  The KL-radius sweep asks whether the headline $d=0.10$ result is a single-radius artifact.  The calibration-size curve asks how many repeated-rating calibration items are needed before the coefficient predictor stabilizes.  The residual breakdown checks the finite-radius approximation error of the local law itself.  The selective GPT-4o gate summarizes how margin thresholds trade coverage against realized gain in the external-proxy setting.

The tables and figures in this subsection summarize the radius sweep, calibration-size sweep, residual breakdown, and selective GPT-4o gate.  Table~\ref{tab:accept-boost-selected} gives a compact selection of those results.

\begin{figure}[H]
\centering
\includegraphics[width=0.92\linewidth]{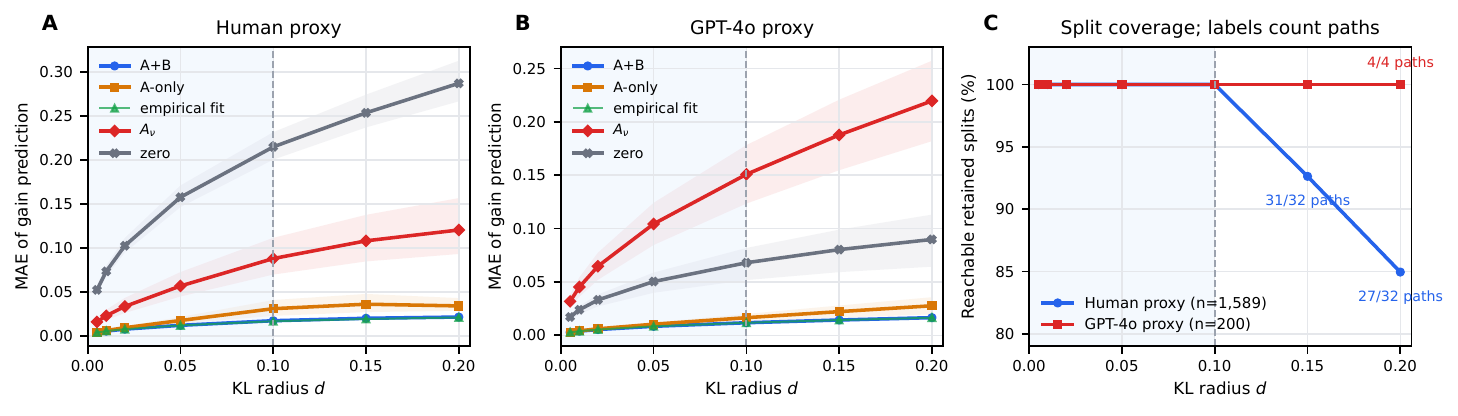}
\caption{KL-radius audit for the shared local window.  The shaded region and vertical dashed line mark $d\le0.10$, the window used for the main prospective prediction tables.  Panel A plots MAE versus KL radius on the human-proxy setting: the $A{+}B$ predictor remains below the $A$-only and zero baselines throughout the shared local grid.  Panel B gives the same comparison for the GPT-4o proxy, where $A{+}B$ remains below $A$-only, $A_\nu$, and zero from $d=0.005$ to $d=0.20$.  Ribbons or intervals, where shown, are 95\% path-cluster bootstrap intervals for pooled-split MAE.  Panel C separates split coverage (the line) from path counts (annotations). For SummEval, coverage is relative to 1,589 retained splits: 100\% through $d=0.10$, $1{,}472/1{,}589=92.6\%$ at $0.15$, and $1{,}350/1{,}589=85.0\%$ at $0.20$; the corresponding path counts are 32, 31, and 27. All four GPT-4o paths remain reachable.}
\label{fig:kl-radius-audit}
\end{figure}

For finite empirical tilt paths, larger KL radii are not always attainable on every split.  We therefore report each radius on the rows that can realize that radius, while retaining lower-radius rows from the same paths.  This matters mainly for the high-curvature human-proxy stress window: all 32 paths are retained through $d=0.10$, while the larger $d=0.20$ sweep has 27 valid paths and 1,350 valid split rows.

\begin{table}[H]
\centering
\scriptsize
\caption{Selected robustness-check results for the radius, calibration-size, and selective-use analyses.  ``Valid'' counts refer to splits or paths that can realize the requested KL radius under the finite empirical exponential-tilt path.}
\label{tab:accept-boost-selected}
\begin{adjustbox}{max width=\textwidth}
\begin{tabular}{lllll}
\toprule
Setting & Diagnostic & Valid count & Radius / threshold & Result \\
\midrule
Human proxy & radius sweep & 32 paths / 1,589 rows & $d=0.10$ & MAE $0.0174$ [0.0158, 0.0193]. $A$-only $0.0310$. $A_\nu$ $0.0880$ \\
Human proxy & radius sweep & 27 paths / 1,350 rows & $d=0.20$ & MAE $0.0216$ [0.0206, 0.0226]. $A$-only $0.0342$. $A_\nu$ $0.1204$ \\
Human proxy & calibration size & 32 paths / 1,571 rows & 10\%, $d=0.10$ & MAE $0.0256$ [0.0232, 0.0287] \\
Human proxy & calibration size & 32 paths / 1,589 rows & 30\%, $d=0.10$ & MAE $0.0174$ [0.0158, 0.0193] \\
Human proxy & residual breakdown & 32 paths & $D\le0.10$ & median max residual $0.0017$. median $R^2=0.99983$ \\
Human proxy & residual breakdown & 27 paths & $D\le0.20$ & median max residual $0.0041$. median $R^2=0.99946$ \\
GPT-4o proxy & radius sweep & 4 paths / 200 rows & $d=0.10$ & MAE $0.0118$ [0.0104, 0.0125]. zero $0.0680$ \\
GPT-4o proxy & selective gate & 200 rows & $\tau=0.01$ & 100.0\% coverage. 100.0\% covered sign accuracy. realized gain $0.0680$ \\
GPT-4o proxy & selective gate & 200 rows & $\tau=0.02$ & 99.5\% coverage. 100.0\% covered sign accuracy. realized gain $0.0677$ \\
\bottomrule
\end{tabular}
\end{adjustbox}
\end{table}

Tables~\ref{tab:radius-sweep-detail} and~\ref{tab:calibration-size-summeval-gpt4o} report the full radius sweep and the SummEval/GPT-4o calibration-size curves that underlie the selected rows above.  Brackets are 95\% path-cluster bootstrap intervals for pooled-split MAE.

\begin{table}[H]
\centering
\scriptsize
\caption{Full KL-radius sweep at $d\in\{0.005,0.01,0.02,0.05,0.10,0.15,0.20\}$ for the human-proxy SummEval setting and the GPT-4o external proxy.  Valid path counts at larger radii reflect finite-empirical tilt-path reachability.}
\label{tab:radius-sweep-detail}
\begin{adjustbox}{max width=\textwidth}
\begin{tabular}{lllllllll}
\toprule
Setting & $d$ & Valid paths & Valid rows & A+B MAE & A-only MAE & Zero MAE & $A_\nu$ MAE & A+B sign acc. \\
\midrule
Human proxy & 0.005 & 32 & 1589 & \textbf{0.0040} [0.0037, 0.0043] & 0.0043 [0.0039, 0.0046] & 0.0524 [0.0482, 0.0572] & 0.0160 [0.0125, 0.0207] & 100.0\% \\
Human proxy & 0.010 & 32 & 1589 & \textbf{0.0056} [0.0052, 0.0060] & 0.0062 [0.0056, 0.0069] & 0.0734 [0.0676, 0.0798] & 0.0230 [0.0180, 0.0298] & 100.0\% \\
Human proxy & 0.020 & 32 & 1589 & \textbf{0.0078} [0.0073, 0.0084] & 0.0094 [0.0082, 0.0108] & 0.1025 [0.0947, 0.1113] & 0.0335 [0.0262, 0.0433] & 100.0\% \\
Human proxy & 0.050 & 32 & 1589 & \textbf{0.0121} [0.0113, 0.0132] & 0.0176 [0.0143, 0.0217] & 0.1576 [0.1464, 0.1709] & 0.0568 [0.0448, 0.0726] & 100.0\% \\
Human proxy & 0.100 & 32 & 1589 & \textbf{0.0174} [0.0158, 0.0193] & 0.0310 [0.0229, 0.0409] & 0.2147 [0.2002, 0.2314] & 0.0880 [0.0698, 0.1115] & 100.0\% \\
Human proxy & 0.150 & 31 & 1472 & \textbf{0.0201} [0.0185, 0.0222] & 0.0361 [0.0272, 0.0478] & 0.2534 [0.2368, 0.2740] & 0.1081 [0.0847, 0.1375] & 100.0\% \\
Human proxy & 0.200 & 27 & 1350 & \textbf{0.0216} [0.0206, 0.0226] & 0.0342 [0.0286, 0.0431] & 0.2869 [0.2664, 0.3125] & 0.1204 [0.0931, 0.1566] & 100.0\% \\
GPT-4o proxy & 0.005 & 4 & 200 & \textbf{0.0028} [0.0024, 0.0030] & 0.0029 [0.0025, 0.0032] & 0.0171 [0.0143, 0.0196] & 0.0318 [0.0253, 0.0382] & 100.0\% \\
GPT-4o proxy & 0.010 & 4 & 200 & \textbf{0.0039} [0.0034, 0.0041] & 0.0042 [0.0036, 0.0046] & 0.0239 [0.0198, 0.0274] & 0.0453 [0.0362, 0.0544] & 100.0\% \\
GPT-4o proxy & 0.020 & 4 & 200 & \textbf{0.0054} [0.0047, 0.0058] & 0.0061 [0.0052, 0.0068] & 0.0331 [0.0270, 0.0382] & 0.0647 [0.0520, 0.0774] & 100.0\% \\
GPT-4o proxy & 0.050 & 4 & 200 & \textbf{0.0084} [0.0074, 0.0090] & 0.0104 [0.0088, 0.0120] & 0.0503 [0.0398, 0.0588] & 0.1044 [0.0846, 0.1241] & 100.0\% \\
GPT-4o proxy & 0.100 & 4 & 200 & \textbf{0.0118} [0.0104, 0.0125] & 0.0165 [0.0130, 0.0200] & 0.0680 [0.0517, 0.0818] & 0.1508 [0.1234, 0.1782] & 100.0\% \\
GPT-4o proxy & 0.150 & 4 & 200 & \textbf{0.0143} [0.0127, 0.0154] & 0.0222 [0.0165, 0.0279] & 0.0803 [0.0592, 0.0990] & 0.1877 [0.1547, 0.2207] & 100.0\% \\
GPT-4o proxy & 0.200 & 4 & 200 & \textbf{0.0165} [0.0148, 0.0178] & 0.0277 [0.0196, 0.0357] & 0.0898 [0.0643, 0.1131] & 0.2196 [0.1819, 0.2573] & 100.0\% \\
\bottomrule
\end{tabular}

\end{adjustbox}
\end{table}

\begin{table}[H]
\centering
\scriptsize
\caption{Full calibration-size curve at $d=0.10$ for the human-proxy SummEval setting and the GPT-4o external proxy.  Calibration fractions $\{10,20,30,50,70\}\%$ are evaluated on the same proxy--gold paths.}
\label{tab:calibration-size-summeval-gpt4o}
\begin{adjustbox}{max width=\textwidth}
\begin{tabular}{llllll}
\toprule
Setting & Calibration & A+B MAE & A-only MAE & Zero MAE & $A_\nu$ MAE \\
\midrule
Human proxy & 10\% & \textbf{0.0256} [0.0232, 0.0287] & 0.0368 [0.0295, 0.0454] & 0.2144 [0.1998, 0.2311] & 0.0876 [0.0694, 0.1112] \\
Human proxy & 20\% & \textbf{0.0199} [0.0180, 0.0221] & 0.0325 [0.0247, 0.0417] & 0.2147 [0.2000, 0.2315] & 0.0880 [0.0696, 0.1116] \\
Human proxy & 30\% & \textbf{0.0174} [0.0158, 0.0193] & 0.0310 [0.0229, 0.0409] & 0.2147 [0.2002, 0.2314] & 0.0880 [0.0698, 0.1115] \\
Human proxy & 50\% & \textbf{0.0163} [0.0147, 0.0182] & 0.0305 [0.0219, 0.0407] & 0.2149 [0.2002, 0.2317] & 0.0891 [0.0711, 0.1127] \\
Human proxy & 70\% & \textbf{0.0170} [0.0153, 0.0190] & 0.0304 [0.0223, 0.0401] & 0.2149 [0.1998, 0.2322] & 0.0889 [0.0708, 0.1120] \\
GPT-4o proxy & 10\% & \textbf{0.0186} [0.0159, 0.0204] & 0.0214 [0.0171, 0.0246] & 0.0682 [0.0518, 0.0823] & 0.1486 [0.1190, 0.1782] \\
GPT-4o proxy & 20\% & \textbf{0.0137} [0.0115, 0.0154] & 0.0182 [0.0145, 0.0219] & 0.0681 [0.0515, 0.0822] & 0.1503 [0.1226, 0.1780] \\
GPT-4o proxy & 30\% & \textbf{0.0118} [0.0104, 0.0125] & 0.0165 [0.0130, 0.0200] & 0.0680 [0.0517, 0.0818] & 0.1508 [0.1234, 0.1782] \\
GPT-4o proxy & 50\% & \textbf{0.0093} [0.0090, 0.0096] & 0.0143 [0.0099, 0.0186] & 0.0685 [0.0509, 0.0829] & 0.1495 [0.1217, 0.1772] \\
GPT-4o proxy & 70\% & \textbf{0.0112} [0.0099, 0.0126] & 0.0159 [0.0122, 0.0199] & 0.0684 [0.0492, 0.0841] & 0.1493 [0.1212, 0.1773] \\
\bottomrule
\end{tabular}

\end{adjustbox}
\end{table}

The calibration-size robustness check also repeats the same learning-curve experiment on the WMT MQM cluster and OpenMEVA-MAGS ROC/WP slices.  The pattern is consistent with the SummEval result: most of the gain arrives by 30--50\% calibration, the coefficient predictor remains below the reliability-only and zero baselines, and very large calibration fractions can slightly increase variance because the evaluation split becomes smaller.

\begin{table}[H]
\centering
\scriptsize
\caption{Calibration-size robustness on additional eligible repeated-rating analysis slices at $d=0.10$.}
\label{tab:calibration-size-extended}
\begin{adjustbox}{max width=\textwidth}
\begin{tabular}{llllllll}
\toprule
Setting & Calibration & Splits & Paths & A+B MAE & A-only MAE & $A_\nu$ MAE & Zero MAE \\
\midrule
WMT MQM clusters & 10\% & 7050 & 141 & \textbf{0.1327} & 0.1740 & 0.6410 & 0.8527 \\
WMT MQM clusters & 20\% & 7050 & 141 & \textbf{0.0992} & 0.1421 & 0.6286 & 0.8529 \\
WMT MQM clusters & 30\% & 7050 & 141 & \textbf{0.0846} & 0.1298 & 0.6230 & 0.8524 \\
WMT MQM clusters & 50\% & 7050 & 141 & \textbf{0.0783} & 0.1238 & 0.6174 & 0.8523 \\
WMT MQM clusters & 70\% & 7050 & 141 & \textbf{0.0851} & 0.1281 & 0.6195 & 0.8505 \\
OpenMEVA-MAGS & 10\% & 500 & 10 & \textbf{0.0231} & 0.0233 & 0.0281 & 0.3327 \\
OpenMEVA-MAGS & 20\% & 500 & 10 & \textbf{0.0168} & 0.0178 & 0.0226 & 0.3328 \\
OpenMEVA-MAGS & 30\% & 500 & 10 & \textbf{0.0148} & 0.0162 & 0.0207 & 0.3327 \\
OpenMEVA-MAGS & 50\% & 500 & 10 & \textbf{0.0140} & 0.0154 & 0.0199 & 0.3326 \\
OpenMEVA-MAGS & 70\% & 500 & 10 & \textbf{0.0157} & 0.0165 & 0.0214 & 0.3317 \\
\bottomrule
\end{tabular}

\end{adjustbox}
\end{table}

Reachability filtering at larger radii does not drive the headline result.  At $d=0.15$, the reachable subset differs from the full $d=0.10$ population by only $-0.0020$ mean gain.  At $d=0.20$, the reachable subset has slightly lower mean gain and less negative calibration $B$, so the stress-window result should be read as a reachable-radius audit rather than a full-population claim.

\begin{table}[H]
\centering
\scriptsize
\caption{Reachability audit for the high-radius human-proxy sweep.  Deltas compare the reachable-path subset with all 32 paths at the shared $d=0.10$ radius.}
\label{tab:reachability-audit}
\begin{adjustbox}{max width=\textwidth}
\begin{tabular}{lrrrrrrr}
\toprule
Stress radius & Valid paths & Valid rows & Mean gain full & Mean gain reachable & $\Delta$ gain & $\Delta$ cal. $A$ & $\Delta$ cal. $B$ \\
\midrule
0.15 & 31 & 1,472 & 0.2147 & 0.2127 & -0.0020 & -0.0136 & 0.0161 \\
0.20 & 27 & 1,350 & 0.2147 & 0.2089 & -0.0058 & -0.0552 & 0.0957 \\
\bottomrule
\end{tabular}
\end{adjustbox}
\end{table}

Because the object is a gain curve rather than a single endpoint, Table~\ref{tab:curve-level-metrics} expands the compact main-table view with grid-mean absolute error and maximum error over $d\in\{0.005,0.01,0.02,0.05,0.10\}$, including descriptive empirical-fit and oracle references.  The ranking is consistent with the endpoint result: the coefficient curve is strong for human-proxy paths, and on GPT-4o it remains below $A$-only, $A_\nu$, and zero on the shared local grid.

\begin{table}[H]
\centering
\scriptsize
\caption{Expanded curve-level prediction errors on the shared local grid $d\le0.10$.  Grid MAE averages absolute error over $d\in\{0.005,0.01,0.02,0.05,0.10\}$ within each curve, then across curves; this grid differs from Table~\ref{tab:prediction-main}b.  Max-error columns summarize per-curve sup-norm error.}
\label{tab:curve-level-metrics}
\begin{adjustbox}{max width=\textwidth}
\begin{tabular}{llllll}
\toprule
Setting & Predictor & Curves & \reviewedit{Grid MAE}{Grid IAE} & Median max error & 90\% max error \\
\midrule
SummEval human & A+B & 1589 & 0.0094 & 0.0143 & 0.0361 \\
SummEval human & A-only & 1589 & 0.0137 & 0.0210 & 0.0812 \\
SummEval human & empirical fit & 1589 & \textbf{0.0093} & \textbf{0.0137} & \textbf{0.0342} \\
SummEval human & $A_\nu$ & 1589 & 0.0434 & 0.0676 & 0.1744 \\
SummEval human & zero & 1589 & 0.1201 & 0.1988 & 0.3003 \\
SummEval human & oracle & 1589 & 0.0012 & 0.0018 & 0.0141 \\
GPT-4o proxy & A+B & 200 & \textbf{0.0064} & \textbf{0.0095} & 0.0231 \\
GPT-4o proxy & A-only & 200 & 0.0080 & 0.0153 & 0.0312 \\
GPT-4o proxy & empirical fit & 200 & \textbf{0.0064} & 0.0099 & \textbf{0.0226} \\
GPT-4o proxy & $A_\nu$ & 200 & 0.0794 & 0.1521 & 0.1920 \\
GPT-4o proxy & zero & 200 & 0.0385 & 0.0707 & 0.0876 \\
GPT-4o proxy & oracle & 200 & 0.0002 & 0.0004 & 0.0012 \\
\bottomrule
\end{tabular}

\end{adjustbox}
\end{table}

Table~\ref{tab:curve-residuals} reports the finite-radius residual sweep.  The $D\le0.05$ human-proxy row has median maximum residual $0.0006$ for the coefficient predictor, and the shared $D\le0.10$ row remains small relative to the realized gains.  This table separates finite-radius Taylor error from endpoint-prediction error.

\begin{table}[H]
\centering
\scriptsize
\caption{Finite-radius residual sweep for the local gain curve.  $D$ is the largest KL radius included in the curve residual computation.  This is a same-pool approximation audit: $A{+}B$ uses moments of the pool on which the exact curve is computed, not coefficients from a separate calibration split.  $A$-only drops curvature, $A_\nu$ is the reliability-only first-order approximation, and zero predicts no gain.  The two-term fit ($a\sqrt d+bd$) and three-term fit ($a\sqrt d+bd+cd^{3/2}$) are least-squares fits to this same exact curve.}
\label{tab:curve-residuals}
\begin{adjustbox}{max width=\textwidth}
\begin{tabular}{llllllll}
\toprule
Setting & $D$ & Paths & Predictor & Median max abs resid & 90\% max abs resid & Median RMSE & Median $R^2$ \\
\midrule
Human proxy & 0.05 & 32 & $A{+}B$ & 0.0006 & 0.0040 & 0.0002 & 0.99996 \\
Human proxy & 0.05 & 32 & $A$-only & 0.0050 & 0.0368 & 0.0022 & 0.99725 \\
Human proxy & 0.05 & 32 & two-term fit & \textbf{0.0000} & 0.0003 & \textbf{0.0000} & 1.00000 \\
Human proxy & 0.05 & 32 & $A_\nu$ & 0.0424 & 0.0908 & 0.0245 & 0.62856 \\
Human proxy & 0.05 & 32 & zero & 0.1433 & 0.2265 & 0.0852 & -3.26251 \\
Human proxy & 0.05 & 32 & three-term fit & \textbf{0.0000} & \textbf{0.0000} & \textbf{0.0000} & 1.00000 \\
Human proxy & 0.10 & 32 & $A{+}B$ & 0.0017 & 0.0133 & 0.0007 & 0.99983 \\
Human proxy & 0.10 & 32 & $A$-only & 0.0105 & 0.0790 & 0.0048 & 0.99364 \\
Human proxy & 0.10 & 32 & two-term fit & 0.0001 & 0.0010 & \textbf{0.0000} & 1.00000 \\
Human proxy & 0.10 & 32 & $A_\nu$ & 0.0631 & 0.1626 & 0.0361 & 0.57232 \\
Human proxy & 0.10 & 32 & zero & 0.1982 & 0.2965 & 0.1198 & -3.36283 \\
Human proxy & 0.10 & 32 & three-term fit & \textbf{0.0000} & \textbf{0.0001} & \textbf{0.0000} & 1.00000 \\
Human proxy & 0.20 & 27 & $A{+}B$ & 0.0041 & 0.0086 & 0.0018 & 0.99946 \\
Human proxy & 0.20 & 27 & $A$-only & 0.0211 & 0.0405 & 0.0097 & 0.98607 \\
Human proxy & 0.20 & 27 & two-term fit & 0.0003 & 0.0005 & 0.0001 & 1.00000 \\
Human proxy & 0.20 & 27 & $A_\nu$ & 0.0949 & 0.1629 & 0.0544 & 0.50930 \\
Human proxy & 0.20 & 27 & zero & 0.2715 & 0.3399 & 0.1670 & -3.51251 \\
Human proxy & 0.20 & 27 & three-term fit & \textbf{0.0000} & \textbf{0.0000} & \textbf{0.0000} & 1.00000 \\
GPT-4o proxy & 0.05 & 4 & $A{+}B$ & 0.0001 & 0.0004 & \textbf{0.0000} & 0.99999 \\
GPT-4o proxy & 0.05 & 4 & $A$-only & 0.0062 & 0.0085 & 0.0029 & 0.95733 \\
GPT-4o proxy & 0.05 & 4 & two-term fit & \textbf{0.0000} & \textbf{0.0000} & \textbf{0.0000} & 1.00000 \\
GPT-4o proxy & 0.05 & 4 & $A_\nu$ & 0.1029 & 0.1240 & 0.0596 & -11.25104 \\
GPT-4o proxy & 0.05 & 4 & zero & 0.0523 & 0.0604 & 0.0318 & -3.44385 \\
GPT-4o proxy & 0.05 & 4 & three-term fit & \textbf{0.0000} & \textbf{0.0000} & \textbf{0.0000} & 1.00000 \\
GPT-4o proxy & 0.10 & 4 & $A{+}B$ & 0.0003 & 0.0010 & 0.0001 & 0.99997 \\
GPT-4o proxy & 0.10 & 4 & $A$-only & 0.0122 & 0.0169 & 0.0057 & 0.90243 \\
GPT-4o proxy & 0.10 & 4 & two-term fit & \textbf{0.0000} & 0.0001 & \textbf{0.0000} & 1.00000 \\
GPT-4o proxy & 0.10 & 4 & $A_\nu$ & 0.1484 & 0.1788 & 0.0860 & -12.78214 \\
GPT-4o proxy & 0.10 & 4 & zero & 0.0702 & 0.0830 & 0.0435 & -3.63254 \\
GPT-4o proxy & 0.10 & 4 & three-term fit & \textbf{0.0000} & \textbf{0.0000} & \textbf{0.0000} & 1.00000 \\
GPT-4o proxy & 0.20 & 4 & $A{+}B$ & 0.0008 & 0.0026 & 0.0003 & 0.99984 \\
GPT-4o proxy & 0.20 & 4 & $A$-only & 0.0239 & 0.0333 & 0.0111 & 0.76309 \\
GPT-4o proxy & 0.20 & 4 & two-term fit & 0.0001 & 0.0001 & \textbf{0.0000} & 1.00000 \\
GPT-4o proxy & 0.20 & 4 & $A_\nu$ & 0.2157 & 0.2593 & 0.1253 & -15.64416 \\
GPT-4o proxy & 0.20 & 4 & zero & 0.0919 & 0.1129 & 0.0587 & -3.95361 \\
GPT-4o proxy & 0.20 & 4 & three-term fit & \textbf{0.0000} & \textbf{0.0000} & \textbf{0.0000} & 1.00000 \\
\bottomrule
\end{tabular}

\end{adjustbox}
\end{table}

Table~\ref{tab:tail-outlier-robustness} and Figure~\ref{fig:tail-outlier-robustness} audit boundedness and tail sensitivity.  The figure should be read narrowly: Panel~A reports only correlation stability of the curvature coefficient $B$ relative to the original preprocessing, not $A$-stability or sign agreement, and Panel~B reports the median absolute change in exact endpoint gain at $d=0.10$ relative to the original preprocessing rather than the prediction error of the local law.  Table~\ref{tab:tail-outlier-robustness} therefore separates $A$-stability, $B$-correlation, $B$-sign agreement, gain-sign agreement, gain-shift size, and MAE.  Mild winsorization and controlled gold-side outliers are comparatively benign in this audit, rank-normalization can substantially change endpoint gain scale, especially in WMT, and proxy-side outliers can strongly perturb curvature, especially in OpenMEVA.  We therefore interpret plug-in $B$ as a useful local diagnostic rather than a tail-robust guarantee.

\begin{table}[tbp]
\centering
\scriptsize
\caption{Tail and outlier robustness audit.  The columns separate $A$-stability, curvature-coefficient correlation, $B$-sign agreement, gain-sign agreement, median absolute endpoint-gain shift at $d=0.10$ relative to the original preprocessing, and within-condition MAE.  Mild winsorization and controlled gold outliers are comparatively benign, rank-normalization can change endpoint gain scale, and proxy outliers can strongly perturb curvature.}
\label{tab:tail-outlier-robustness}
\begin{adjustbox}{max width=\textwidth}
\begin{tabular}{llrrrrrrr}
\toprule
Dataset & Condition & Paths & corr($A$) & corr($B$) & $B$ sign (\%) & Gain sign (\%) & median $|\Delta_{\rm shift}|$ & MAE \\
\midrule
SummEval & Winsor 1/99 & 32 & 0.999 & 0.997 & 93.8 & 100.0 & 0.003 & 0.005 \\
SummEval & Winsor 5/95 & 32 & 0.976 & 0.967 & 46.9 & 100.0 & 0.022 & 0.004 \\
SummEval & Rank-normal & 32 & 0.175 & 0.983 & 84.4 & 100.0 & 0.052 & 0.004 \\
SummEval & Gold outlier & 32 & 0.994 & 0.997 & 98.6 & 100.0 & 0.005 & 0.004 \\
SummEval & Proxy outlier & 32 & 0.997 & 0.940 & 100.0 & 100.0 & 0.092 & 0.017 \\
WMT-MQM-clusters & Winsor 1/99 & 141 & 0.999 & 0.957 & 98.6 & 100.0 & 0.012 & 0.007 \\
WMT-MQM-clusters & Winsor 5/95 & 141 & 0.989 & 0.859 & 97.2 & 100.0 & 0.075 & 0.008 \\
WMT-MQM-clusters & Rank-normal & 141 & 0.524 & 0.667 & 95.0 & 100.0 & 0.467 & 0.002 \\
WMT-MQM-clusters & Gold outlier & 141 & 0.996 & 0.979 & 97.5 & 100.0 & 0.022 & 0.011 \\
WMT-MQM-clusters & Proxy outlier & 141 & 0.997 & 0.773 & 95.0 & 99.9 & 0.352 & 0.072 \\
OpenMEVA-MAGS & Winsor 1/99 & 10 & 1.000 & 0.999 & 90.0 & 100.0 & 0.005 & 0.003 \\
OpenMEVA-MAGS & Winsor 5/95 & 10 & 0.995 & 0.953 & 100.0 & 100.0 & 0.024 & 0.003 \\
OpenMEVA-MAGS & Rank-normal & 10 & 0.768 & 0.884 & 90.0 & 100.0 & 0.055 & 0.002 \\
OpenMEVA-MAGS & Gold outlier & 10 & 0.959 & 0.955 & 93.0 & 100.0 & 0.007 & 0.003 \\
OpenMEVA-MAGS & Proxy outlier & 10 & 0.986 & -0.107 & 20.0 & 100.0 & 0.152 & 0.009 \\
\bottomrule
\end{tabular}
\end{adjustbox}
\end{table}

\begin{figure}[tbp]
\centering
\includegraphics[width=0.88\linewidth]{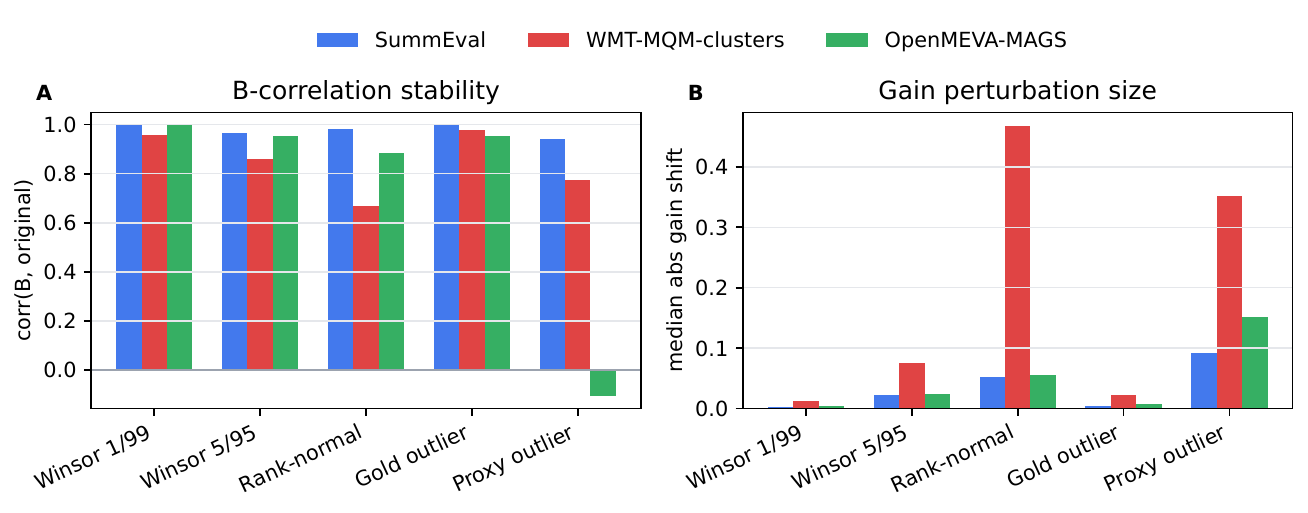}
\caption{Tail and outlier robustness audit.  Panel A reports stability of the curvature coefficient $B$, measured as correlation with the original-preprocessing coefficient.  Panel B reports the median absolute change in exact endpoint gain at $d=0.10$ relative to the original preprocessing; this is a gain-scale shift, not the prediction error of the local law.  Mild winsorization and controlled gold-side outliers are comparatively benign in this audit.  Rank-normalization can change endpoint gain scale, especially for WMT, and proxy-side outliers can substantially perturb curvature, especially in OpenMEVA.  Sign-agreement and $A$-stability metrics are reported in Table~\ref{tab:tail-outlier-robustness}.}
\label{fig:tail-outlier-robustness}
\end{figure}

\begin{figure}[tbp]
\centering
\includegraphics[width=0.95\linewidth]{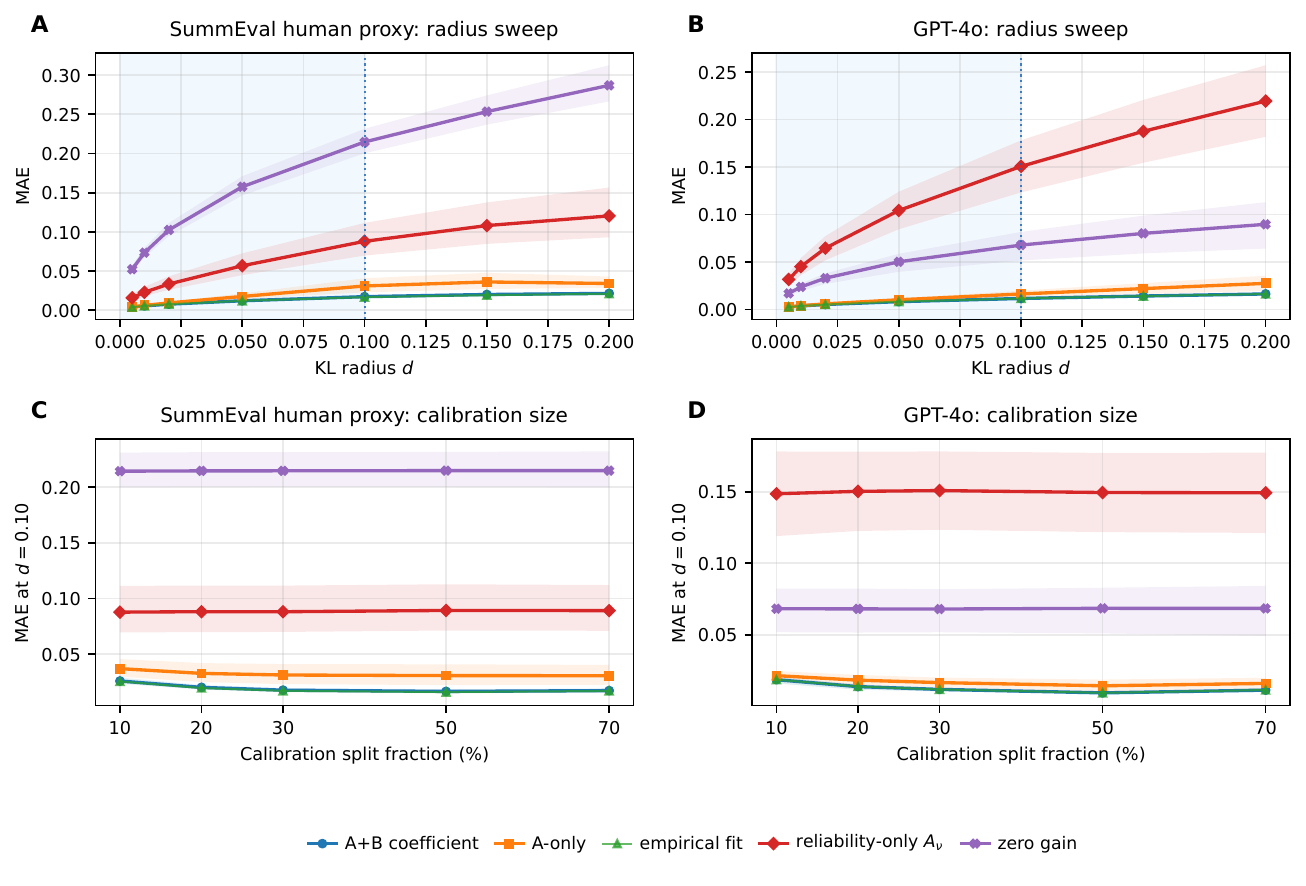}
\caption{Radius and calibration-size sensitivity: panels A/C are SummEval human-proxy paths; B/D are GPT-4o proxy paths.  Panels A and B report MAE as the KL radius $d$ increases; the shaded region and dashed line mark the shared $d\le0.10$ local window and main endpoint.  Panels C and D report MAE at $d=0.10$ as the calibration split fraction varies on the same proxy--gold paths; larger calibration fractions leave smaller evaluation splits.  Ribbons or intervals, where shown, are 95\% path-cluster bootstrap intervals for pooled-split MAE.  The shared legend compares coefficient $A{+}B$, $A$-only, the empirical fit, reliability-only $A_\nu$, and zero gain.}
\label{fig:accept-boost-radius-calibration}
\end{figure}

\begin{figure}[tbp]
\centering
\includegraphics[width=0.95\linewidth]{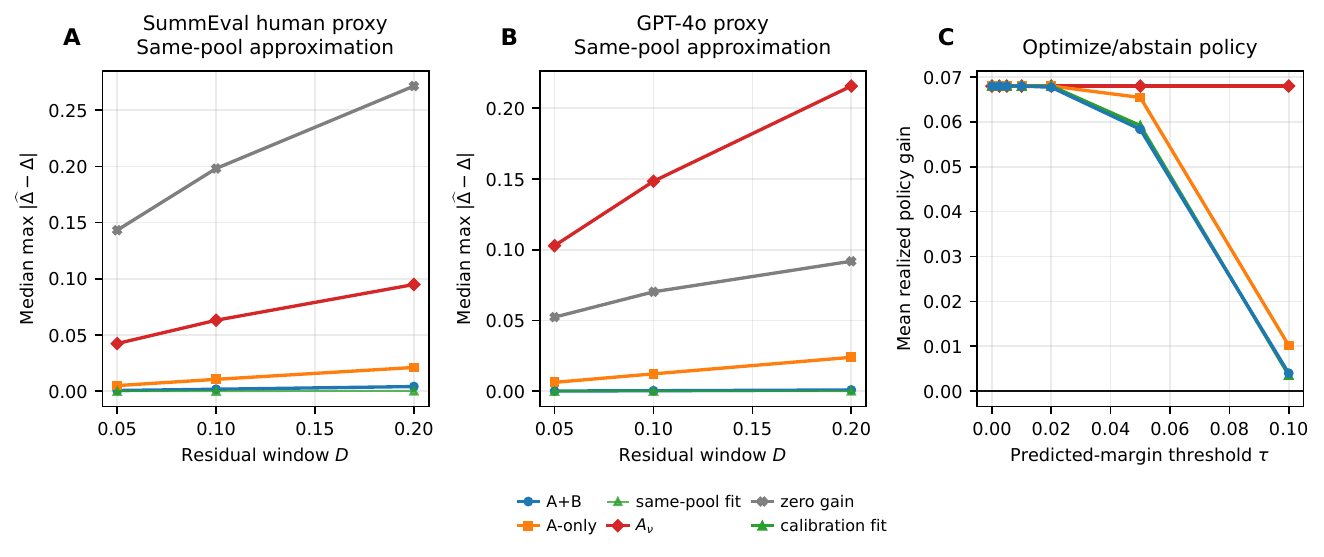}
\caption{Finite-range residual and selective-use diagnostics.  Panels A/B are same-pool approximation audits for SummEval/GPT-4o: moments and empirical fits use the same pool and exact curve whose median maximum residual is reported over $d\le D$. They do not measure held-out prediction error.  Panel C is the GPT-4o external-proxy selective-use audit: $\tau$ thresholds the predicted endpoint-gain margin before applying the optimize-if-positive rule, and lower-margin cases are abstained rather than treated as confident positive selections.}
\label{fig:accept-boost-residual-selective}
\end{figure}

\FloatBarrier
\subsection{OOD transfer details}
\label{app:ood}

The OOD experiments use calibration-only target standardization: in every split, gold and proxy scores are $z$-scored with the mean and standard deviation of that split's calibration items, and source coefficients are computed in the same way.  Each route uses a 30\% target calibration split and 50 random repeats.  We report SummEval $\rightarrow$ WMT MQM clusters, WMT MQM clusters $\rightarrow$ SummEval, and SummEval $\rightarrow$ OpenMEVA-MAGS ROC/WP.  For $\theta=(A,B)$, the source-prior predictor is
\[
\hat\theta_{\rm shrink}=\lambda\hat\theta_{\rm target}+(1-\lambda)\bar\theta_{\rm source},
\]
where $\bar\theta_{\rm source}$ is the mean source coefficient vector.  The main table uses $\lambda=0.75$ as a fixed target-heavy shrinkage value for the summary row, not a value tuned on evaluation targets.  Table~\ref{tab:ood-lambda-sweep-app} reports the full $\lambda\in\{0,.25,.5,.75,1\}$ sweep.  As a sensitivity analysis for SummEval $\rightarrow$ WMT clusters, we also $z$-score each path with the mean and standard deviation of all its items.  This full-path standardization uses the gold scores of evaluation items, so it is not prospective; it gives similar conclusions (values after Table~\ref{tab:ood-main}).

\begin{table}[H]
\centering
\small
\caption{Cross-slice transfer prediction from SummEval to eligible WMT MQM clusters under calibration-only target standardization.  The $A_\nu$ row is the target-side reliability-only baseline.  MAE intervals at $d=0.10$ are percentile intervals from 2,000 path-level bootstrap draws (seed 0), averaging the 50 splits within each WMT target path.}
\label{tab:ood-transfer-filled}
\begin{adjustbox}{max width=0.92\textwidth}
\begin{tabular}{lrrrr}
\toprule
Predictor & MAE @0.05 & sign @0.05 & MAE @0.10 [95\% CI] & sign @0.10 \\
\midrule
source-only & 0.0315 & 100.0\% & 0.0436 [0.0393, 0.0479] & 100.0\% \\
target calibration & 0.0191 & 100.0\% & 0.0259 [0.0248, 0.0269] & 100.0\% \\
source prior & \textbf{0.0188} & 100.0\% & \textbf{0.0256} [0.0246, 0.0267] & 100.0\% \\
$A_\nu$ & 0.1371 & 100.0\% & 0.2021 [0.1955, 0.2083] & 100.0\% \\
zero & 0.1792 & 0.0\% & 0.2452 [0.2389, 0.2517] & 0.0\% \\
\bottomrule

\end{tabular}
\end{adjustbox}
\end{table}

Table~\ref{tab:ood-main} summarizes all three routes at $d=0.10$.  SummEval~$\rightarrow$~WMT clusters tests whether a summarization prior helps an eligible translation-quality slice, WMT~$\rightarrow$~SummEval tests the reverse direction, and SummEval~$\rightarrow$~OpenMEVA-MAGS tests transfer into a different generation domain on a smaller slice.  Source-only transfer beats zero gain on all three routes and $\Anu$ on two of three, so it is informative but not invariant.  Target calibration improves on source-only transfer on every route.  The fixed target-heavy source prior ($\lambda=0.75$, not tuned on evaluation targets) is slightly below target-only calibration on both WMT routes but above it on SummEval~$\rightarrow$~OpenMEVA, so source coefficients are useful priors that target calibration should update before the forecast is used.

\begin{table}[H]
\centering
\small
\caption{Cross-slice transfer at $d=0.10$ under calibration-only target standardization.  WMT routes use the eligible MQM cluster slice, and OpenMEVA uses MAGS ROC/WP.  Source-prior uses a fixed $\lambda=0.75$ target-heavy shrinkage value.  \Appref{app:ood}{B} gives interval and sensitivity details where available.  Lower MAE is better.}
\label{tab:ood-main}
\begin{adjustbox}{max width=\textwidth}
\begin{tabular}{lrrrrr}
\toprule
Route & Source-only & Source-prior & Target cal. & $\Anu$ & Zero \\
\midrule
SummEval $\rightarrow$ WMT clusters & 0.0436 & \textbf{0.0256} & 0.0259 & 0.2021 & 0.2452 \\
WMT clusters $\rightarrow$ SummEval & 0.0306 & \textbf{0.0206} & 0.0230 & 0.1755 & 0.2718 \\
SummEval $\rightarrow$ OpenMEVA-MAGS & 0.1295 & 0.0325 & \textbf{0.0170} & 0.0429 & 0.4053 \\
\bottomrule
\end{tabular}
\end{adjustbox}
\end{table}

The most important OOD metric is magnitude MAE, not sign accuracy.  The calibration-only protocol confirms the main conclusion across directions: target calibration reduces error relative to source-only transfer on all three routes, but fixed source-prior shrinkage does not consistently improve over target-only calibration.  Under full-path standardization, the SummEval $\rightarrow$ WMT-cluster MAE values at $d=0.10$ are $0.0407$ for source-only, $0.0258$ for target calibration, $0.0235$ for source prior, $0.2053$ for target $A_\nu$, and $0.2432$ for zero gain.

\begin{table}[H]
\centering
\scriptsize
\caption{OOD transfer sensitivity to source-prior shrinkage at $d=0.10$.  $\lambda=0$ is source-only and $\lambda=1$ is target calibration only.}
\label{tab:ood-lambda-sweep-app}
\begin{adjustbox}{max width=\textwidth}
\begin{tabular}{lllllllll}
\toprule
Route & $\lambda$ & Splits & Source-only & Target cal. & Source-prior & $A_\nu$ & Zero & Prior sign \\
\midrule
SummEval $\rightarrow$ WMT clusters & 0 & 7050 & 0.0436 & \textbf{0.0259} & 0.0436 & 0.2021 & 0.2452 & 100.0\% \\
SummEval $\rightarrow$ WMT clusters & 0.25 & 7050 & 0.0436 & \textbf{0.0259} & 0.0358 & 0.2021 & 0.2452 & 100.0\% \\
SummEval $\rightarrow$ WMT clusters & 0.5 & 7050 & 0.0436 & \textbf{0.0259} & 0.0293 & 0.2021 & 0.2452 & 100.0\% \\
SummEval $\rightarrow$ WMT clusters & 0.75 & 7050 & 0.0436 & 0.0259 & \textbf{0.0256} & 0.2021 & 0.2452 & 100.0\% \\
SummEval $\rightarrow$ WMT clusters & 1 & 7050 & 0.0436 & \textbf{0.0259} & \textbf{0.0259} & 0.2021 & 0.2452 & 100.0\% \\
WMT clusters $\rightarrow$ SummEval & 0 & 1600 & 0.0306 & \textbf{0.0230} & 0.0306 & 0.1755 & 0.2718 & 100.0\% \\
WMT clusters $\rightarrow$ SummEval & 0.25 & 1600 & 0.0306 & \textbf{0.0230} & 0.0248 & 0.1755 & 0.2718 & 100.0\% \\
WMT clusters $\rightarrow$ SummEval & 0.5 & 1600 & 0.0306 & 0.0230 & \textbf{0.0214} & 0.1755 & 0.2718 & 100.0\% \\
WMT clusters $\rightarrow$ SummEval & 0.75 & 1600 & 0.0306 & 0.0230 & \textbf{0.0206} & 0.1755 & 0.2718 & 100.0\% \\
WMT clusters $\rightarrow$ SummEval & 1 & 1600 & 0.0306 & \textbf{0.0230} & \textbf{0.0230} & 0.1755 & 0.2718 & 100.0\% \\
SummEval $\rightarrow$ OpenMEVA-MAGS & 0 & 500 & 0.1295 & \textbf{0.0170} & 0.1295 & 0.0429 & 0.4053 & 100.0\% \\
SummEval $\rightarrow$ OpenMEVA-MAGS & 0.25 & 500 & 0.1295 & \textbf{0.0170} & 0.0966 & 0.0429 & 0.4053 & 100.0\% \\
SummEval $\rightarrow$ OpenMEVA-MAGS & 0.5 & 500 & 0.1295 & \textbf{0.0170} & 0.0637 & 0.0429 & 0.4053 & 100.0\% \\
SummEval $\rightarrow$ OpenMEVA-MAGS & 0.75 & 500 & 0.1295 & \textbf{0.0170} & 0.0325 & 0.0429 & 0.4053 & 100.0\% \\
SummEval $\rightarrow$ OpenMEVA-MAGS & 1 & 500 & 0.1295 & \textbf{0.0170} & \textbf{0.0170} & 0.0429 & 0.4053 & 100.0\% \\
\bottomrule
\end{tabular}

\end{adjustbox}
\end{table}

\FloatBarrier
\subsection{GPT-4o external-proxy diagnostics}
\label{app:gpt-failure-filled}

The GPT-4o proxy scores cover all 1,600 SummEval items.  The GPT-4o external-proxy experiment differs from the human-proxy setting primarily in proxy ranking and top-tail alignment.  Table~\ref{tab:margin-sweep-filled} reports the margin-conditioned fixed-gain sign behavior, and Table~\ref{tab:gpt-failure-filled} summarizes the broader rank and tail diagnostics.  The median proxy standard deviation is $0.678$, and the median model/human proxy-standard-deviation ratio across dimensions is $0.909$.  Median Spearman is $0.285$, median Pearson is $0.329$, and median top-10\% overlap with the human-gold top set is $0.150$.  Relative to the human-proxy view, the median drops are $0.288$ in Spearman and $0.256$ in top-10\% overlap.

\begin{table}[H]
\centering
\small
\caption{GPT-4o external-proxy ranking and tail diagnostics.  Overall rows are medians over SummEval dimensions unless otherwise stated.  The proxy and human top sets are the 10\% of items with the highest proxy and human-gold scores.  Gold delta is the mean gold of the proxy top set minus the mean gold of all items; oracle gap is the mean gold of the human top set minus that of the proxy top set; overlap is the share of the proxy top set that is also in the human top set.}
\label{tab:gpt-failure-filled}
\begin{adjustbox}{max width=0.82\textwidth}
\begin{tabular}{lr}
\toprule
Diagnostic & Overall \\
\midrule
median proxy std. & 0.678 \\
median Spearman(proxy, human gold) & 0.285 \\
median Pearson(proxy, human gold) & 0.329 \\
median top-10\% overlap with human top set & 0.150 \\
median top-10\% gold delta & 0.1666 \\
median top-10\% oracle gap & 0.682 \\
std ratio model/human, median over dimensions & 0.909 \\
$\Delta$ Spearman model--human, median & -0.288 \\
$\Delta$ top-10 overlap model--human, median & -0.256 \\
\bottomrule
\end{tabular}
\end{adjustbox}
\end{table}

\FloatBarrier
\subsection{External-judge reproducibility record}
\label{app:gpt-repro}

External judges can change across model snapshots and prompts, so the GPT-4o result is treated as an external-proxy evaluation of one fixed model-and-prompt configuration rather than a claim about all LLM judges.  The reproducibility package includes the aggregate proxy scores and the scripts used for the reported analysis.  Table~\ref{tab:gpt4o-repro-record} summarizes what can be recovered from the available records and which fields are not recoverable from the aggregate file alone.  In particular, the aggregate artifact is sufficient to reproduce the reported post-processing analyses, but exact API reruns are not fully reproducible without a raw per-call log containing the exact model snapshot, decoding parameters, and invalid-output/retry counts.

\begin{table}[H]
\centering
\scriptsize
\caption{Reproducibility record for the GPT-4o external-proxy configuration.  File paths are relative to the code release, and hashes are the first 12 hexadecimal digits of SHA-256.}
\label{tab:gpt4o-repro-record}
\begin{adjustbox}{max width=\textwidth}
\begin{tabular}{lll}
\toprule
Field & Value & Status \\
\midrule
Aggregated proxy artifact & data/model\_proxy\_4o\_aggregated.csv & 6,400 rows; 4 SummEval dimensions; $n_{\rm judges}\in\{3,6,9,12\}$ \\
Exact model snapshot & not stored in aggregated artifact & filename indicates GPT-4o; raw per-call log is needed for an exact model id \\
System prompt & experiments/09\_gpt\_proxy\_analysis/examples/summary\_system\_prompt.txt & SHA-256 prefix ec6ab5c98615 \\
User prompt template & experiments/09\_gpt\_proxy\_analysis/examples/summary\_user\_template.txt & SHA-256 prefix e121b070a2ea \\
Score parser & Pydantic schema with score, confidence, and rationale & experiments/09\_gpt\_proxy\_analysis/01\_generate\_proxy\_scores.py \\
Score range & 1.0--5.0 in generator defaults & score is clamped to range after parsing \\
Decoding parameters & not explicitly set in generator & temperature, top-$p$, and maximum tokens use API defaults unless the original command overrode the code \\
Invalid-output handling & structured parse with retries & up to 5 retries by default; invalid count not stored in aggregate \\
Normalization split & not stored in proxy artifact & downstream experiments standardize inside each analysis protocol \\
\bottomrule
\end{tabular}

\end{adjustbox}
\end{table}

\FloatBarrier
\subsection{Finite-range diagnostics and best-of-\texorpdfstring{$N$}{N} bridge}
\label{app:stability-bon-filled}

Finite-range diagnostics check whether the $d\le0.1$ window used for local fitting is numerically regular.  The SummEval human-proxy setting has median minimum tilted variance $0.631$ and median minimum effective sample size 1338.  GPT-4o proxy has median minimum variance $0.375$ and ESS 1327.  Thus the GPT-4o external-proxy behavior is not caused by tilt-path collapse on the shared local range.  Its tail quality is less favorable than the human-proxy setting: the mean gold of the items with the maximum proxy score, on which the tilt concentrates as $t\to\infty$, minus the pool mean has median $-0.054$ across paths and is negative on 50\% of paths.  Table~\ref{tab:finite-range-stability-filled} and Figure~\ref{fig:stability-multidataset} report the finite-range audit, while Tables~\ref{tab:bon-nladder}--\ref{tab:prompt-conditional-bon-scope} and Figure~\ref{fig:bon-bootstrap-app-filled} report the Best-of-$N$ bridge.

\begin{table}[H]
\centering
\small
\caption{Finite-range diagnostics on $d\le0.1$. Cumulant columns are medians across paths of each path's maximum absolute cumulant.  Tail-baseline is the median across paths of the mean gold of the items at the maximum proxy score minus the pool mean gold; the last column is the share of paths on which this difference is negative.}
\label{tab:finite-range-stability-filled}
\begin{adjustbox}{max width=\textwidth}
\begin{tabular}{lrrrrrrr}
\toprule
Setting & Paths & median min Var$_\pi(r)$ & median min ESS & median max $|\kappa_4(r)|$ & median max $|\kappa_4(J,r)|$ & tail-baseline & tail below baseline \\
\midrule
Human proxy & 32 & 0.631 & 1338 & 0.258 & 0.150 & 0.571 & 0.0\% \\
GPT-4o proxy & 4 & 0.375 & 1327 & 0.035 & 0.040 & -0.054 & 50.0\% \\
\bottomrule
\end{tabular}
\end{adjustbox}
\end{table}

Best-of-$N$ is a discrete selection rule, not the smooth exponential-tilt path analyzed by the theorem.  Table~\ref{tab:bon-nladder} computes the exact pooled BoN distribution, with replacement and uniform tie-breaking, on the full item set of each of 191 paths and for $N\in\{2,3,4,6,8,12,16,24,32,48,64\}$.  At $N=2$, the median induced radius is $0.166$--$0.193$, outside the checked window in every setting, and the near-untied ROUGE-1 proxy matches the atomless closed forms $d_N=\log N-(N-1)/N$ and $\mathrm{RelESS}_N=(2N-1)/N^2$.  Only 4 of the 2,101 nontrivial points (0.19\%) lie inside $d\le0.10$.  All four are SummEval held-out expert paths at $N=2$ (three consistency raters and one fluency rater; Table~\ref{tab:bon-inwindow}), on which 85.4--89.1\% of the 1,600 items carry the top proxy score.  BoN and the pooled tilt weight items only through their proxy score, so both concentrate on that tied block, and the reachable KL is capped at $\log(m/n_{\rm top})=0.116$--$0.157$.  At these four points, the coefficients and the descriptive fit are computed on the same pool on which BoN is evaluated, so the comparison is a same-pool audit rather than a held-out prediction.  All listed nonzero predictors overestimate the four BoN gains; zero underestimates them; the mean absolute error is $0.0197$ for $A{+}B$, $0.0884$ for $A$-only, $0.1238$ for $\Anu$, $0.2490$ for zero gain, and $0.0043$ for the descriptive two-term fit.  These are scope exceptions from one dataset, not evidence that BoN generally lands in the window.

Table~\ref{tab:bon-softened} sweeps the window continuously with two softened rules on 30/70 splits (coefficients from the calibration split, errors on the evaluation split): the mixture $p_\alpha=(1-\alpha)p_0+\alpha p_{{\rm BoN},8}$ and the rank tilt $q_{\lambda,i}\propto p_{0,i}U_i^\lambda$, where $U_i$ is the one-based average rank of $r_i$ divided by the pool size (ties receive their average rank).  Inside $0.005<d\le0.10$, $A{+}B$ is below zero gain and $\Anu$ in every cell and within $0.001$ of the calibration two-term fit.  It improves on $A$-only for SummEval and WMT, favors it for the GPT-4o mixture, nearly ties it for the GPT-4o rank tilt, and is worse than $A$-only on OpenMEVA.  This supports an informative local approximation for smooth rules inside the window, not universal dominance of the curvature term.  Only two to four grid values per split fall inside the window, mostly in its lower half.

The prompt-conditional audit in Table~\ref{tab:prompt-conditional-bon-scope} and Figure~\ref{fig:prompt-conditional-bon-scope} applies BoN within each prompt instead of over the pooled item set.  Even at $N=2$, the mean prompt KL is outside the local window (0.144 on SummEval and 0.131 on OpenMEVA), and only 12.5\% and 11.7\% of prompts are inside.  At $N=4$, only the six saturated SummEval expert paths, on which 78--89\% of items share the top score, remain reachable by a pooled raw-score tilt within $d\le0.25$; for $N\ge8$, no path is.  For $N\ge4$, the prompts with prompt KL at most 0.10 (5.8\% on SummEval and 3.0\% on OpenMEVA) are prompts whose candidates are all or almost all tied, not instances of mild selection.  This mismatch is why BoN is reported as a scope audit rather than as another validation of the smooth local theorem.

\begin{table}[H]
\centering
\scriptsize
\caption{Exact pooled Best-of-$N$ (with replacement, uniform tie-breaking) over each path's full item set; $d_{\rm BoN}=\KL(p\Vert p_0)$ in nats and $\mathrm{RelESS}=\mathrm{ESS}/m$.  Nontrivial means $N\ge2$ (11 values per path); counts are (path, $N$) points with their share.  Held-out-human rows use a held-out rater as the proxy; the ROUGE-1 and GPT-4o rows use that model.  Item pools have $m=1{,}600$ (SummEval), $430$--$2{,}820$ (WMT), and $1{,}000$ (OpenMEVA) items.  MAE uses only the four in-window points (same-pool coefficients, $n=4$); -- means no in-window point.  The $d\le0.25$ counts include saturated paths whose reachable KL $\log(m/n_{\rm top})$ is below 0.25.  The lower panel gives medians by $N$.}
\label{tab:bon-nladder}
\begin{adjustbox}{max width=\textwidth}
\begin{tabular}{@{}c@{}}
\begin{tabular}{lrrrrrrr}
\toprule
 &  & \multicolumn{2}{c}{Median at $N=2$} & \multicolumn{2}{c}{Nontrivial points ($N\ge2$)} & \multicolumn{2}{c}{In-window MAE} \\
\cmidrule(lr){3-4}\cmidrule(lr){5-6}\cmidrule(lr){7-8}Setting & Paths & $d_{\mathrm{BoN}}$ & RelESS & $d\le0.10$ & $d\le0.25$ & A+B & A-only \\
\midrule
SummEval human & 32 & 0.172 & 0.772 & 4 (1.1\%) & 92 (26.1\%) & \textbf{0.0197} & 0.0884 \\
SummEval ROUGE-1 & 4 & 0.193 & 0.750 & 0 (0.0\%) & 4 (9.1\%) & -- & -- \\
SummEval GPT-4o & 4 & 0.189 & 0.754 & 0 (0.0\%) & 4 (9.1\%) & -- & -- \\
WMT MQM & 141 & 0.186 & 0.760 & 0 (0.0\%) & 154 (9.9\%) & -- & -- \\
OpenMEVA-MAGS & 10 & 0.166 & 0.771 & 0 (0.0\%) & 10 (9.1\%) & -- & -- \\
\midrule
All settings & 191 & 0.185 & 0.761 & 4 (0.2\%) & 264 (12.6\%) & \textbf{0.0197} & 0.0884 \\
\bottomrule
\end{tabular}
\\[6pt]
\begin{tabular}{lrrrrrr}
\toprule
Setting & $N=2$ & $N=3$ & $N=4$ & $N=8$ & $N=16$ & $N=64$ \\
\midrule
\multicolumn{7}{l}{\textit{Median $d_{\mathrm{BoN}}$ (nats)}} \\
SummEval human & 0.172 & 0.377 & 0.549 & 0.988 & 1.233 & 1.266 \\
SummEval ROUGE-1 & 0.193 & 0.432 & 0.636 & 1.204 & 1.835 & 3.174 \\
SummEval GPT-4o & 0.189 & 0.421 & 0.619 & 1.171 & 1.768 & 3.000 \\
WMT MQM & 0.186 & 0.406 & 0.582 & 0.959 & 1.117 & 1.131 \\
OpenMEVA-MAGS & 0.166 & 0.379 & 0.560 & 1.059 & 1.592 & 2.563 \\
\midrule
\multicolumn{7}{l}{\textit{Median relative ESS}} \\
SummEval human & 0.772 & 0.592 & 0.480 & 0.325 & 0.285 & 0.282 \\
SummEval ROUGE-1 & 0.750 & 0.556 & 0.438 & 0.234 & 0.121 & 0.031 \\
SummEval GPT-4o & 0.754 & 0.562 & 0.444 & 0.242 & 0.128 & 0.034 \\
WMT MQM & 0.760 & 0.583 & 0.484 & 0.352 & 0.324 & 0.323 \\
OpenMEVA-MAGS & 0.771 & 0.584 & 0.468 & 0.265 & 0.146 & 0.051 \\
\bottomrule
\end{tabular}

\end{tabular}
\end{adjustbox}
\end{table}

\begin{table}[H]
\centering
\scriptsize
\caption{The four nontrivial BoN points inside $d\le0.10$, all SummEval held-out expert paths at $N=2$ ($m=1{,}600$; gold is the mean of the other two experts).  Coefficients and the empirical fit use the same full pool on which BoN is evaluated (descriptive, not held-out).  Gains, predictions, and errors are in SummEval 1--5 units.}
\label{tab:bon-inwindow}
\begin{adjustbox}{max width=\textwidth}
\begin{tabular}{lrrrrr}
\toprule
 & \multicolumn{3}{c}{Consistency, expert} & Fluency, expert &  \\
\cmidrule(lr){2-4}\cmidrule(lr){5-5}Quantity & rater 0 & rater 1 & rater 2 & rater 2 & Mean \\
\midrule
$d_{\mathrm{BoN}}$ (nats) & 0.096 & 0.093 & 0.078 & 0.088 & 0.089 \\
ESS & 1{,}422 & 1{,}427 & 1{,}457 & 1{,}438 & 1{,}436 \\
RelESS & 0.889 & 0.892 & 0.911 & 0.899 & 0.898 \\
Items at top proxy score (\%) & 85.4 & 86.0 & 89.1 & 87.1 & 86.9 \\
KL cap $\log(m/n_{\mathrm{top}})$ & 0.157 & 0.151 & 0.116 & 0.138 & 0.140 \\
\midrule
Actual BoN gain & 0.2806 & 0.2510 & 0.2761 & 0.1881 & 0.2490 \\
A+B prediction & 0.3020 & 0.2775 & 0.2941 & 0.2010 & 0.2686 \\
\midrule
$|$error$|$, A+B & 0.0213 & 0.0265 & 0.0180 & 0.0129 & 0.0197 \\
$|$error$|$, A-only & 0.0999 & 0.0987 & 0.0837 & 0.0711 & 0.0884 \\
$|$error$|$, $A_\nu$ & 0.1617 & 0.0490 & 0.1422 & 0.1424 & 0.1238 \\
$|$error$|$, empirical fit & \textbf{0.0050} & \textbf{0.0087} & \textbf{0.0015} & \textbf{0.0020} & \textbf{0.0043} \\
$|$error$|$, zero & 0.2806 & 0.2510 & 0.2761 & 0.1881 & 0.2490 \\
\bottomrule
\end{tabular}

\end{adjustbox}
\end{table}

\begin{table}[H]
\centering
\scriptsize
\caption{Softened BoN-direction rules inside $0.005<d\le0.10$: the mixture with BoN at $N=8$ and the rank tilt.  Pooled item-set rules (not prompt-conditional); coefficients from 30\% calibration splits, errors on the 70\% evaluation items, 30 splits per path.  Path-macro MAE weights each path equally; ``Paths: A+B wins'' counts paths on which $A{+}B$ has lower error than $A$-only.  Only two to four grid values per split fall in the window.  Native units, not comparable across datasets.}
\label{tab:bon-softened}
\begin{adjustbox}{max width=\textwidth}
\begin{tabular}{llrrrrrr}
\toprule
 &  & \multicolumn{5}{c}{Path-macro MAE, $0.005<d\le0.10$} & Paths: \\
\cmidrule(lr){3-7}Dataset & Rule & A+B & A-only & $A_\nu$ & Emp.\ fit & Zero & A+B wins \\
\midrule
SummEval human & Mixture & 0.0189 & 0.0244 & 0.0556 & \textbf{0.0186} & 0.0939 & 32/32 \\
 & Rank tilt & 0.0103 & 0.0154 & 0.0458 & \textbf{0.0101} & 0.1163 & 31/32 \\
\midrule
SummEval GPT-4o & Mixture & \textbf{0.0168} & 0.0209 & 0.0976 & 0.0169 & 0.0241 & 4/4 \\
 & Rank tilt & 0.0064 & 0.0065 & 0.0784 & \textbf{0.0063} & 0.0414 & 3/4 \\
\midrule
WMT MQM & Mixture & 0.1207 & 0.1432 & 0.3917 & \textbf{0.1198} & 0.3140 & 135/141 \\
 & Rank tilt & 0.0492 & 0.0607 & 0.3197 & \textbf{0.0488} & 0.4794 & 132/141 \\
\midrule
OpenMEVA-MAGS & Mixture & 0.0174 & \textbf{0.0156} & 0.0235 & 0.0169 & 0.1552 & 2/10 \\
 & Rank tilt & 0.0156 & \textbf{0.0139} & 0.0211 & 0.0153 & 0.1421 & 3/10 \\
\bottomrule
\end{tabular}

\end{adjustbox}
\end{table}

\begin{table}[H]
\centering
\scriptsize
\caption{Prompt-conditional Best-of-$N$ scope audit.  KL is computed within each prompt. The reported mean averages path-level prompt means; the median is the median of path-level prompt medians.  ``Pooled tilt reachable'' is the share of paths whose mean prompt KL is below both 0.25 and the pooled tilt's supremum $\log(m/n_{\rm top})$, so that one pooled raw-score tilt attains it.  Even $N=2$ is outside the local window on average and for most prompts.  At $N=4$, only the six saturated SummEval expert paths remain reachable; for $N\ge8$, no path is.  For $N\ge4$, the prompts with KL at most 0.10 are prompts whose candidates are (almost) all tied.}
\label{tab:prompt-conditional-bon-scope}
\begin{adjustbox}{max width=\textwidth}
\begin{tabular}{lrrrrrr}
\toprule
 &  & mean & median & prompts & pooled tilt & mean \\
 &  & prompt & prompt & KL$\le.10$ & reachable, & BoN \\
Dataset & $N$ & KL & KL & (\%) & $d\le0.25$ (\%) & gain \\
\midrule
SummEval & 2 & 0.144 & 0.161 & 12.5 & 100.0 & 0.226 \\
SummEval & 4 & 0.436 & 0.518 & 5.8 & 18.8 & 0.361 \\
SummEval & 8 & 0.741 & 0.922 & 5.8 & 0.0 & 0.437 \\
SummEval & 16 & 0.989 & 1.244 & 5.8 & 0.0 & 0.476 \\
SummEval & 32 & 1.143 & 1.385 & 5.8 & 0.0 & 0.493 \\
SummEval & 64 & 1.207 & 1.386 & 5.8 & 0.0 & 0.497 \\
\midrule
OpenMEVA-MAGS & 2 & 0.131 & 0.142 & 11.7 & 100.0 & 0.285 \\
OpenMEVA-MAGS & 4 & 0.451 & 0.499 & 3.0 & 0.0 & 0.531 \\
OpenMEVA-MAGS & 8 & 0.803 & 0.835 & 3.0 & 0.0 & 0.698 \\
OpenMEVA-MAGS & 16 & 1.043 & 0.914 & 3.0 & 0.0 & 0.772 \\
OpenMEVA-MAGS & 32 & 1.113 & 0.916 & 3.0 & 0.0 & 0.785 \\
OpenMEVA-MAGS & 64 & 1.117 & 0.916 & 3.0 & 0.0 & 0.786 \\
\bottomrule
\end{tabular}

\end{adjustbox}
\end{table}

\begin{figure}[H]
\centering
\includegraphics[width=0.86\linewidth]{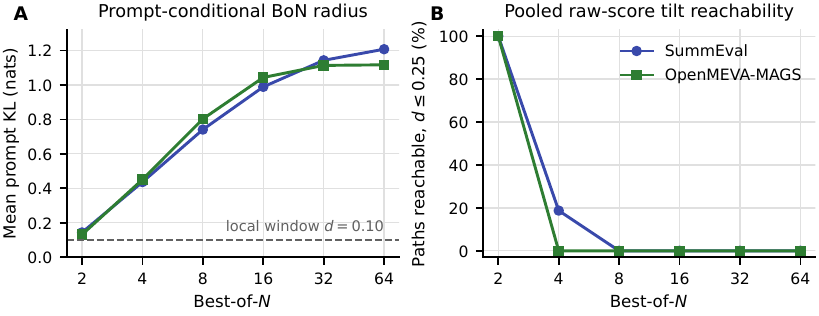}
\caption{Prompt-conditional Best-of-$N$ scope audit.  Panel A reports the mean prompt-level KL of exact with-replacement prompt-conditional BoN, with the $d=0.10$ window shown by the dashed line.  Panel B reports the fraction of paths whose mean prompt KL is attainable by a pooled raw-score tilt within $d\le0.25$, accounting for paths whose tied top scores cap the pooled tilt.  Mean BoN gains are in Table~\ref{tab:prompt-conditional-bon-scope}.}
\label{fig:prompt-conditional-bon-scope}
\end{figure}

\begin{figure}[H]
\centering
\includegraphics[width=\linewidth]{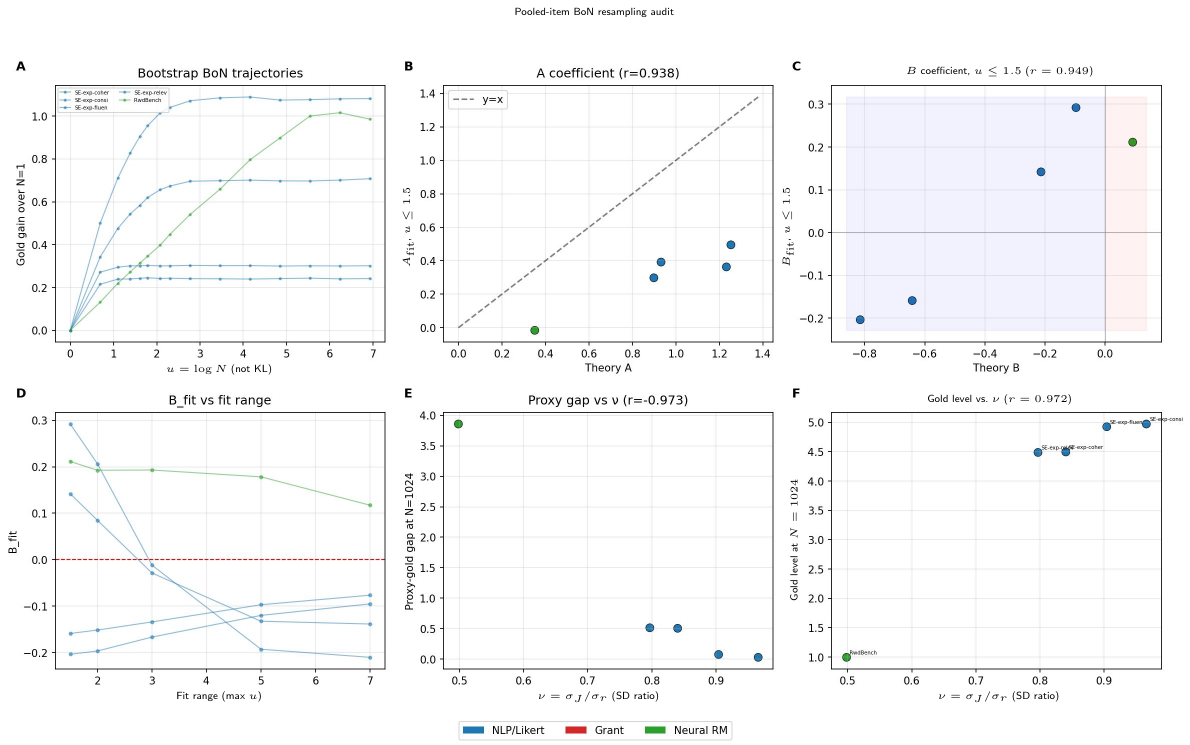}
\caption{Pooled-item Best-of-$N$ resampling analysis, not prompt-conditional selection or a theorem test. Candidates are drawn with replacement from the pooled item set, and each candidate's proxy is an independently drawn rating or RewardBench model column; gold is the itemwise mean over all columns, not LOO. \reviewedit{Panels A--D use the bridge coordinate $u=\log N$, not the induced KL.}{Panel A uses the bridge coordinate $u=\log N$ (labeled $d$ in the plot), not the induced KL.} Panels B--C compare fitted and theory coefficients, and Panel D shows fitted $B$ as the fit range grows. Panels E/F show the proxy--gold gap/gold level, not gold gain. The displayed $\nu$ is a square-root variance ratio, not ICC. }
\label{fig:bon-bootstrap-app-filled}
\end{figure}

\begin{figure}[H]
\centering
\includegraphics[width=0.88\linewidth]{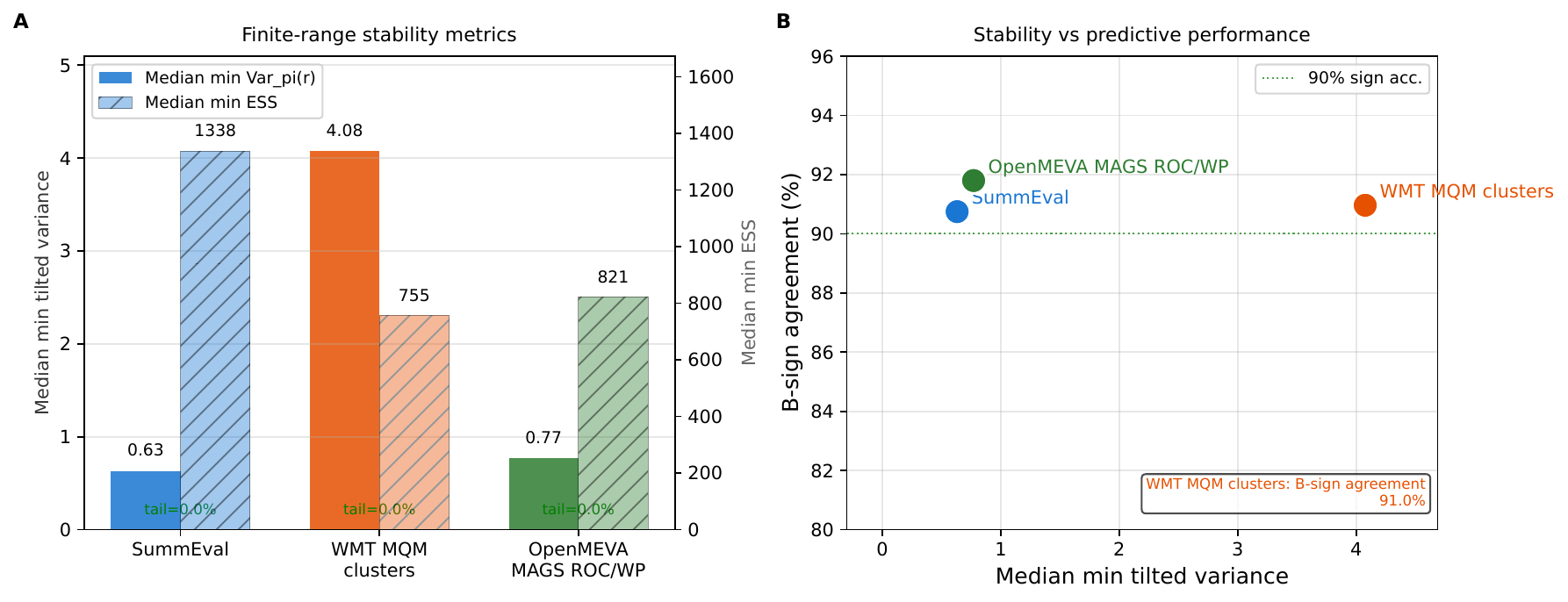}
\caption{Finite-range diagnostics across eligible analysis slices.  Panel A reports median minimum tilted variance on the left scale and median minimum effective sample size on the accompanying labels/right scale for each slice on the shared local window $d\le0.1$; tail labels give the share of paths whose items at the maximum proxy score have mean gold below the pool mean.  Panel B relates median minimum tilted variance to curvature-sign agreement with the evaluation curve-fit coefficient.  This is a numerical stability diagnostic, not a guarantee of accurate prediction.}
\label{fig:stability-multidataset}
\end{figure}

\FloatBarrier
\subsection{LoRA-DPO bridge details}
\label{app:dpo}

\paragraph{Setup.}  The reference policy is Qwen2.5-3B-Instruct.  From 1,000 UltraFeedback prompts \citep{cui2024ultrafeedback} (pinned revision), the reference policy generates eight candidates per prompt; the highest- and lowest-scoring candidates under Skywork-Reward-Llama-3.1-8B-v0.2, the optimized proxy $r$, form 968 preference pairs with a nonzero score gap.  OpenAssistant reward-model-deberta-v3-large-v2 \citep{openassistant2023reward} supplies the frozen silver target $J$ and is never used for training.  DPO uses $\beta=0.1$, learning rate $10^{-5}$, effective batch size 16, and LoRA rank, alpha, and dropout of 16, 32, and 0.05; the three seeds share data and pools and differ only in training randomness, with separate training runs of 64, 90, 120, 150, 180, and 200 steps, each giving one checkpoint.  The training, calibration (100 prompts), and evaluation (200 prompts) prompt sets are disjoint.  Calibration and evaluation pools are fixed, hash-verified sets of 32 candidates per prompt from the reference policy (temperature 0.9, top-$p$ 0.95, at most 320 new tokens).  One calibration prompt has zero variance in both proxy and silver scores and is excluded.  For each of the 99 remaining prompts, $(\hat A_x,\hat B_x)$ is computed on its 32 candidates and averaged uniformly, giving $\hat A=0.2268$ and $\hat B=-0.0465$.  The surrogate $\Anu$ is $\sqrt2\,\widehat{\Var}(J)/\widehat{\mathrm{SD}}(r)$ per prompt, averaged ($0.2599$); no rater reliability exists for a deterministic silver target.

\paragraph{Estimand.}  For evaluation prompt $x$, $w_{x,i}\propto\exp\{\log p_{\rm DPO}(y_i\mid x)-\log p_{\rm ref}(y_i\mid x)\}$ is normalized over its 32 candidates, and $d_x=\KL(w_x\Vert\mathrm{Unif}_{32})$.  This KL is defined directly between two discrete distributions and needs no tilt assumption; the tilt enters only through the idealized comparator, which reweights the same pool by proxy score to match $d_x$.  The realized gain is the $w_x$-weighted silver mean minus the uniform-pool mean.  Normalized ratios are not the DPO policy restricted to the pool, and because the candidates came from a temperature and top-$p$ proposal while the ratios use untempered log-probabilities, they are not proposal-corrected importance weights.  The realized gain is therefore a self-normalized diagnostic, not a full-policy or fresh-generation estimate.  Predictions average $\hat A\sqrt{d_x}+\hat Bd_x$ over evaluation prompts, and the error of a checkpoint is the absolute difference between the prompt-averaged prediction and the prompt-averaged target.

\paragraph{Checkpoint gate and radius tail.}  We keep checkpoints whose mean prompt KL is at most 0.10: steps 64, 90, 120, and 150 for seed 0 and steps 64, 90, and 120 for seeds 1 and 2.  A median-KL gate keeps the same ten checkpoints.  The gate uses no silver label, gain, or error, but it is transductive in the reported radius.  Not every prompt is local: across the 2,000 retained prompt--checkpoint cells, the median, 90th, 95th, and 99th percentiles and the maximum of $d_x$ are 0.0257, 0.0919, 0.1224, 0.1867, and 0.4668, and 8.7\% exceed 0.10.  Errors are averaged over the retained checkpoints within each seed, and Table~\ref{tab:dpo} reports the mean and sample SD across the three seed averages.  $A$-only minus $A{+}B$ realized MAE is 0.00187, 0.00177, and 0.00189 for seeds 0, 1, and 2.  At all ten checkpoints, the ordering is realized gain $<$ $A{+}B$ prediction $<$ $A$-only prediction $<$ idealized gain.  Both forecasts therefore underpredict the frontier, and the negative $B$ moves $A{+}B$ toward the realized gain.  DPO changes a conditional distribution per prompt, whereas the bridge averages prompt-specific coefficients into one pair, so it is an empirical conditional-optimization bridge rather than an in-window theorem validation.

\paragraph{Comparators.}  Table~\ref{tab:dpo-comparators} reports 46 predictor variants against both targets, grouped by information setting.  Fourteen have lower realized-gain MAE than $A{+}B$, including a Spearman $A$-only variant ($0.01378$) and pre-pool flexible calibrators on a within-prompt rank coordinate ($0.01096$ for binned).  For the same fitted $\hat h$, the aggregate projection--direct gap is at most $0.00037$.  Realized-gain MAE mixes predictor error with the optimizer's shortfall, whose mean is $0.02116$, so a lower realized MAE can arise from cancellation; frontier MAE and realized MAE answer different questions, and no variant is universally best.  The four cross-validated pre-pool rows (binned and spline on raw or prompt-rank proxies) depend on the tie order of grouped cross-validation among 100 equal-size prompt groups, and a rerun on different hardware selects different hyperparameters; we report the original outputs.  The other cross-validated binned and spline rows use the same scheme and could shift likewise.

\begin{table}[H]
\centering
\scriptsize
\caption{All DPO comparators on the ten mean-KL-gated checkpoints of Table~\ref{tab:dpo} (mean $\pm$ sample SD over three seed means; lower is better).  The realized gain is a self-normalized likelihood-ratio reweighting of fixed pools, scored by a frozen silver reward model; it is not a fresh-generation, human-rated, or exact full-policy gain.  (a) Pre-pool: calibration prompts and measured $d_x$ only.  (b) Target-pool-aware direct integration: $\hat h$ fitted on calibration prompts and integrated over each evaluation prompt's proxy scores.  (c) Projection of the same $\hat h$; the last column is its projection-minus-direct prediction gap.  (d) Cross-fitted diagnostics that use evaluation silver labels; they are not deployable.  Raw, prompt-centered, prompt-$z$, and prompt-rank denote the proxy coordinate: the raw score, the score minus its prompt mean, the score standardized within its prompt, and its within-prompt rank.  Frozen-curve rows integrate the calibrator over calibration candidates and evaluate the resulting prompt-averaged gain curve at $d_x$.  The Spearman $A$-only variant replaces the Pearson correlation in $\hat A=\sqrt2\,\hat\sigma_J\,\widehat{\rm corr}(J,r)$ by Spearman's, averaged over calibration prompts.  MLP rows are exploratory.}
\label{tab:dpo-comparators}
\begin{adjustbox}{max width=\textwidth}
\begin{tabular}{lccc}
\toprule
 & MAE vs idealized & MAE vs realized & Projection$-$direct \\
Predictor & tilt frontier & likelihood-ratio gain & gap \\
\midrule
\multicolumn{4}{l}{\textit{(a) Pre-pool: calibration prompts and measured $d_x$ only}} \\
$A{+}B$ & 0.00395 $\pm$ 0.00047 & 0.01721 $\pm$ 0.00116 & -- \\
$A$-only & 0.00210 $\pm$ 0.00053 & 0.01905 $\pm$ 0.00120 & -- \\
Spearman $A$-only variant & 0.00738 $\pm$ 0.00045 & 0.01378 $\pm$ 0.00113 & -- \\
$A_\nu$ & 0.00371 $\pm$ 0.00061 & 0.02487 $\pm$ 0.00129 & -- \\
Zero & 0.04197 $\pm$ 0.00017 & 0.02081 $\pm$ 0.00062 & -- \\
Calibration empirical curve & 0.00430 $\pm$ 0.00047 & 0.01686 $\pm$ 0.00116 & -- \\
Raw isotonic (frozen curve) & 0.01878 $\pm$ 0.00080 & 0.03994 $\pm$ 0.00150 & -- \\
Raw binned (frozen curve) & 0.00193 $\pm$ 0.00058 & 0.02309 $\pm$ 0.00126 & -- \\
Raw spline (frozen curve) & \textbf{0.00131} $\pm$ 0.00055 & 0.02247 $\pm$ 0.00124 & -- \\
Raw MLP (frozen curve) & 0.01497 $\pm$ 0.00075 & 0.03612 $\pm$ 0.00144 & -- \\
Prompt-rank isotonic (frozen curve) & 0.00433 $\pm$ 0.00047 & 0.01683 $\pm$ 0.00116 & -- \\
Prompt-rank binned (frozen curve) & 0.01020 $\pm$ 0.00039 & \textbf{0.01096} $\pm$ 0.00107 & -- \\
Prompt-rank spline (frozen curve) & 0.00804 $\pm$ 0.00042 & 0.01312 $\pm$ 0.00111 & -- \\
Prompt-rank MLP (frozen curve) & 0.00577 $\pm$ 0.00045 & 0.01539 $\pm$ 0.00114 & -- \\
\midrule
\multicolumn{4}{l}{\textit{(b) Target-pool-aware: direct integration of $\hat h(r)$}} \\
Raw isotonic & 0.01786 $\pm$ 0.00067 & 0.03902 $\pm$ 0.00138 & -- \\
Raw binned & 0.00237 $\pm$ 0.00055 & 0.02353 $\pm$ 0.00127 & -- \\
Raw spline & 0.00239 $\pm$ 0.00045 & 0.02355 $\pm$ 0.00116 & -- \\
Prompt-centered isotonic & \textbf{0.00059} $\pm$ 0.00039 & 0.02170 $\pm$ 0.00112 & -- \\
Prompt-centered binned & 0.00573 $\pm$ 0.00040 & 0.01543 $\pm$ 0.00108 & -- \\
Prompt-centered spline & 0.00274 $\pm$ 0.00043 & 0.01842 $\pm$ 0.00111 & -- \\
Prompt-$z$ isotonic & 0.00088 $\pm$ 0.00050 & 0.02031 $\pm$ 0.00116 & -- \\
Prompt-$z$ binned & 0.01044 $\pm$ 0.00041 & \textbf{0.01072} $\pm$ 0.00110 & -- \\
Prompt-$z$ spline & 0.00377 $\pm$ 0.00048 & 0.01739 $\pm$ 0.00116 & -- \\
Prompt-rank isotonic & 0.00385 $\pm$ 0.00047 & 0.01731 $\pm$ 0.00116 & -- \\
Prompt-rank binned & 0.00903 $\pm$ 0.00041 & 0.01213 $\pm$ 0.00110 & -- \\
Prompt-rank spline & 0.00624 $\pm$ 0.00044 & 0.01492 $\pm$ 0.00113 & -- \\
\midrule
\multicolumn{4}{l}{\textit{(c) Target-pool-aware: two-term projection of the same $\hat h(r)$}} \\
Raw isotonic & 0.01813 $\pm$ 0.00067 & 0.03929 $\pm$ 0.00138 & 0.00027 $\pm$ 0.00002 \\
Raw binned & 0.00262 $\pm$ 0.00056 & 0.02378 $\pm$ 0.00128 & 0.00025 $\pm$ 0.00001 \\
Raw spline & 0.00257 $\pm$ 0.00044 & 0.02372 $\pm$ 0.00115 & 0.00017 $\pm$ 0.00001 \\
Prompt-centered isotonic & \textbf{0.00063} $\pm$ 0.00045 & 0.02179 $\pm$ 0.00112 & 0.00008 $\pm$ 0.00001 \\
Prompt-centered binned & 0.00559 $\pm$ 0.00041 & 0.01557 $\pm$ 0.00108 & 0.00014 $\pm$ 0.00001 \\
Prompt-centered spline & 0.00264 $\pm$ 0.00044 & 0.01852 $\pm$ 0.00111 & 0.00010 $\pm$ 0.00001 \\
Prompt-$z$ isotonic & 0.00085 $\pm$ 0.00050 & 0.02033 $\pm$ 0.00116 & 0.00002 $\pm$ 0.00000 \\
Prompt-$z$ binned & 0.01007 $\pm$ 0.00043 & \textbf{0.01109} $\pm$ 0.00111 & 0.00037 $\pm$ 0.00003 \\
Prompt-$z$ spline & 0.00370 $\pm$ 0.00048 & 0.01746 $\pm$ 0.00116 & 0.00007 $\pm$ 0.00001 \\
Prompt-rank isotonic & 0.00381 $\pm$ 0.00047 & 0.01735 $\pm$ 0.00116 & 0.00004 $\pm$ 0.00000 \\
Prompt-rank binned & 0.00878 $\pm$ 0.00042 & 0.01238 $\pm$ 0.00111 & 0.00024 $\pm$ 0.00002 \\
Prompt-rank spline & 0.00614 $\pm$ 0.00044 & 0.01502 $\pm$ 0.00113 & 0.00010 $\pm$ 0.00001 \\
\midrule
\multicolumn{4}{l}{\textit{(d) Evaluation-prompt cross-fit diagnostic (uses evaluation silver labels)}} \\
Raw isotonic (cross-fit) & 0.00809 $\pm$ 0.00034 & 0.02924 $\pm$ 0.00102 & -- \\
Raw spline (cross-fit) & 0.00185 $\pm$ 0.00030 & 0.02290 $\pm$ 0.00109 & -- \\
Prompt-centered isotonic (cross-fit) & 0.00514 $\pm$ 0.00050 & 0.02630 $\pm$ 0.00114 & -- \\
Prompt-centered spline (cross-fit) & 0.00347 $\pm$ 0.00048 & 0.02463 $\pm$ 0.00114 & -- \\
Prompt-$z$ isotonic (cross-fit) & 0.00201 $\pm$ 0.00054 & 0.02317 $\pm$ 0.00120 & -- \\
Prompt-$z$ spline (cross-fit) & \textbf{0.00126} $\pm$ 0.00051 & 0.01990 $\pm$ 0.00119 & -- \\
Prompt-rank isotonic (cross-fit) & 0.00197 $\pm$ 0.00050 & 0.01919 $\pm$ 0.00118 & -- \\
Prompt-rank spline (cross-fit) & 0.00383 $\pm$ 0.00048 & \textbf{0.01733} $\pm$ 0.00117 & -- \\
\bottomrule
\end{tabular}

\end{adjustbox}
\end{table}

\FloatBarrier
\subsection{Measurement intervention}
\label{app:measurement-intervention-filled}

The intervention experiment aggregates more raters into the proxy while holding the item set and target notion fixed.  The result is intentionally mixed.  Reliability and $A$ improve almost uniformly, but surrogate headroom depends on both signal and curvature.  The reported headroom statistic is the vertex height of the concave two-term surrogate:
\[
H_*(A,B)=-\frac{A^2}{4B},\qquad A>0,\ B<0.
\]
It is recorded as missing outside this positive-alignment, negative-curvature regime.  The vertex can lie outside the checked KL window, so this is an extrapolated surrogate statistic, not a validated maximum attainable gain or a maximum restricted to $d\le D$.  The headroom-error summary in Table~\ref{tab:scaling-a-vs-nu-filled} uses the same convention.  In SummEval, moving from $k=1$ to $k=2$ raises $A$ in all four groups but makes $B$ more negative and decreases surrogate headroom in all four.  The eligible WMT MQM cluster slice has positive median headroom change, but only 12.8\% of groups have a positive 2.5th percentile across random rater partitions, not a bootstrap confidence bound.

\begin{table}[H]
\centering
\scriptsize
\caption{Measurement-intervention summary.  $H_*$ is the extrapolated surrogate vertex height $-A^2/(4B)$ for $A>0,B<0$, not a checked-window maximum.  $\Delta H$ combines changes in signal and curvature; $q_{.025}$ is its empirical 2.5th percentile over random rater partitions.  Pair reps is the total number of random rater partitions (500 per group), each giving one paired before/after comparison.}
\label{tab:measurement-intervention-filled}
\begin{adjustbox}{max width=\textwidth}
\begin{tabular}{llrrrrrrrrrr}
\toprule
Dataset & Transition & Groups & Pair reps & $\Delta\nu>0$ & $\Delta A>0$ & $\Delta B<0$ & $\Delta H>0$ & $q_{.025}(\Delta H)>0$ & med. $\Delta A$ & med. $\Delta B$ & med. $\Delta H$ \\
\midrule
SummEval & $k=1\to2$ & 4 & 2000 & 100.0\% & 100.0\% & 100.0\% & 0.0\% & 0.0\% & 0.100 & -0.069 & -0.569 \\
SummEval & $k=2\to3$ & 4 & 2000 & 100.0\% & 100.0\% & 100.0\% & 75.0\% & 0.0\% & 0.045 & -0.018 & 0.002 \\
WMT MQM clusters & $k=1\to2$ & 47 & 23500 & 95.7\% & 100.0\% & 72.3\% & 72.3\% & 12.8\% & 0.364 & -0.089 & 0.331 \\
\bottomrule
\end{tabular}
\end{adjustbox}
\end{table}

\begin{figure}[tbp]
\centering
\includegraphics[width=0.84\linewidth]{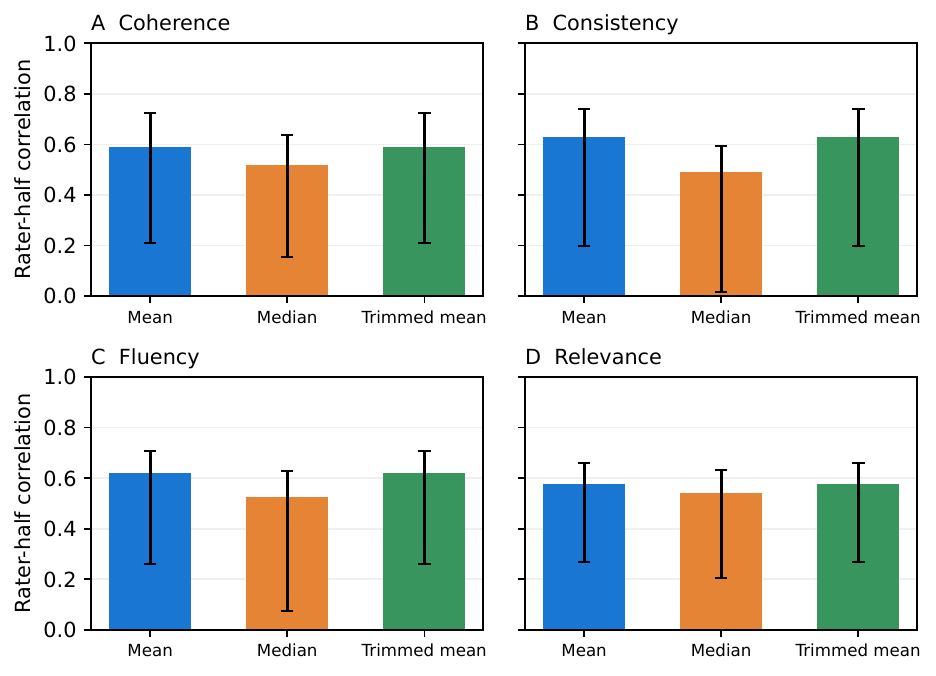}
\caption{Protocol split-half stability.  Panel A reports coherence split-half correlations under mean, median, and trimmed-mean proxy construction.  Panel B gives the same comparison for consistency.  Panel C gives the same comparison for fluency.  Panel D gives the same comparison for relevance.  Error bars summarize random rater-half splits on the same items. The panels describe sensitivity to rater composition, while Table~\ref{tab:measurement-intervention-filled} shows that improving measurement changes both signal and curvature.}
\label{fig:protocol-stability-filled}
\end{figure}

\FloatBarrier
\subsection{Proxy perturbation baselines}
\label{app:proxy-perturbation-filled}

Proxy perturbations isolate whether the exact coefficient law is sensitive to actual proxy--gold alignment.  The main text summarizes this sanity check.  Here Table~\ref{tab:proxy-perturbation-filled} gives the full noise sweep, Figure~\ref{fig:proxy-perturbation-filled} shows the actual-gain bars for the single-slice stress test, and Figure~\ref{fig:proxy-perturbation-multidataset} repeats the audit across eligible slices.  Oracle and original proxies produce positive gains and positive exact $A$.  The random-normal and shuffle controls answer slightly different questions.  A random-normal proxy replaces the original score with independent noise, so both the induced ranking and the score distribution change.  A shuffled proxy preserves the original proxy histogram and variance but destroys itemwise proxy--gold alignment.  In both controls, the realized gain and exact $A$ are close to zero on average, showing that the coefficient law is tracking directional alignment rather than proxy scale alone.  The shuffle result is especially diagnostic because $A_\nu$ stays spuriously high even when the original proxy scale is preserved, confirming that $A_\nu$ is driven by scale and reliability rather than alignment.  Sign-flipping the proxy makes the actual gain negative and exact $A$ negative, whereas $A_\nu$ still predicts a large positive signal.

\begin{table}[H]
\centering
\scriptsize
\caption{Expanded proxy perturbation baselines.  The random-normal control removes both alignment and the original score distribution; the shuffle control preserves the proxy marginal distribution while destroying itemwise alignment.  In both cases $A_\nu$ remains spuriously high.}
\label{tab:proxy-perturbation-filled}
\begin{adjustbox}{max width=\textwidth}
\begin{tabular}{lrrrrrr}
\toprule
Mode & actual $\Delta(0.10)$ & $A_{\rm exact}$ & $A_\nu$ & Spearman & MAE $A{+}B$ & MAE $A_\nu$ \\
\midrule
noisy $\sigma=0.5$ & 0.1893 & 0.678 & 0.855 & 0.500 & \textbf{0.0170} & 0.0828 \\
noisy $\sigma=1.0$ & 0.1503 & 0.537 & 0.676 & 0.403 & \textbf{0.0165} & 0.0646 \\
noisy $\sigma=2.0$ & 0.0974 & 0.338 & 0.428 & 0.258 & \textbf{0.0156} & 0.0392 \\
oracle & 0.3246 & 1.111 & 0.954 & 1.000 & \textbf{0.0133} & 0.0664 \\
original & 0.2159 & 0.754 & 0.955 & 0.606 & \textbf{0.0158} & 0.0878 \\
random normal & 0.0009 & 0.003 & 0.957 & 0.003 & \textbf{0.0158} & 0.3017 \\
shuffle & -0.0000 & -0.001 & 0.959 & -0.001 & \textbf{0.0147} & 0.3033 \\
sign flip & -0.2577 & -0.758 & 0.957 & -0.605 & \textbf{0.0222} & 0.5602 \\
\bottomrule
\end{tabular}
\end{adjustbox}
\end{table}

\begin{figure}[tbp]
\centering
\includegraphics[width=0.80\linewidth]{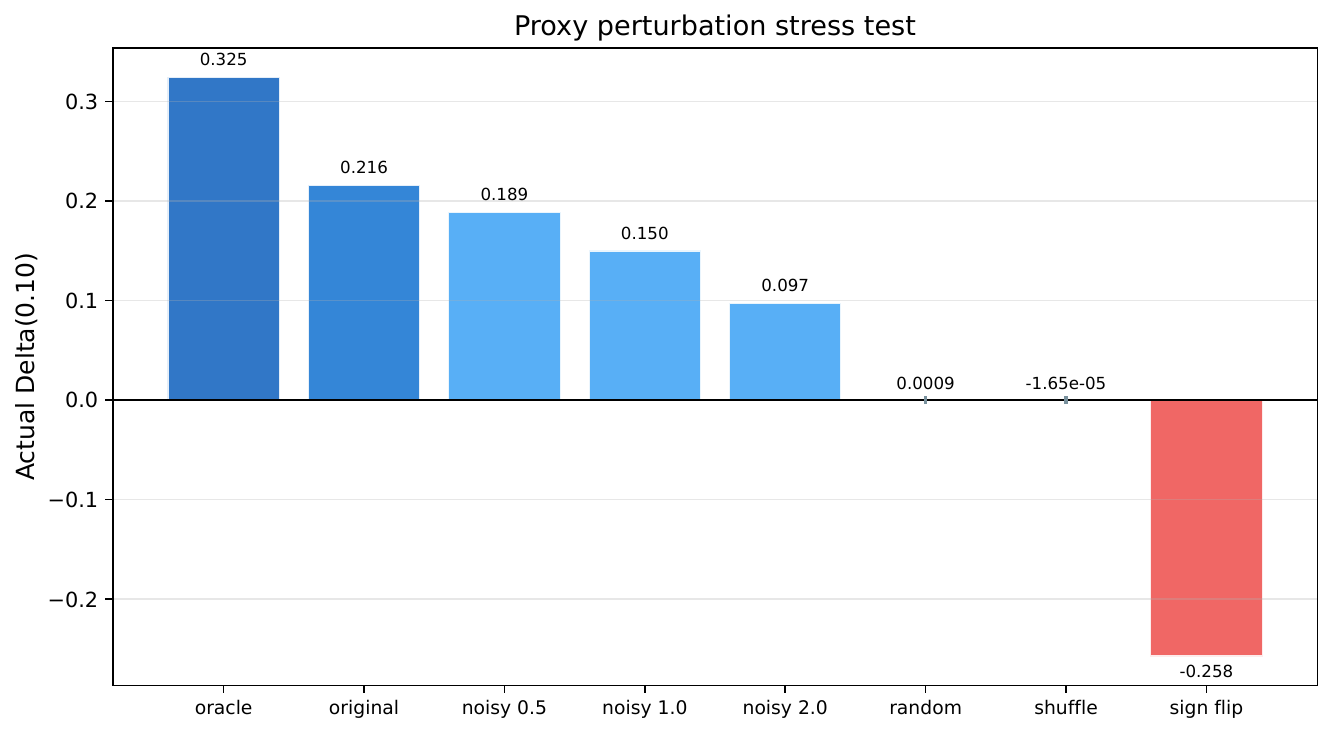}
\caption{Proxy perturbation stress test.  The bars show realized endpoint gains $\Delta(0.10)$ under noisy, oracle, shuffled, random-normal, and sign-flipped proxy variants.  Coefficient values and baseline prediction errors are reported in Table~\ref{tab:proxy-perturbation-filled}; the multi-dataset extension is Figure~\ref{fig:proxy-perturbation-multidataset}.  The shuffle condition shows that the result is not a proxy-scale artifact: preserving the marginal score distribution while destroying itemwise alignment collapses realized gain.}
\label{fig:proxy-perturbation-filled}
\end{figure}

\begin{figure}[tbp]
\centering
\includegraphics[width=\linewidth]{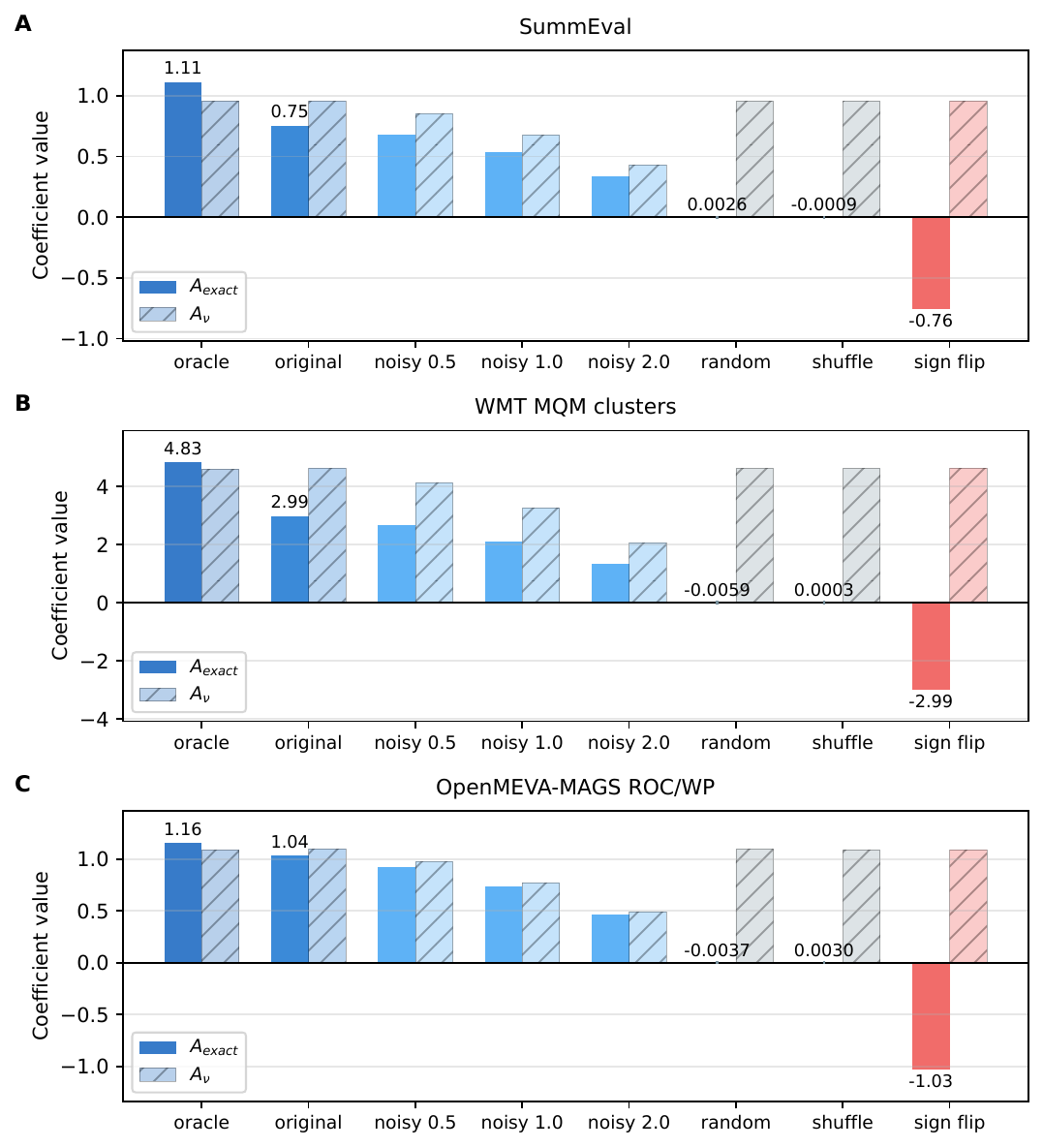}
\caption{Proxy perturbation stress test across eligible analysis slices.  Panel A shows SummEval, Panel B WMT MQM clusters, and Panel C OpenMEVA-MAGS ROC/WP.  In each panel, the exact first-order coefficient $A$ tracks alignment direction across perturbations, while the reliability-only $A_\nu$ baseline fails to distinguish aligned from unaligned proxies.}
\label{fig:proxy-perturbation-multidataset}
\end{figure}

\FloatBarrier
\subsection{Failure-attribution guide}
\label{app:failure-attribution}

A complete report of the diagnostic names the analysis slice, path count, calibration fraction, target construction, proxy source and configuration, planned selection rule, reachable KL radii, tilted variance, RelESS, endpoint and curve-level errors with their uncertainty, and the baselines used for comparison, including $A$-only whenever the interval for $B$ includes zero.  Large prediction error can indicate coefficient noise, calibration--evaluation shift, finite-range breakdown, or proxy--gold rank and tail mismatch.  Table~\ref{tab:failure-attribution-main} maps each pattern to its likely source and a recommended follow-up check (Section~\ref{sec:discussion}).

\begin{table}[H]
\centering
\scriptsize
\caption{Failure-attribution guide.  The same MAE or sign result can have different causes, so each pattern triggers a different follow-up check.}
\label{tab:failure-attribution-main}
\begin{adjustbox}{max width=\textwidth}
\begin{tabular}{lll}
\toprule
Observed pattern & Likely source & Recommended follow-up \\
\midrule
large MAE, unstable $\hat A,\hat B$ & coefficient noise & add calibration ratings or aggregate raters \\
source-only OOD error, target calibration works & calibration--evaluation shift & use source coefficients as a prior \\
accurate MAE, weak sign accuracy & near-zero margin & report magnitude and margin-conditioned sign \\
weak Spearman or top-tail overlap & external-proxy mismatch & audit ranking before optimizing the judge \\
ESS or variance failure & finite-range breakdown & shrink the KL window or collect tail evidence \\
interval for $B$ includes zero & weak or unstable curvature & report $A{+}B$ and $A$-only; add ratings \\
planned rule outside checked window & discrete or strong selection & lower pressure, validate the rule, or abstain \\
\bottomrule
\end{tabular}
\end{adjustbox}
\end{table}

\subsection{Broader impacts}
\label{app:broader-impacts}

A local forecast of human gain can help identify low-gain or misaligned proxies before more aggressive optimization, direct human-evaluation effort to where it changes a decision, and support more conservative deployment choices.  The same diagnostic could also be used to tune stronger optimization against flawed reward models or LLM judges, which could amplify annotator bias, judge-specific artifacts, or unfair group effects, particularly if a forecast is applied outside its checked window or to a population on which it was not calibrated.  These risks motivate the paper's emphasis on repeated human calibration, slice-specific reporting, bounded KL windows, abstention when the evidence is insufficient, and human oversight rather than unconstrained proxy maximization.

\clearpage
\providecommand{\answerYes}[1][]{\textcolor{blue}{[Yes] #1}}
\providecommand{\answerNo}[1][]{\textcolor{orange}{[No] #1}}
\providecommand{\answerNA}[1][]{\textcolor{gray}{[NA] #1}}

\section*{NeurIPS Paper Checklist}

\begin{enumerate}

\item {\bf Claims}
    \item[] Question: Do the main claims made in the abstract and introduction accurately reflect the paper's contributions and scope?
    \item[] Answer: \answerYes{}
    \item[] Justification: The abstract and introduction state the local pre-optimization diagnostic claim and its local, one-sided scope, the repeated-rating calibration setting, the comparison with flexible calibration under matched information, the scope of the GPT-4o external proxy, the LoRA-DPO bridge, and the OOD-transfer result.

\item {\bf Limitations}
    \item[] Question: Does the paper discuss the limitations of the work performed by the authors?
    \item[] Answer: \answerYes{}
    \item[] Justification: Section~\ref{sec:discussion} discusses the locality and one-sidedness of the theory, the checked KL window and how to choose it, the cost of repeated ratings and when to recalibrate, the scope of the judge, reward-model, Best-of-$N$, and DPO results, and differences in reproducibility.  Sections~\ref{sec:exp-prepool}--\ref{sec:exp-bridge} report where flexible calibrators match or beat the coefficient forecast, when the estimate of $B$ is unreliable, and where discrete selection and DPO depart from the theorem.

\item {\bf Theory assumptions and proofs}
    \item[] Question: For each theoretical result, does the paper provide the full set of assumptions and a complete (and correct) proof?
    \item[] Answer: \answerYes{}
    \item[] Justification: The main text states the local coefficient law and its scope. Appendix~\ref{app:proofs} gives the assumptions, proofs, finite-range extensions, and cross-references for the theoretical results.

    \item {\bf Experimental result reproducibility}
    \item[] Question: Does the paper fully disclose all the information needed to reproduce the main experimental results of the paper to the extent that it affects the main claims and/or conclusions of the paper (regardless of whether the code and data are provided or not)?
    \item[] Answer: \answerYes{}
    \item[] Justification: Sections~\ref{sec:protocol}--\ref{sec:experiments} and Appendix~\ref{app:experiments} describe the tilt construction, proxy--gold paths, calibration/evaluation splits, information regimes, baselines and their tuning, path counts, skipped cases, the DPO setup, and the reported diagnostics.  Appendix~\ref{app:repro-assets} distinguishes three levels of reproducibility: released GPT-4o aggregates reproduce all downstream analyses but not the API calls; scored files and recorded model revisions reproduce the open reward-model and judge analyses; and the DPO inputs, code, and seeds are released, but the LoRA weights were not retained.  Archived training-job logs support the reported stage-level compute accounting.

\item {\bf Open access to data and code}
    \item[] Question: Does the paper provide open access to the data and code, with sufficient instructions to faithfully reproduce the main experimental results, as described in supplemental material?
    \item[] Answer: \answerYes{}
    \item[] Justification: Code, prepared inputs, and derived results are released at \url{https://github.com/pacifiic/local-proxy-diagnostics}, with requirements, run scripts, artifact builders, and the cached GPT-4o and open-proxy score files.  This reproduces every reported analysis from prepared inputs.  The exact GPT-4o API score-generation run is not reproducible from the aggregate artifact alone, and the DPO LoRA weights were not retained; Appendix~\ref{app:repro-assets} describes source-data access and these limitations.

\item {\bf Experimental setting/details}
    \item[] Question: Does the paper specify all the training and test details (e.g., data splits, hyperparameters, how they were chosen, type of optimizer) necessary to understand the results?
    \item[] Answer: \answerYes{}
    \item[] Justification: The paper reports datasets, path counts, repeat counts, skipped cases, calibration fractions, fixed target-heavy transfer weight, KL ranges, baselines and their calibration-internal tuning grids, DPO hyperparameters and checkpoint selection, and evaluation metrics in the main text and Appendix~\ref{app:experiments}.

\item {\bf Experiment statistical significance}
    \item[] Question: Does the paper report error bars suitably and correctly defined or other appropriate information about the statistical significance of the experiments?
    \item[] Answer: \answerYes{}
    \item[] Justification: Table~\ref{tab:prediction-main}a reports path-level bootstrap 95\% intervals, and Figure~\ref{fig:ab-contribution} shows path-cluster bootstrap intervals.  Appendices~\ref{app:flexible} and~\ref{app:b-stability} give paired path-clustered bootstrap intervals for every comparison in Tables~\ref{tab:prediction-main}b and~\ref{tab:target-aware}, and the DPO results report the mean and sample SD across three training seeds.  The captions state that GPT-4o has only four paths, so its intervals are descriptive.

\item {\bf Experiments compute resources}
    \item[] Question: For each experiment, does the paper provide sufficient information on the computer resources (type of compute workers, memory, time of execution) needed to reproduce the experiments?
    \item[] Answer: \answerYes{}
    \item[] Justification: Appendix~\ref{app:repro-assets} reports the CPU and memory configuration, eight NVIDIA B200 GPUs, one GPU per worker, CPU reproduction costs, and logged GPU runtimes for candidate generation, DPO training, and open-judge scoring.  The accounting is stage-level rather than a complete per-experiment profile of end-to-end runtime and peak memory.
    
\item {\bf Code of ethics}
    \item[] Question: Does the research conducted in the paper conform, in every respect, with the NeurIPS Code of Ethics \url{https://neurips.cc/public/EthicsGuidelines}?
    \item[] Answer: \answerYes{}
    \item[] Justification: The paper analyzes existing evaluation datasets and derived diagnostics. It does not introduce a deployed system. It frames external-judge results as a scoped external-proxy evaluation rather than a deployment recommendation.

\item {\bf Broader impacts}
    \item[] Question: Does the paper discuss both potential positive societal impacts and negative societal impacts of the work performed?
    \item[] Answer: \answerYes{}
    \item[] Justification: Appendix~\ref{app:broader-impacts} discusses both positive and negative broader impacts of the work. On the positive side, a local proxy-optimization diagnostic can help identify low-gain or misaligned proxies before more aggressive optimization, target human-evaluation effort more efficiently, and support more conservative deployment decisions. On the negative side, the same diagnostic could be used to tune stronger optimization against flawed reward models or LLM judges, which could amplify annotator bias, judge-specific artifacts, or unfair group effects. These risks motivate the paper's emphasis on repeated human calibration, slice-specific reporting, bounded KL windows, and human oversight rather than unconstrained proxy maximization.
    
\item {\bf Safeguards}
    \item[] Question: Does the paper describe safeguards that have been put in place for responsible release of data or models that have a high risk for misuse (e.g., pre-trained language models, image generators, or scraped datasets)?
    \item[] Answer: \answerNA{}
    \item[] Justification: The paper does not release a high-risk model, scraped dataset, or deployed system requiring controlled access safeguards.

\item {\bf Licenses for existing assets}
    \item[] Question: Are the creators or original owners of assets (e.g., code, data, models), used in the paper, properly credited and are the license and terms of use explicitly mentioned and properly respected?
    \item[] Answer: \answerYes{}
    \item[] Justification: Appendix~\ref{app:repro-assets}, Table~\ref{tab:asset-provenance}, lists the source, license or access terms, and redistribution status for each external asset used in the paper. The table includes SummEval (MIT), WMT MQM (Apache-2.0), RewardBench (repository Apache-2.0; dataset card ODC-BY), MT-Bench/FastChat (repository Apache-2.0; \texttt{mt\_bench\_human\_judgments} dataset card CC-BY-4.0), UltraFeedback (MIT), and the model weights used for scoring and training: Skywork-Reward-Llama-3.1-8B-v0.2 (Skywork Community License), ArmoRM-Llama3-8B-v0.1 (Llama~3 license), OpenAssistant reward-model-deberta-v3-large-v2 (MIT), Qwen2.5-32B-Instruct (Apache-2.0), and Qwen2.5-3B-Instruct (Qwen Research License).  Model weights are not redistributed. For assets without a verified public license page for the source used here, such as OpenMEVA-MAGS ROC/WP and the QGEval human files, the paper marks the license/terms as not specified and does not redistribute the raw data.

\item {\bf New assets}
    \item[] Question: Are new assets introduced in the paper well documented and is the documentation provided alongside the assets?
    \item[] Answer: \answerYes{}
    \item[] Justification: The public release documents the analysis scripts, generated diagnostics, artifact builders, requirements, run scripts, checksums, and cached proxy scores, under an MIT license for the code.  The cached scores support reproduction of the reported analyses, while exact reruns of GPT-4o score generation would require raw per-call logs beyond the released aggregate file.  Appendix~\ref{app:repro-assets} summarizes these assets.

\item {\bf Crowdsourcing and research with human subjects}
    \item[] Question: For crowdsourcing experiments and research with human subjects, does the paper include the full text of instructions given to participants and screenshots, if applicable, as well as details about compensation (if any)? 
    \item[] Answer: \answerNA{}
    \item[] Justification: The paper analyzes pre-existing human annotation datasets and does not conduct new crowdsourcing or new human-subject data collection.

\item {\bf Institutional review board (IRB) approvals or equivalent for research with human subjects}
    \item[] Question: Does the paper describe potential risks incurred by study participants, whether such risks were disclosed to the subjects, and whether Institutional Review Board (IRB) approvals (or an equivalent approval/review based on the requirements of your country or institution) were obtained?
    \item[] Answer: \answerNA{}
    \item[] Justification: The paper uses secondary analysis of existing annotations and does not introduce new human-subject interaction.

\item {\bf Declaration of LLM usage}
    \item[] Question: Does the paper describe the usage of LLMs if it is an important, original, or non-standard component of the core methods in this research? Note that if the LLM is used only for writing, editing, or formatting purposes and does \emph{not} impact the core methodology, scientific rigor, or originality of the research, declaration is not required.
    \item[] Answer: \answerYes{}
    \item[] Justification: LLMs are used as proxies under evaluation, not as part of the method: a GPT-4o judge, a Qwen2.5-32B-Instruct judge, LLM-based reward models, and a Qwen2.5-3B-Instruct policy for the DPO bridge.  Section~\ref{sec:experiments} and Appendices~\ref{app:gpt-repro}, \ref{app:open-proxies}, \ref{app:rm-example}, and~\ref{app:dpo} describe the models, revisions, prompts or scoring settings, and the fields that are not recoverable from stored artifacts.

\end{enumerate}

\end{document}